\documentclass{article}

\usepackage{microtype}
\usepackage{graphicx}
\usepackage{subcaption}
\usepackage{booktabs} 

\usepackage{hyperref}

\usepackage[preprint]{icml2026}

\usepackage{amsmath}
\usepackage{amssymb}
\usepackage{mathtools}
\usepackage{amsthm}

\usepackage{booktabs} 
\usepackage{multirow} 
\usepackage{graphicx} 
\usepackage{subcaption}

\usepackage{adjustbox}
\usepackage[normalem]{ulem}
\usepackage[table]{xcolor}  
\usepackage{amsmath}
\usepackage{amsfonts}
\usepackage{bm}
\usepackage{enumitem}
\usepackage{subcaption} 
\usepackage{arydshln}
\usepackage{url}

\usepackage{makecell}   
\usepackage[table]{xcolor}
\usepackage{colortbl}
\usepackage{pgf}

\usepackage{pifont}
\usepackage{newunicodechar}

\usepackage[capitalize,noabbrev]{cleveref}

\theoremstyle{plain}
\newtheorem{theorem}{Theorem}[section]

\newtheorem{lemma}[theorem]{Lemma}

\theoremstyle{definition}

\newtheorem{assumption}[theorem]{Assumption}
\theoremstyle{remark}

\usepackage[textsize=tiny]{todonotes}

\icmltitlerunning{Detecting Contextual Hallucinations in LLMs with Frequency-Aware Attention}

\begin{document}

\twocolumn[
    \icmltitle{Detecting Contextual Hallucinations in Large Language Models\\ with Frequency-Aware Attention }



  \icmlsetsymbol{equal}{*}

  \begin{icmlauthorlist}
    \icmlauthor{Siya Qi}{kcl}
    \icmlauthor{Yudong Chen}{waric}
    \icmlauthor{Runcong Zhao}{kcl}
    \icmlauthor{Qinglin Zhu}{kcl}
    \icmlauthor{Zhanghao Hu}{kcl}
    \icmlauthor{Wei Liu}{kcl} \\
    \icmlauthor{Yulan He}{kcl,alan}
    \icmlauthor{Zheng Yuan}{shef,alan}
    \icmlauthor{Lin Gui}{kcl}


  \end{icmlauthorlist}

  \icmlaffiliation{kcl}{Department of Informatics, King's College London, UK}
  \icmlaffiliation{waric}{Department of Statistics, University of Warwick, UK}
  \icmlaffiliation{shef}{School of Computer Science, The University of Sheffield, UK}
  \icmlaffiliation{alan}{The Alan Turing Institute, UK}

  \icmlcorrespondingauthor{Lin Gui}{lin.1.gui@kcl.ac.uk}
  \icmlcorrespondingauthor{Siya Qi}{siya.qi@kcl.ac.uk}

  \icmlkeywords{Machine Learning, ICML}
  \vskip 0.3in
]



\printAffiliationsAndNotice{}  

\begin{abstract}
Hallucination detection is critical for ensuring the reliability of large language models (LLMs) in context-based generation.
Prior work has explored intrinsic signals available during generation, among which attention offers a direct view of grounding behavior. However, existing approaches typically rely on coarse summaries that fail to capture fine-grained instabilities in attention.
Inspired by signal processing, we introduce a frequency-aware perspective on attention by analyzing its variation during generation.
We model attention distributions as discrete signals and extract high-frequency components that reflect rapid local changes in attention.
Our analysis reveals that hallucinated tokens are associated with high-frequency attention energy, reflecting fragmented and unstable grounding behavior.
Based on this insight, we develop a lightweight hallucination detector using high-frequency attention features.
Experiments on the RAGTruth and HalluRAG benchmarks show that our approach achieves performance gains over verification-based, internal-representation-based, and attention-based methods across models and tasks.\footnote{Code and data are available at \url{https://github.com/siyaqi/FrequencyAwareHallucination}.}

\end{abstract}

\begin{figure}[ht]
  \vskip 0.4in
  \begin{center}
    \centerline{\includegraphics[width=0.95\columnwidth]{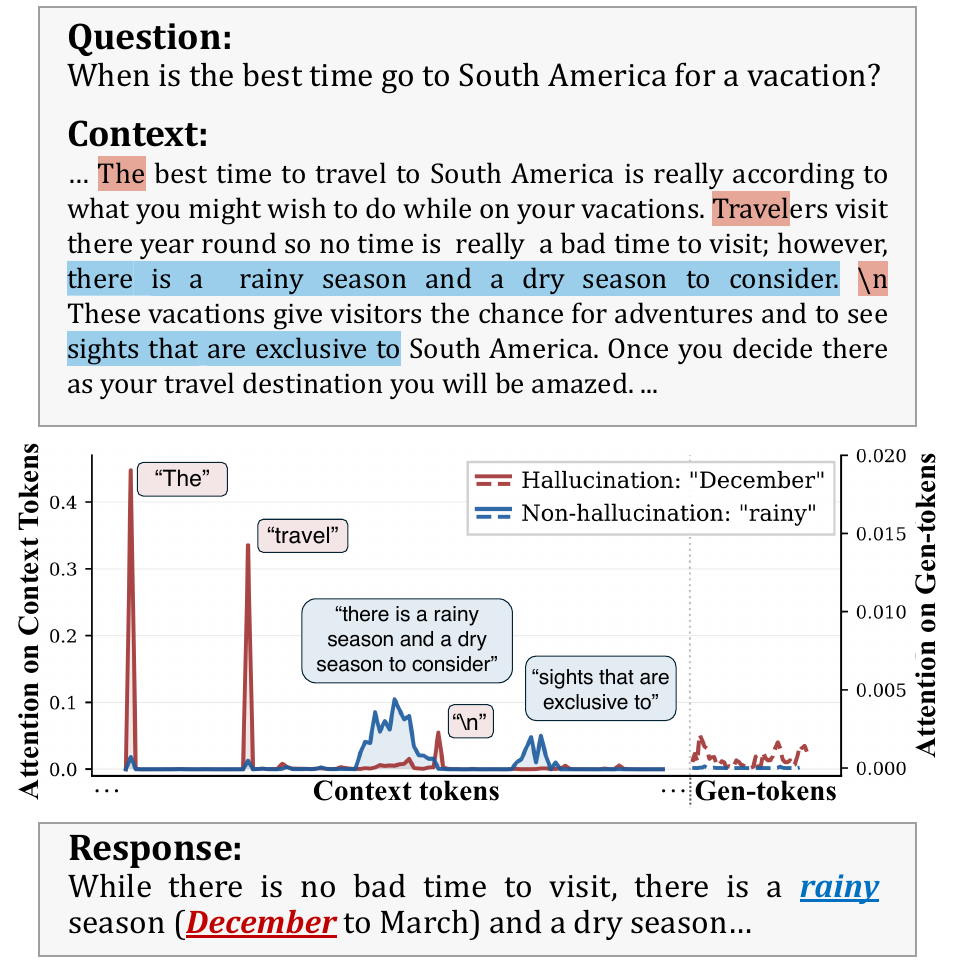}}
    \caption{
      Attention weights over context and previously generated tokens for a grounded token (blue, ``rainy'') and a hallucinated token (red, ``December'') in a context-based QA example.
}
    \label{fig:example}
  \end{center}
\end{figure}

\section{Introduction}

Large Language Models (LLMs) have achieved strong performance across many natural language processing tasks, yet they can produce \emph{hallucinated outputs} that are fluent but not supported by the given input or external facts \cite{ji2023survey, deemter-2024-pitfalls}. This issue is especially pronounced in the \textbf{context-based generation} settings, 
such as the summarization task, where models are explicitly expected to ground in a provided source context \cite{hu2024lrp4ragdetectinghallucinationsretrievalaugmented}. 
Therefore, effective hallucination detection is essential for building trustworthy language systems and enabling downstream mitigation strategies.

A common approach to hallucination detection verifies model outputs against the source context, for example, using semantic similarity measures or LLM-as-a-judge frameworks \cite{Kryscinski2020, Laban2021, manakul-etal-2023-selfcheckgpt}. Such methods compare the output with contextual evidence and are applied post hoc, rather than reflecting the model’s generation dynamics.
Motivated by this limitation, recent work has explored hallucination detection from intrinsic signals available during generation, including token probabilities, hidden representations, and attention patterns 
\cite{sun2024redeep, chen2024inside, chuang2024lookback}. Among these signals, attention is particularly informative, as it directly reflects how the model allocates content focus during generation \cite{Wiegreffe2019}.
Prior studies have shown that hallucinated generations are often associated with unstable or diffuse attention behavior \cite{Feng2023, gong-etal-2024-damro, sriramananllm}. However, \textit{how to quantify attention stability or uncertainty remains an open problem}.

Most existing attention-based methods summarize attention using coarse statistics, such as attention mass, entropy, or transition patterns \cite{huang2025dynamicattentionguidedcontextdecoding, sun2024redeep, chuang2024lookback}. While effective at capturing overall concentration, such scaling-based metrics often discard fine-grained sequential variation. 
For example, in Figure~\ref{fig:example}, we prompt LLaMA-2-7B-Chat with a context-based question and require the model to answer strictly using the provided context. Although most generated tokens align well with the source, the model produces a hallucinated token containing month names (e.g., “December”) that does not appear in the context. 
Compared to a grounded token, the attention distribution associated with this hallucinated token exhibits noticeably stronger local fluctuations across context positions, with sharper peaks and more abrupt changes.

Therefore, we argue that directly mapping an attention sequence to a single scalar cannot fully capture the structure of attention patterns. Inspired by signal processing, we treat attention weights over context tokens as a discrete temporal signal indexed by token position, where stable grounding corresponds to smooth, slowly varying signals, while instability manifests as rapid local oscillations. 
To explicitly quantify such variation, we perform frequency-aware decomposition of the attention signal: low-frequency components capture global trends of attention allocation across the context, whereas high-frequency components isolate sharp spikes and abrupt local changes. We hypothesize that hallucinated tokens are associated with energy in these high-frequency attention components.

Motivated by this perspective, we study contextual hallucination detection via frequency-aware analysis of attention signals. Our contributions are threefold:
\textbf{(1)} We formulate attention as discrete signals and introduce a unified frequency-based framework for analyzing attention variation for hallucination detection.
\textbf{(2)} We instantiate this framework using simple yet efficient high-frequency extraction operators, enabling the quantification of attention instability at both token and span levels.
\textbf{(3)} Through extensive experiments across multiple models and tasks, we show that frequency-based attention features can improve hallucination detection over existing verification-based, internal representation-based, and attention-based baselines.

\begin{figure}[t!]
    \vskip 0.1in
  \begin{center}
    \centerline{\includegraphics[width=0.95\columnwidth]{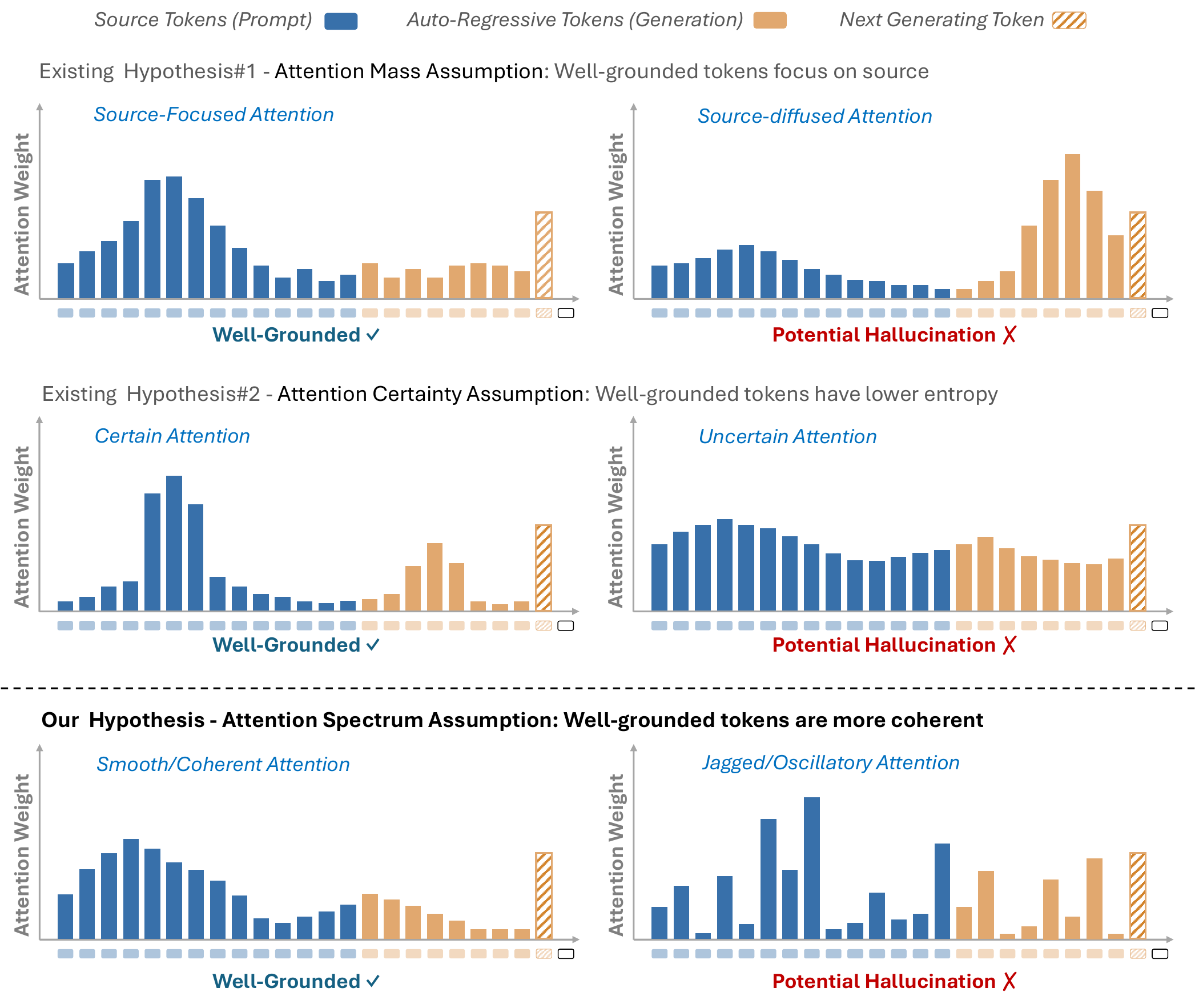}}
    \caption{
      Three hypotheses for identifying hallucination tokens from incoming attention patterns.
We illustrate three representative assumptions for distinguishing well-grounded tokens (\ding{51}) from potential hallucinations (\ding{55}) based on the incoming attention to the next generated token.}
    \label{fig:hypothese_compare}
  \end{center}
\end{figure}

\section{Background}
In this section, we review attention weight as an intrinsic signal for hallucination and introduce a frequency-based perspective that characterizes attention instability beyond the coarse allocation statistics.

\subsection{Attention-based Hallucination Detection}

For each generated token, the attention distribution reflects how the model aggregates information from the source context and previously generated tokens, and is therefore widely used as an intrinsic signal of grounding and information usage \cite{campbell2023localizing, li2024inferencetimeinterventionelicitingtruthful, Snyder2024earlyDetection, huang2025dynamicattentionguidedcontextdecoding}. Prior work exploits attention for hallucination detection by making different assumptions about grounded generation behavior.

One line of work focuses on attention transitions, using backward attention from generated tokens to the source context as an indicator of grounding \cite{chuang2024lookback, ogasa2025hallucinatedMulti-View}. Another line analyzes the distributional properties of attention, either by identifying context tokens emphasized by specific attention heads \cite{sun2024redeep} or by summarizing attention uncertainty using entropy-based measures \cite{vazhentsev2025uncertaintyawareattentionheadsefficient}. These approaches are motivated by the intuition that grounded generation exhibits focused and consistent evidence attribution, whereas hallucination is associated with diffuse or ambiguous attention.

As illustrated in \autoref{fig:hypothese_compare}, existing attention-based methods primarily capture \emph{where} attention is allocated or \emph{how} dispersed it is, but do not explicitly model the internal structure of attention patterns. In contrast, our hypothesis emphasizes the \emph{coherence} of attention distributions, motivated by the observation that hallucinated generations often exhibit fragmented or rapidly oscillating attention behavior that is not captured by static allocation or entropy-based statistics.

\begin{figure*}[ht]
  \begin{center}
    \centerline{\includegraphics[width=1.85\columnwidth]{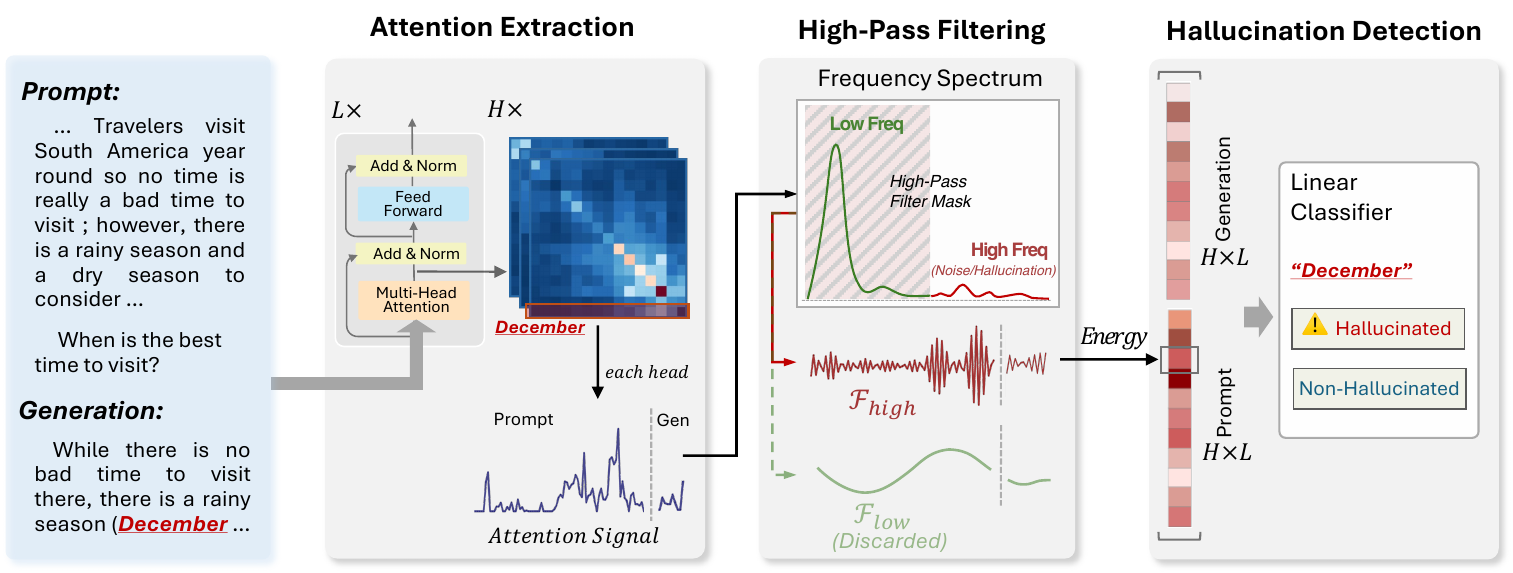}}
    \caption{
      Overview of frequency-aware attention modeling for hallucination detection.
Attention weights are extracted from each layer and head ($L$ layers and $H$ heads in total), treated as token-level signals, and decomposed using high-pass filtering to isolate high-frequency variations ($\mathcal{F}_{\text{high}}$), whose energy is aggregated for hallucination detection.}
    \label{fig:main_fig}
  \end{center}
  \vskip -0.2in
\end{figure*}

\subsection{A Frequency-based View of Attention}

To capture attention variation beyond static summaries, we model attention weights over tokens as discrete signals. 
In signal processing, smooth and coherent patterns are naturally represented by low-frequency components, whereas abrupt changes and local irregularities manifest as high-frequency components. As such, frequency analysis offers a natural characterization of attention stability and instability.

Rather than committing to a single formulation, we consider several standard operators that instantiate this frequency-based view with complementary inductive biases. 
The Discrete Fourier Transform (DFT) \cite{Cooley1969TheFF} provides a global decomposition of attention signals into frequency components. 
The Discrete Wavelet Transform (DWT) \cite{wavelet1989} uses localized basis functions, allowing abrupt attention changes to be captured while preserving positional information. 
As a simpler local alternative, the discrete Laplacian operator \cite{oppenheim1997signals} directly highlights second-order differences directly in the token domain, acting as an implicit high-pass filter.

Despite their different forms, these operators share a common goal: isolating high-frequency variation that reflects fragmented or rapidly oscillating attention behavior. Together, they provide a unified frequency-based framework for analyzing attention instability during generation.

\section{Frequency-Aware Attention Modeling}
\label{sec:method}

Viewing attention from a frequency-aware perspective offers a principled way to analyze variation patterns that are difficult to characterize directly in the token domain. 
In this framework (shown in \autoref{fig:main_fig}), attention weights over tokens are modeled as discrete signals, enabling the use of frequency-aware operators to quantify how attention evolves and fluctuates during generation, beyond what is captured by aggregate statistics. 
We first demonstrate our motivation and intuition through a simplified setting here.

\subsection{Motivation and Intuition}

While a rigorous mechanistic understanding of attention behavior under hallucinated generation remains an open problem, we can gain preliminary insights from a simplified toy setting. Specifically, consider a scenario where tokens come from several latent semantic sources. As the number of such sources increases, neighboring tokens are more likely to differ in topic, causing the compatibility between the current token and its context to change abruptly across adjacent positions. Such abrupt local changes translate into adjacent differences in attention weights, yielding jagged attention patterns along the sequence (see \autoref{app:three_step_proof} for a complete proof). Although such toy analysis relies on simplifying assumptions, it formalizes a general link between semantic heterogeneity and attention instability.

Real LLM attention is, however, substantially far more complex. It is shaped by architectural depth, multi-head specialization, and long-range contextual interactions, all of which can give rise to higher-order and multi-scale instabilities that cannot be captured by adjacent difference measures alone. 
From a signal-processing perspective, adjacent differences correspond only to a primitive form of high-pass filtering. Frequency-aware operators (e.g., DFT, DWT, Laplacian), by contrast, are able to systematically extract high-frequency components across multiple resolutions \cite{donoho1998minimax,stein2011fourier}. This motivates our frequency-aware framework for quantifying complex attention instabilities associated with hallucination detection.

\subsection{Problem Setup}
\label{subsec:setup}
We consider a context-based generation setting, where a language model generates a response $\mathbf{gen}=(gen_1,\dots,gen_T)$ conditioned on a retrieved context $\mathbf{ctx}=(ctx_1,\dots,ctx_N)$.
At each generation step $i$, the model produces token $t_i$ based on $(\mathbf{ctx},\mathbf{gen}_{<i})$ and exposes attention weights across layers and heads.
Our primary task is token-level hallucination detection: for each generated token $t_i$, predict whether $t_i$ is supported by the provided context.

Rather than operating on attention weights directly, our detector uses frequency-aware features derived from attention.
For each step $i$ when generating $t_i$, we compute a feature vector $\mathbf{v}$ (defined in \S\ref{subsec:instability}) and predict
\begin{equation}
\hat{r}_i = f(\mathbf{v}),
\end{equation}
where $f(\cdot)$ is a lightweight linear classifier, and $\hat{r}_i$ indicates a prediction result for hallucination or non-hallucination.

\subsection{Attention as Discrete Signals}
\label{subsec:signal}

At generation step $i$, we extract attention weight distributions from all transformer layers and heads.
We distinguish two types of attention signals: \emph{context-directed attention}, attending from the current token $t_i$ to the input context tokens, and \emph{generated-token attention}, attending to previously generated tokens before this generation step.

For each layer $l\in\{1,\dots,L\}$ and head $h\in\{1,\dots,H\}$, we define the context-directed attention vector and the generated-token attention vector:
\begin{equation}
\begin{aligned}
a^{(ctx)}_{l,h} &= [a^{(ctx)}_{l,h,1}, a^{(ctx)}_{l,h,2}, \dots, a^{(ctx)}_{l,h,N}] \in \mathbb{R}^N , \\
a^{(gen)}_{l,h} &= [a^{(gen)}_{l,h,1}, a^{(gen)}_{l,h,2}, \dots, a^{(gen)}_{l,h,i-1}] \in \mathbb{R}^{i-1} ,
\end{aligned}
\end{equation}
which respectively capture attention over $N$ context tokens and the previously generated tokens before $t_i$.

Both $a^{(ctx)}_{l,h}$ and $a^{(gen)}_{l,h}$ are treated as one-dimensional discrete signals that indexed by token position.
Taking the Discrete Fourier Transform (DFT) as an example, and let $\mathcal{F}(\cdot)$ denote DFT.
Mapping the attention weight vector to the frequency domain yields
\begin{equation}
\hat{a}_{l,h}^{(\tau)} = \mathcal{F}\big(a^{(\tau)}_{l,h}\big),
\end{equation}
where $\tau\in\{ctx,gen\}$; details for the DWT and Laplacian operators are provided in Appendix~\ref{apx:energy_equal}.

\subsection{Energy-Based High-Frequency Instability}
\label{subsec:instability}

Building on the frequency-based view, we isolate and quantify high-frequency components of attention weight signals.
The formulation applies uniformly to both context-directed and previously generated-token attention.
To simplify notation, we use $\mathbf{x}\in\mathbb{R}^n$ to denote the attention signal obtained for $t_i$,
corresponding to either $a^{(ctx)}_{l,h}$ or $a^{(gen)}_{l,h}$ for layer $l$ and head $h$,
with signal length $n$.

Given an attention signal $\mathbf{x}$, we extract high-frequency component by high-pass operator (e.g., DFT, DWT, Laplacian)
\begin{equation}
\mathbf{z}^{\mathrm{hf}} = \mathcal{F}_{\text{high}}(\mathbf{x}),
\end{equation}
where $\mathcal{F}_{\text{high}}$ denotes a high-pass operator that suppresses smooth, low-frequency components while preserving high-frequency components representing rapid local variation.

For transform-based operators such as DFT and DWT, $\mathcal{F}_{\text{high}}$ is implemented by frequency-domain masking followed by an inverse transform.
Let $\hat{\mathbf{x}} = \mathcal{F}(\mathbf{x})$ denote the DFT coefficients.
We retain only high-frequency components by applying a frequency mask by following definition:
\begin{equation}
\hat{\mathbf{x}}^{\mathrm{hf}}_k =
\begin{cases}
\hat{\mathbf{x}}_k, & k \in \mathcal{K}_{\mathrm{high}}, \\
0, & \text{otherwise},
\end{cases}
\end{equation}
where $\mathcal{K}_{\mathrm{high}}$ indexes the retained high-frequency band.
The corresponding high-frequency component $\mathbf{z}^{\mathrm{hf}}$ in the temporal domain is obtained via inverse transform.
Wavelet- and Laplacian-based operators provide alternative realizations of $\mathcal{F}_{\text{high}}$ that emphasize localized high-frequency variation, while serving the same purpose of isolating rapid attention fluctuations (detailed in Appendix~\ref{apx:energy_equal}).

We summarize the magnitude of high-frequency variation by using the $\ell_2$ norm on the high-frequency component
\begin{equation}
\rho = \| \mathbf{z}^{\mathrm{hf}} \|_2 ,
\end{equation}
which measures the energy of high-frequency components in the attention signal in temporal domain.

The use of the $\ell_2$ norm is theoretically motivated by \textbf{Parseval's theorem}, which establishes the equivalence between signal energy measured in the temporal domain and that measured in the frequency domain.
For DFT-based filtering,
\begin{equation}
\| \mathbf{z}^{\mathrm{hf}} \|_2^2
= \sum_{j=1}^{n} |\mathbf{z}^{\mathrm{hf}}_j|^2
= \frac{1}{n} \sum_{k=0}^{n-1} |\hat{\mathbf{x}}^{\mathrm{hf}}_k|^2 .
\end{equation}

\begin{table*}[t]
\centering
\small

\setlength{\tabcolsep}{3.8pt} 

\newcommand{\lightrule}{\arrayrulecolor{black!30}\cmidrule(lr){3-12}\arrayrulecolor{black}}

\caption{Comparison of different methods on RAGTruth and HalluRAG.
Wavelet-high and Fourier-high denote leveraging the high-frequency components obtained via DWT and DFT, respectively.
\textbf{Bold} and \underline{underlined} values indicate the best and second-best performance within each model group.}
\label{tab:rag_results}

\resizebox{0.83\textwidth}{!}{
\begin{tabular}{c l cc cc cc cc cc}
\toprule
\multirow{2.5}{*}{\textbf{Model}} & 
\multicolumn{1}{c}{\multirow{2.5}{*}{\textbf{Method}}}
& \multicolumn{2}{c}{\textbf{RT-QA}} & \multicolumn{2}{c}{\textbf{RT-D2T}} & \multicolumn{2}{c}{\textbf{RT-Summ}} & \multicolumn{2}{c}{\textbf{HalluRAG}} & \multicolumn{2}{c}{\textbf{Overall Avg.}} \\
\cmidrule(lr){3-4} \cmidrule(lr){5-6} \cmidrule(lr){7-8} \cmidrule(lr){9-10} \cmidrule(lr){11-12}
&& F & AUROC & F & AUROC & F & AUROC & F & AUROC & Avg-F & Avg-A \\
\midrule

\multirow{11}{*}{\rotatebox{90}{\textbf{LLaMA-7B}}}
& SelfcheckGPT & 0.6289 & 0.6942 & 0.6552 & 0.8026 & \textbf{0.6349} & 0.6674 & 0.5388 & 0.5963 & 0.6145 & 0.6901 \\
& RefChecker & 0.5914 & 0.5865 & 0.5713 & 0.6349 & \underline{0.6059} & 0.6080 & 0.5819 & 0.5722 & 0.5876 & 0.6004 \\ 
\lightrule 
& EigenScore & 0.4979 & 0.5253 & 0.5256 & 0.5297 & 0.5065 & 0.4989 & 0.5364 & 0.6127 & 0.5166 & 0.5416 \\
& Redeep & 0.5972 & 0.6364 & 0.4791 & 0.3960 & 0.5897 & 0.5760 & 0.5912 & 0.6490 & 0.5643 & 0.5643 \\
& Lookback-lens & 0.6930 & 0.8482 & 0.6175 & 0.8442 & 0.5328 & 0.7156 & 0.6266 & 0.7405 & 0.6175 & 0.7871 \\ 
\lightrule 
& Attn-variance & 0.4807 & 0.6147 & 0.4839 & 0.5890 & 0.4886 & 0.6492 & 0.4489 & 0.5571 & 0.4755 & 0.6025 \\
& Attn-entropy & 0.6832 & 0.8481 & 0.6011 & 0.8368 & 0.5031 & 0.6722 & 0.6020 & 0.6937 & 0.5973 & 0.7627 \\ 
\lightrule 
& Laplacian & 0.7107 & 0.8449 & \underline{0.6878} & 0.8519 & 0.5779 & 0.7040 & 0.6370 & 0.7429 & 0.6534 & 0.7859 \\
& Wavelet-high & \underline{0.7194} & \underline{0.8526} & \textbf{0.6898} & \underline{0.8569} & 0.5929 & \underline{0.7165} & \underline{0.6384} & \underline{0.7550} & \underline{0.6601} & \underline{0.7953} \\
& Fourier-high & \textbf{0.7277} & \textbf{0.8584} & 0.6870 & \textbf{0.8595} & 0.5875 & \textbf{0.7426} & \textbf{0.6438} & \textbf{0.7603} & \textbf{0.6615} & \textbf{0.8052} \\
\midrule

\multirow{11}{*}{\rotatebox{90}{\textbf{LLaMA-13B}} }
& SelfcheckGPT & 0.6029 & 0.6425 & 0.6346 & 0.7989 & 0.5554 & 0.5909 & 0.6337 & 0.7216 & 0.6066 & 0.6885 \\
& RefChecker & 0.5928 & 0.5893 & 0.5784 & 0.6381 & \underline{0.5924} & 0.6346 & 0.6098 & 0.5963 & 0.5934 & 0.6146 \\
\lightrule
& EigenScore & 0.4788 & 0.4734 & 0.6063 & 0.6582 & 0.5121 & 0.5097 & 0.6316 & 0.6947 & 0.5572 & 0.5840 \\
& Redeep & 0.5740 & 0.6457 & 0.5273 & 0.6449 & 0.5010 & 0.5253 & 0.5769 & 0.6294 & 0.5448 & 0.6113 \\
& Lookback-lens & 0.6947 & 0.8727 & 0.7137 & 0.8766 & 0.5679 & 0.6702 & 0.6594 & \textbf{0.7929} & 0.6589 & 0.8031 \\
\lightrule
& Attn-variance & 0.5636 & 0.7497 & 0.5992 & 0.7177 & 0.4947 & 0.6629 & 0.4529 & 0.6298 & 0.5276 & 0.6900 \\
& Attn-entropy & 0.6909 & 0.8650 & 0.7034 & 0.8717 & 0.5527 & 0.6865 & 0.6284 & 0.7270 & 0.6438 & 0.7875 \\
\lightrule
& Laplacian & 0.6834 & 0.8552 & 0.7194 & 0.8796 & 0.5548 & 0.6700 & \underline{0.6659} & 0.7624 & 0.6559 & 0.7918 \\
& Wavelet-high & \underline{0.7029} & \underline{0.8741} & \textbf{0.7383} & \textbf{0.8932} & 0.5651 & \underline{0.7042} & \textbf{0.6684} & 0.7809 & \underline{0.6687} & \underline{0.8131} \\
& Fourier-high & \textbf{0.7068} & \textbf{0.8792} & \underline{0.7278} & \underline{0.8825} & \textbf{0.5929} & \textbf{0.7362} & 0.6616 & \underline{0.7899} & \textbf{0.6723} & \textbf{0.8198} \\
\midrule

\multirow{11}{*}{\rotatebox{90}{\textbf{Mistral-7B}} }
& SelfcheckGPT & 0.6551 & 0.7084 & 0.6958 & 0.8353 & \textbf{0.6966} & 0.7166 & 0.5515 & 0.6473 & 0.6498 & 0.7269 \\
& RefChecker & 0.6144 & 0.6121 & 0.5907 & 0.6287 & 0.6664 & 0.6596 & 0.6176 & 0.6031 & 0.6223 & 0.6259 \\
\lightrule
& EigenScore & 0.4904 & 0.4611 & 0.5449 & 0.6507 & 0.4582 & 0.4973 & 0.6581 & 0.7924 & 0.5379 & 0.6004 \\
& Redeep & 0.6270 & 0.7357 & 0.6013 & 0.6648 & 0.5379 & 0.5680 & 0.4987 & 0.5646 & 0.5662 & 0.6333 \\

& Lookback-lens & 0.7832 & \textbf{0.9148} & 0.7151 & 0.8845 & 0.6759 & 0.7954 & 0.7115 & 0.7966 & 0.7214 & 0.8478 \\
\lightrule
& Attn-variance & 0.5636 & 0.7497 & 0.6218 & 0.7021 & 0.4840 & 0.5836 & 0.5367 & 0.6953 & 0.5515 & 0.6827 \\
& Attn-entropy & 0.7810 & 0.9083 & 0.7049 & 0.8734 & 0.6551 & 0.7867 & 0.6803 & 0.7504 & 0.7053 & 0.8297 \\
\lightrule
& Laplacian & 0.7807 & 0.9049 & \underline{0.7255} & \underline{0.8849} & 0.6723 & 0.7978 & 0.7001 & 0.8098 & 0.7197 & 0.8494 \\
& Wavelet-high & \underline{0.7876} & \underline{0.9117} & 0.7136 & 0.8829 & \underline{0.6849} & \textbf{0.8075} & \textbf{0.7274} & \textbf{0.8360} & \textbf{0.7284} & \textbf{0.8595} \\
& Fourier-high & \textbf{0.7885} & 0.9099 & \textbf{0.7267} & \textbf{0.8863} & 0.6761 & \underline{0.8037} & \underline{0.7152} & \underline{0.8128} & \underline{0.7266} & \underline{0.8532} \\
\bottomrule
\end{tabular}
}

\end{table*}

Applying the above procedure independently to each attention head and each attention type yields a score $\rho^{(\tau)}_{l,h}$ for every layer $l$, head $h$, and attention type $\tau \in \{ctx,gen\}$. 
\begin{equation}
\begin{aligned}
\mathbf{v}^{(\tau)} &= [\rho^{(\tau)}_{1,1}, \rho^{(\tau)}_{1,2}, \dots, \rho^{(\tau)}_{L,H}]^\top, \quad \tau \in \{ctx,gen\}, \\
\mathbf{v} &= [\mathbf{v}^{(ctx)};\mathbf{v}^{(gen)}] .
\end{aligned}
\end{equation}
We aggregate these scores across all layers and heads to form feature vectors $\mathbf{v}^{(ctx)}$ and $\mathbf{v}^{(gen)}$ for context-directed and generated-token attention, respectively. We then concatenate these vectors as $\mathbf{v}$ and feed them into the classifier to obtain the final prediction $\hat{r}_i$ for token-level hallucination detection.
As a robustness check under coarser supervision, we also report span-level results by applying a fixed sliding window, averaging the feature vectors within each window, and feeding the averaged vectors into the classifier.

\section{Experiment Setting}
\label{sec:experiment}
\subsection{Baselines}
In this study, we evaluate our approach against a diverse set of representative baselines for comparison, spanning verification-based, internal representation-based, and attention-based paradigms, as detailed below.

\textbf{Verification-based Methods.} \textbf{SelfCheckGPT} \cite{manakul-etal-2023-selfcheckgpt} detects hallucinations by measuring stochastic consistency across multiple responses sampled from a language model.
\textbf{RefChecker} \cite{RefChecker_2024_hu} extracts claims from model outputs and verifies them against context using a dedicated checker.
Both of them operate in a prompt-based manner and rely on the generative behavior of the underlying LLM.

\textbf{Internal Representation-Based Methods.} \textbf{EigenScore} \cite{chen2024inside} leverages output probability distributions to construct a semantic consistency graph and quantifies uncertainty via spectral analysis. 
\textbf{ReDeEP} integrates internal signals from hidden states and attention mass to detect hallucinations \citep{sun2024redeep}.
Both methods represent strong internal-signal-based baselines.

\textbf{Attention-based Methods.} We include \textbf{Lookback-Lens} \cite{chuang2024lookback} as a strong attention-based baseline, which characterizes hallucination by quantifying the allocation of attention between retrieved context tokens and previously generated tokens.
In addition, we implement two mechanistic attention-based baselines based on attention \textbf{variance} and attention \textbf{entropy}.
For these baselines, statistics are computed separately over context tokens and generated tokens, and concatenated across all attention heads to form a unified feature representation for classification.

\subsection{Implementation Settings}
To evaluate the effectiveness and robustness of our method, we conduct experiments across three widely used open-source LLMs, including LLaMA-7B-Chat, LLaMA-13B-Chat, and Mistral-7B-Instruct, covering diverse model sizes and architectures.
Experiments are conducted on two context-based hallucination benchmarks: RAGTruth~\cite{niu-etal-2024-ragtruth}, covering question answering (QA), data-to-text (D2T), and summarization (Summ), and HalluRAG~\cite{ridder2025halluragdatasetdetectingcloseddomain}, which focuses on the QA task.
Both datasets provide token-level hallucination annotations.
Frequency-aware attention features are aggregated using a single-layer logistic regression classifier, and detection is performed at either the token or span level.
We report F1 and AUROC on the test sets, using AUROC as the primary metric.
Additional details are provided in Appendix~\ref{apx:model_detail}.

\section{Results and Analysis}

\subsection{Overall Performance}

Across all evaluated models and datasets, frequency-aware analysis of attention consistently improves hallucination detection over verification-based, internal representation-based, and attention-based baselines.
Explicitly isolating high-frequency components of attention yields stronger performance than aggregate statistics such as variance or entropy, highlighting the importance of modeling fine-grained attention variation.
By focusing on \emph{how} attention varies within the sequence rather than \emph{where} attention mass is allocated, frequency-aware features provide complementary discriminative signals beyond attention mass alone.

These improvements hold across diverse task formats and models. 
For example, on summarization task, Fourier-high improves AUROC by 6.6\% over Lookback-Lens on LLaMA-13B, and by 2.7\% on LLaMA-7B.
This consistency across datasets and generation settings suggests that high-frequency attention variation captures a general property of ungrounded generation, rather than task- or structure-specific artifacts.
We further evaluate cross-domain generalization by training the detector on one task and testing it on another (see \autoref{tab:sliding_cross_domain}).
Compared to Lookback-Lens, our method exhibits substantially more robust transfer performance, indicating that frequency-aware attention features generalize better across task boundaries.

Among all three operators, Fourier-based features achieve the strongest overall performance, followed by wavelets and the Laplacian.
This ordering aligns with both their inductive biases and the structural properties of attention signals: attention distributions are typically sparse and highly uneven across positions \cite{nawrot2025sparsefrontiersparseattention}, often dominated by a small subset of tokens, which favors operators that aggregate high-frequency variation globally.
Correspondingly, Fourier-based filtering is particularly effective at capturing global high-frequency energy, while wavelets and the Laplacian emphasize progressively more localized variation.

\begin{table}[!htbp]
\centering
\small

\caption{Performance comparison on RagTruth-Avg and HalluRAG for sliding-window=8. Within each model, the best result is highlighted in \textbf{bold}, and the second-best is \underline{underlined}.}
\label{tab:sliding8_short}

\resizebox{0.45\textwidth}{!}{
\begin{tabular}{c l cc cc}
\toprule
\multirow{2}{*}{\textbf{Model}}
& \multicolumn{1}{c}{\multirow{2}{*}{\textbf{Method}}}
& \multicolumn{2}{c}{\textbf{RagTruth-Avg}}
& \multicolumn{2}{c}{\textbf{HalluRAG}} \\
\cmidrule(lr){3-4} \cmidrule(lr){5-6}
& & F & AUROC & F & AUROC \\
\midrule

\multirow{4}{*}{\rotatebox{90}{\scriptsize \textbf{LLaMA-7B}}}
& Lookback-lens
& 0.6773 & 0.7884 & 0.6623 & 0.7856 \\
& Laplacian
& 0.6747 & 0.7877 & 0.6775 & 0.7754 \\
& Wavelet-High
& \underline{0.6911} & \underline{0.8117} & \underline{0.6812} & \underline{0.7928} \\
& Fourier-high
& \textbf{0.7003} & \textbf{0.8412} & \textbf{0.6866} & \textbf{0.8100} \\
\midrule

\multirow{4}{*}{\rotatebox{90}{\scriptsize \textbf{LLaMA-13B}}}
& Lookback-lens
& 0.6781 & 0.8183 & 0.7096 & \underline{0.8505} \\
& Laplacian
& 0.6819 & 0.8039 & \underline{0.7217} & 0.8264 \\
& Wavelet-High
& \underline{0.6953} & \underline{0.8448} & \textbf{0.7417} & 0.8438 \\
& Fourier-high
& \textbf{0.7063} & \textbf{0.8585} & \underline{0.7217} & \textbf{0.8515} \\
\midrule

\multirow{4}{*}{\rotatebox{90}{\scriptsize \textbf{Mistral-7B}}}
& Lookback-lens
& 0.7554 & \underline{0.8764} & 0.7752 & 0.8403 \\
& Laplacian
& 0.7322 & 0.8495 & \textbf{0.8056} & \textbf{0.8920} \\
& Wavelet-High
& \underline{0.7597} & 0.8742 & 0.7682 & 0.8802 \\
& Fourier-high
& \textbf{0.7663} & \textbf{0.8812} & \underline{0.7883} & \underline{0.8872} \\
\bottomrule
\end{tabular}
}
\vskip -0.1in

\end{table}






\subsection{Span-level Hallucination Detection}

In practical settings, hallucinations often occur as contiguous spans rather than isolated tokens.
To evaluate whether our method generalizes beyond token-level detection, we conduct span-level hallucination detection under a sliding-window setting, where consecutive tokens are grouped into fixed-size chunks and classified at the chunk level.
Following prior work such as Lookback-Lens, we use a chunk size of 8 and aggregate attention features within each chunk to form a single representation for prediction.
Results across the two benchmarks are reported in \autoref{tab:sliding8_short}, with full per-task results provided in \autoref{tab:sliding8_full} for RagTruth.

Span-level evaluation is more challenging due to coarser supervision and the need to aggregate signals across multiple tokens. Despite this increased difficulty, Fourier- and Wavelet-based variants achieve the strongest performance on all three tasks.
Fourier-based features yield consistent gains over Lookback-Lens, improving AUROC by 5.3\% on LLaMA-7B (RAGTruth-Avg), with a particularly improvement of 10.1\% on the summarization task (see \autoref{tab:sliding8_full}).
Overall, these results indicate that frequency-based features remain robust when individual token-level signals are aggregated into spans, and that our approach generalizes naturally from token- to span-level hallucination detection across multiple tasks without requiring architectural modification.

\begin{figure}[ht]
  \begin{center}
    \centerline{\includegraphics[width=0.9\columnwidth]{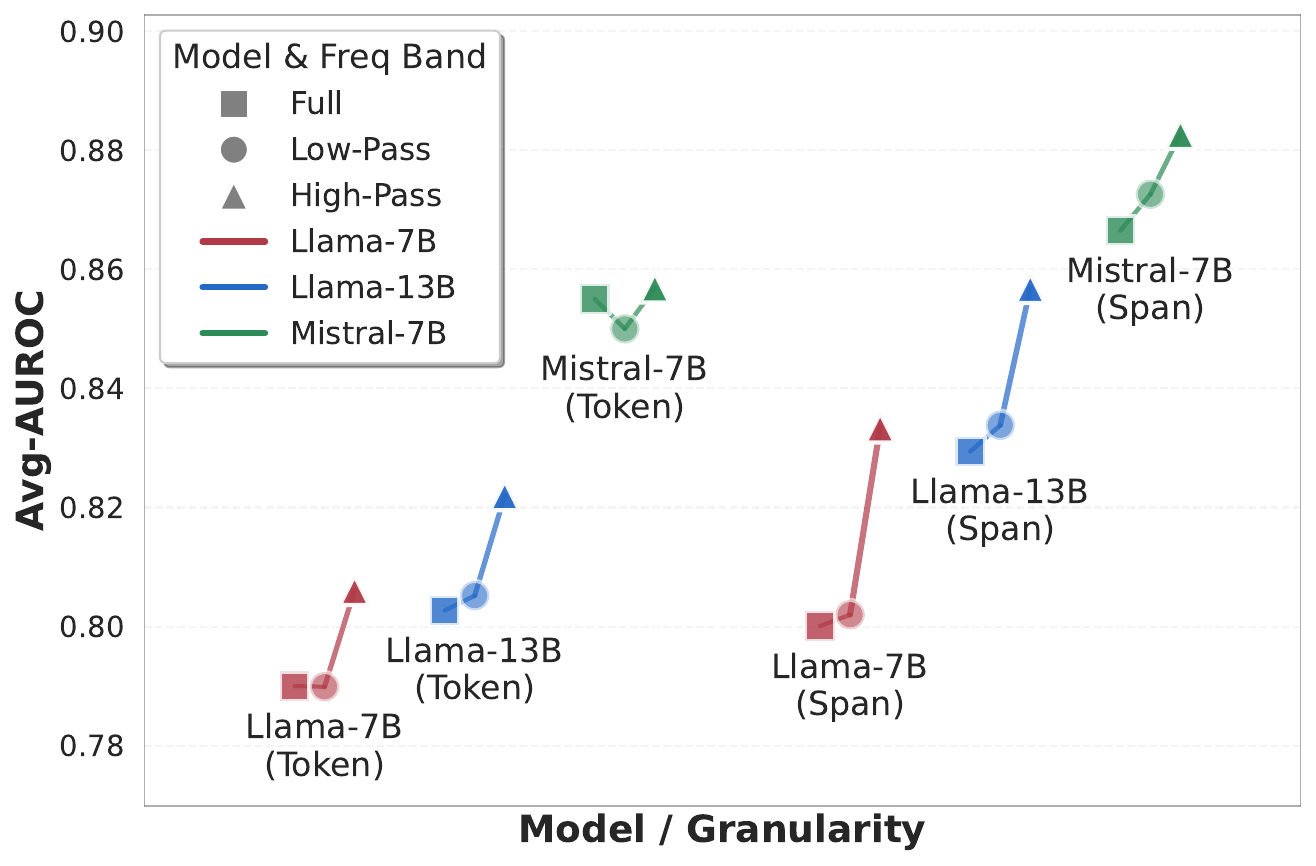}}
    \caption{ Comparing full-, low-, and high-pass Fourier attention features. Average AUROC across models under token- and span-level evaluation settings.
    }
    \label{fig:low_high}
  \end{center}
  \vskip -0.3in
\end{figure}

\subsection{Ablation Study on High/Low-pass Components}

We further analyze the contribution of different frequency bands by comparing low-pass, high-pass, and full-spectrum attention features, as shown in \autoref{fig:low_high}. 
Across all models and both token- and span-level evaluation settings, high-pass components achieve the strongest performance, while low-pass and full-spectrum features perform comparably. 
This pattern suggests that, without explicitly isolating high-frequency variation, attention signals are largely dominated by low-frequency components that offer limited discriminative power for hallucination detection.

Low-frequency components mainly reflect smooth, global alignment patterns common to both grounded and hallucinated outputs, whereas high-pass components emphasize rapid local irregularities in attention that are more closely associated with unstable generation behavior. 
Additional analyses on Fourier frequency cutoffs and Wavelet detail levels are provided in \autoref{fig:ablation_cutoff}, \autoref{tab:wavelet_sliding1_full}, and \autoref{tab:wavelet_sliding8_full}.

\subsection{Understanding Frequency-Aware Features}
\label{sec:analysis}

To better understand how frequency-aware attention features contribute to hallucination detection, we analyze the learned linear classifier from three perspectives: layer-level importance, head-level sparsity, and the relative contribution of context versus generated attention.

\subsubsection{Layer-wise Importance}
As shown in \autoref{fig:layerwise}, we report the average absolute classifier coefficients assigned to attention heads in each layer. The importance of high-frequency attention signals is clearly non-uniform across model depth. 
For Fourier-based features, importance peaks in the middle layers, with a pronounced maximum around layer 14. These layers are commonly associated with higher-level semantic processing and factual reasoning, suggesting that hallucination-related attention instability emerges most prominently once the model gradually transitions from surface-level decoding to semantic composition \cite{dola_Chuang_2024}. 

\begin{figure}[t]
  \begin{center}
    \centerline{\includegraphics[width=0.9\columnwidth]{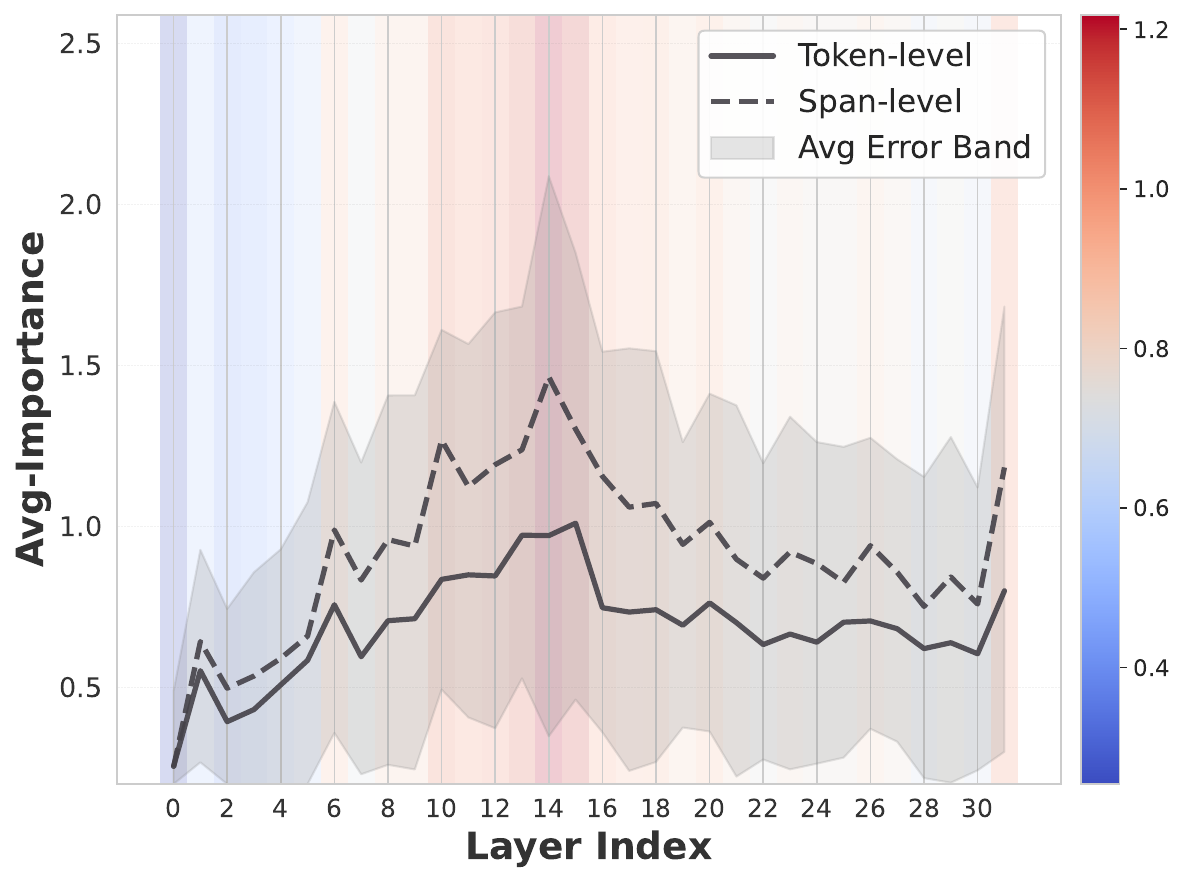}}
    \caption{
      Layer-wise importance of high-frequency Fourier-high attention features for LLaMA-7B. 
      }
    \label{fig:layerwise}
  \end{center}
  \vskip -0.3in
\end{figure}

Across almost all layers, span-level (dashed line) assigns higher importance weights than token-level detection (solid line). This indicates that aggregating attention signals over longer spans amplifies structured high-frequency variation, rather than weakening it, further supporting the robustness of frequency-aware features under coarser supervision.

\subsubsection{Head-wise Sparsity}
Beyond layer-level trends, we examine whether detection signals are broadly distributed across heads or concentrated in a small subset. To this end, we perform detection using only the Top-$k$ attention heads ranked by their average absolute classifier coefficients.

\begin{table}[htbp]
\caption{AUROC Results for LLaMA-7B: Original vs. Top-$k$ head only performance.}
\label{tab:head_sparsity}
\centering
\small
\setlength{\tabcolsep}{4.0pt}
\setlength{\dashlinedash}{0.6pt}
\setlength{\dashlinegap}{1.4pt}

\begin{tabular}{l cccc}
\toprule

\multicolumn{5}{c}{\textbf{RagTruth-Avg}} \\
\cmidrule(lr){1-5}
\multirow{2}{*}{\textbf{Method}} 
& \multirow{2}{*}{\textbf{Original}} 
& \multicolumn{3}{c}{\textbf{Top-$k$ heads}} \\
\cmidrule(lr){3-5}
& & \textit{k = 100} & \textit{50} & \textit{10} \\
\midrule
\quad Laplacian & 0.8003 & 0.7824 & 0.7666 & 0.7131 \\
\quad Wavelet-high   & 0.8087 & 0.7900 & 0.7708 & 0.7193 \\
\quad Fourier-high   & 0.8205 & 0.7915 & 0.7684 & 0.6995 \\

\midrule 
\multicolumn{5}{c}{\textbf{HalluRAG}} \\
\cmidrule(lr){1-5}
\multirow{2}{*}{\textbf{Method}} 
& \multirow{2}{*}{\textbf{Original}} 
& \multicolumn{3}{c}{\textbf{Top-$k$ heads}} \\
\cmidrule(lr){3-5}
& & \textit{k = 100} & \textit{50} & \textit{10} \\
\midrule
\quad Laplacian & 0.7429 & 0.7222 & 0.6986 & 0.6479 \\
\quad Wavelet-high   & 0.7550 & 0.7229 & 0.7016 & 0.6562 \\
\quad Fourier-high   & 0.7629 & 0.7229 & 0.6814 & 0.6373 \\

\bottomrule
\end{tabular}

\end{table}

As shown in \autoref{tab:head_sparsity}, a remarkably small fraction of heads accounts for most of the detection performance.
Using only the top 100 heads (less than 10\% of all heads in LLaMA-7B) recovers over 95\% of the original AUROC across operators and datasets.
This pronounced sparsity indicates that hallucination-related high-frequency attention variation is not a diffuse property shared across all heads.
Instead, a small subset of heads consistently exhibits strong sensitivity to attention instability, suggesting that these heads act as implicit internal indicators of ungrounded generation.

\subsubsection{Context v.s. Generated Attention Contribution}
\label{sec:context_generated_analysis}

A critical design choice in our method is incorporating high-frequency attention signals from both source context tokens and previously generated tokens during generation. 
We analyze the learned classifier to assess the relative importance of these two sources in practice.

As shown in \autoref{fig:context_generated}, frequency-aware signals derived from context-token attention are consistently more informative for hallucination detection than those from generated-token attention across all spectral operators.
Among them, the Fourier-based method exhibits the largest context–generated importance gap, which aligns with its superior overall performance across tasks and models, suggesting that capturing frequency-aware patterns in context attention contributes to more robust hallucination detection.

This asymmetry aligns closely with the core motivation of Lookback-lens \cite{chuang2024lookback}, which emphasizes the role of backward attention to contextual evidence in grounded generation. 
Our results extend this view by showing that not only the presence of backward attention, but also its stability and variation, play a critical role in practice. This interpretation is further supported by more ablation results in \autoref{tab:ablation_context_new_full}, where removing only context-based features leads to substantially larger performance degradation than removing only generated-token features.

\section{Broader Related Work}

In this section, we briefly situate our work within the broader literature on hallucination detection in LLMs. 

\begin{figure}[t]
  \begin{center}
    \centerline{\includegraphics[width=0.85\columnwidth]{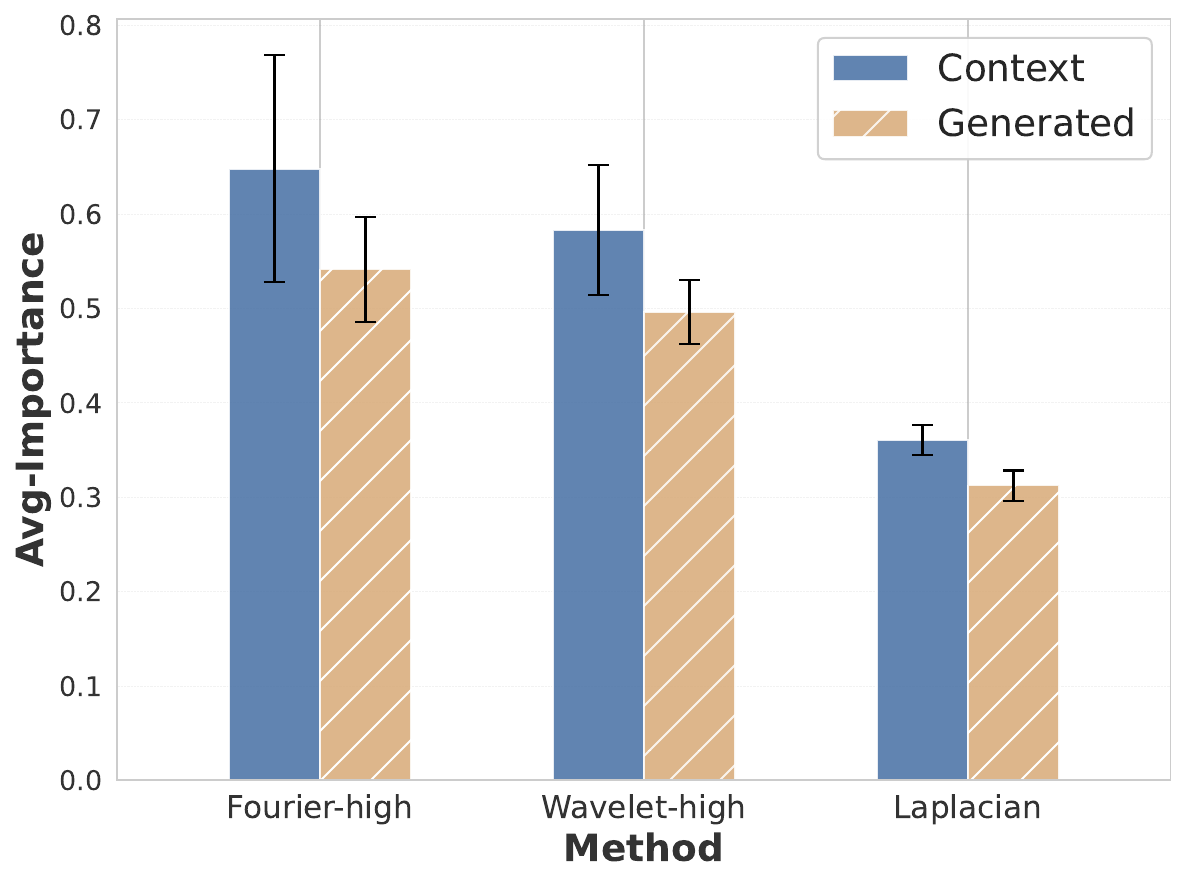}}
    \caption{
      Average absolute classifier importance assigned to context-attention and generated-attention features over all examined models and datasets.
    }
    \label{fig:context_generated}
  \end{center}
  \vskip -0.3in
\end{figure}

A substantial line of work detects hallucinations through \textbf{external knowledge verification}, validating model outputs against curated corpora, knowledge bases, or web search results \cite{min-etal-2023-factscore, feng2023factkb, chern2023factool}. While effective when reliable evidence is available, these approaches are constrained by knowledge coverage, retrieval quality, and inference-time latency, limiting their applicability in on-the-fly detection settings.

In contrast, \textbf{intrinsic detection} methods rely solely on model-internal signals and are lightweight enough to operate during decoding. Prior work has explored semantic consistency between input and output, uncertainty and self-consistency across generations, and generation-time statistics such as token probabilities or hidden-state geometry \citep{qi2025surveyautomatichallucinationevaluation}. 
More recent studies probe attention mechanisms as intrinsic indicators of hallucination \cite{sun2024redeep, chuang2024lookback, campbell2023localizing}, which mainly characterize where attention is allocated or how concentrated it is, but largely overlook fine-grained local variation within the token sequence.

Using \textbf{frequency-aware} tools offers an additional perspective for examining internal model dynamics \cite{kiruluta2025attentionatomsspectraldictionary, li2025llmhallucinationdetectionfast}. By modeling internal signals as discrete sequences, frequency-aware representations capture structural patterns that are difficult to characterize using token-level statistics alone. Attention distributions are also typically sparse and highly non-uniform across token positions \cite{nawrot2025sparsefrontiersparseattention, gu2025attentionsinkemergeslanguage}, providing a natural basis for frequency-based decomposition, where abrupt allocation correspond to high-frequency components and global distribution patterns align with low-frequency trends.

Building on this perspective, our work introduces a frequency-aware characterization of attention variation as an intrinsic signal for hallucination detection.

\section{Conclusion}

This work presents a frequency-aware perspective for analyzing hallucination in LLMs. We show that hallucinated generation is associated with high-frequency variation in attention patterns, revealing a structural property of model behavior that is not captured by attention allocation, transition, or aggregate uncertainty measures. This provides a principled way to distinguish grounded from ungrounded generation based on how attention varies across tokens.

Building on this insight, we introduce a frequency-aware attention modeling framework that extracts high-frequency attention signals using spectral operators. Across multiple models and tasks, our approach improves hallucination detection at both token and span levels, indicating its potential as a generalizable and task-agnostic intrinsic signal for analysis.
Overall, our findings suggest that hallucinations are reflected not only in where attention is placed, but also in how it varies across tokens over time, highlighting frequency-aware analysis as a promising direction for diagnosing and mitigating unreliable LLM generation.

\section*{Acknowledgements}
SQ is funded by a PhD scholarship provided by K-CSC. YH was supported by the UK Engineering and Physical Sciences Research Council (EPSRC) through a Turing AI Fellowship (grant no. EP/V020579/1, EP/V020579/2). 
The authors acknowledge the use of King’s Computational Research, Engineering and Technology Environment (CREATE) at King’s College London.



\section*{Impact Statement}
This paper presents a frequency-aware analysis of internal attention in LLMs for detecting hallucinated generation.
The goal of this work is to advance the understanding and evaluation of LLMs by providing intrinsic signals that can help identify unreliable outputs.

There are potential positive societal impacts of this work, including improved reliability and transparency of language model deployments in applications where factual correctness is important.
By relying solely on model-internal signals, the proposed approach avoids dependence on external knowledge sources and may be used as a lightweight evaluation tool.

At the same time, this work does not aim to prevent hallucination or provide guarantees of correctness, and the signals identified should not be used as the sole basis for high-stakes decisions.
As with other analysis and detection methods, responsible deployment requires appropriate safeguards and complementary evaluation mechanisms.
Overall, we believe the ethical and societal implications of this work are aligned with well-established goals in machine learning research, and do not warrant further specific discussion.





\bibliography{main}

@inproceedings{Kryscinski2020,
  author        = {Wojciech Kryscinski and
                   Bryan McCann and
                   Caiming Xiong and
                   Richard Socher},
  bibsource     = {dblp computer science bibliography, https://dblp.org},
  booktitle     = {Proceedings of the 2020 Conference on Empirical Methods in Natural
                   Language Processing, {EMNLP} 2020, Online, November 16-20, 2020},
  doi           = {10.18653/V1/2020.EMNLP-MAIN.750},
  editor        = {Bonnie Webber and
                   Trevor Cohn and
                   Yulan He and
                   Yang Liu},
  pages         = {9332--9346},
  publisher     = {Association for Computational Linguistics},
  title         = {{E}valuating the {F}actual {C}onsistency of {A}bstractive {T}ext {S}ummarization},
  url           = {https://doi.org/10.18653/v1/2020.emnlp-main.750},
  year          = {2020}
}

@book{stein2011fourier,
  title={Fourier analysis: an introduction},
  author={Stein, Elias M and Shakarchi, Rami},
  volume={1},
  year={2011},
  publisher={Princeton University Press}
}

@article{Laban2021,
  author        = {Philippe Laban and
                   Tobias Schnabel and
                   Paul N. Bennett and
                   Marti A. Hearst},
  bibsource     = {dblp computer science bibliography, https://dblp.org},
  doi           = {10.1162/TACL\_A\_00453},
  journal       = {Trans. Assoc. Comput. Linguistics},
  pages         = {163--177},
  title         = {{S}umma{C}: {R}e-Visiting {N}{L}{I}-based {M}odels for {I}nconsistency {D}etection in {S}ummarization},
  url           = {https://doi.org/10.1162/tacl\_a\_00453},
  volume        = {10},
  year          = {2022}
}

@inproceedings{Wiegreffe2019,
  author        = {Sarah Wiegreffe and
                   Yuval Pinter},
  bibsource     = {dblp computer science bibliography, https://dblp.org},
  booktitle     = {Proceedings of the 2019 Conference on Empirical Methods in Natural
                   Language Processing and the 9th International Joint Conference on
                   Natural Language Processing, {EMNLP-IJCNLP} 2019, Hong Kong, China,
                   November 3-7, 2019},
  doi           = {10.18653/V1/D19-1002},
  editor        = {Kentaro Inui and
                   Jing Jiang and
                   Vincent Ng and
                   Xiaojun Wan},
  pages         = {11--20},
  publisher     = {Association for Computational Linguistics},
  title         = {{A}ttention is not not {E}xplanation},
  url           = {https://doi.org/10.18653/v1/D19-1002},
  year          = {2019}
}

@article{Feng2023,
  author        = {Feng, Jiazhan and Yang, Linfeng and Li, Zeyu and Zhang, Jiazheng and others},
  journal       = {arXiv preprint arXiv:2308.05374},
  title         = {{T}rustworthy {L}{L}{M}s: {A} {S}urvey and {G}uideline for {E}valuating {L}arge {L}anguage {M}odels' {A}lignment},
  year          = {2023}
}

@misc{qi2025surveyautomatichallucinationevaluation,
  archiveprefix = {arXiv},
  author        = {Siya Qi and Lin Gui and Yulan He and Zheng Yuan},
  eprint        = {2404.12041},
  primaryclass  = {cs.CL},
  title         = {{A} {S}urvey of {A}utomatic {H}allucination {E}valuation on {N}atural {L}anguage {G}eneration},
  url           = {https://arxiv.org/abs/2404.12041},
  year          = {2025}
}

@article{li2025llmhallucinationdetectionfast,
  author        = {Jinxin Li and
                   Gang Tu and
                   ShengYu Cheng and
                   Junjie Hu and
                   Jinting Wang and
                   Rui Chen and
                   Zhilong Zhou and
                   Dongbo Shan},
  bibsource     = {dblp computer science bibliography, https://dblp.org},
  doi           = {10.48550/ARXIV.2509.13154},
  eprint        = {2509.13154},
  eprinttype    = {arXiv},
  journal       = {CoRR},
  title         = {{LLM} {H}allucination {D}etection: {A} {F}ast {F}ourier {T}ransform {M}ethod {B}ased on {H}idden {L}ayer {T}emporal {S}ignals},
  url           = {https://doi.org/10.48550/arXiv.2509.13154},
  volume        = {abs/2509.13154},
  year          = {2025}
}

@inproceedings{li2024inferencetimeinterventionelicitingtruthful,
  author        = {Kenneth Li and
                   Oam Patel and
                   Fernanda B. Vi{\'{e}}gas and
                   Hanspeter Pfister and
                   Martin Wattenberg},
  bibsource     = {dblp computer science bibliography, https://dblp.org},
  booktitle     = {Advances in Neural Information Processing Systems 36: Annual Conference
                   on Neural Information Processing Systems 2023, NeurIPS 2023, New Orleans,
                   LA, USA, December 10 - 16, 2023},
  editor        = {Alice Oh and
                   Tristan Naumann and
                   Amir Globerson and
                   Kate Saenko and
                   Moritz Hardt and
                   Sergey Levine},
  title         = {{I}nference-Time {I}ntervention: {E}liciting {T}ruthful {A}nswers from a {L}anguage {M}odel},
  year          = {2023}
}

@article{campbell2023localizing,
  author        = {James Campbell and
                   Richard Ren and
                   Phillip Guo},
  bibsource     = {dblp computer science bibliography, https://dblp.org},
  doi           = {10.48550/ARXIV.2311.15131},
  eprint        = {2311.15131},
  eprinttype    = {arXiv},
  journal       = {CoRR},
  title         = {{L}ocalizing {L}ying in {L}lama: {U}nderstanding {I}nstructed {D}ishonesty on {T}rue-False {Q}uestions {T}hrough {P}rompting, {P}robing, and {P}atching},
  url           = {https://doi.org/10.48550/arXiv.2311.15131},
  volume        = {abs/2311.15131},
  year          = {2023}
}

@article{ridder2025halluragdatasetdetectingcloseddomain,
  author        = {Fabian Ridder and
                   Malte Schilling},
  bibsource     = {dblp computer science bibliography, https://dblp.org},
  doi           = {10.48550/ARXIV.2412.17056},
  eprint        = {2412.17056},
  eprinttype    = {arXiv},
  journal       = {CoRR},
  title         = {{T}he {H}allu{R}{A}{G} {D}ataset: {D}etecting {C}losed-Domain {H}allucinations in {RAG} {A}pplications {U}sing an {L}{L}{M}'s {I}nternal {S}tates},
  url           = {https://doi.org/10.48550/arXiv.2412.17056},
  volume        = {abs/2412.17056},
  year          = {2024}
}

@article{huang2025dynamicattentionguidedcontextdecoding,
  author        = {Yanwen Huang and
                   Yong Zhang and
                   Ning Cheng and
                   Zhitao Li and
                   Shaojun Wang and
                   Jing Xiao},
  bibsource     = {dblp computer science bibliography, https://dblp.org},
  doi           = {10.48550/ARXIV.2501.01059},
  eprint        = {2501.01059},
  eprinttype    = {arXiv},
  journal       = {CoRR},
  title         = {{D}ynamic {A}ttention-Guided {C}ontext {D}ecoding for {M}itigating {C}ontext {F}aithfulness {H}allucinations in {L}arge {L}anguage {M}odels},
  url           = {https://doi.org/10.48550/arXiv.2501.01059},
  volume        = {abs/2501.01059},
  year          = {2025}
}

@inproceedings{Snyder2024earlyDetection,
  author        = {Ben Snyder and
                   Marius Moisescu and
                   Muhammad Bilal Zafar},
  bibsource     = {dblp computer science bibliography, https://dblp.org},
  booktitle     = {Proceedings of the 30th {ACM} {SIGKDD} Conference on Knowledge Discovery
                   and Data Mining, {KDD} 2024, Barcelona, Spain, August 25-29, 2024},
  doi           = {10.1145/3637528.3671796},
  editor        = {Ricardo Baeza{-}Yates and
                   Francesco Bonchi},
  pages         = {2721--2732},
  publisher     = {{ACM}},
  title         = {{O}n {E}arly {D}etection of {H}allucinations in {F}actual {Q}uestion {A}nswering},
  url           = {https://doi.org/10.1145/3637528.3671796},
  year          = {2024}
}

@inproceedings{ogasa2025hallucinatedMulti-View,
  author        = {Ogasa, Yuya and Arase, Yuki},
  booktitle     = {Proceedings of the 14th Joint Conference on Lexical and Computational Semantics (* SEM 2025)},
  pages         = {381--394},
  title         = {{H}allucinated {S}pan {D}etection with {M}ulti-View {A}ttention {F}eatures},
  year          = {2025}
}

@article{vazhentsev2025uncertaintyawareattentionheadsefficient,
  author        = {Artem Vazhentsev and
                   Lyudmila Rvanova and
                   Gleb Kuzmin and
                   Ekaterina Fadeeva and
                   Ivan Lazichny and
                   Alexander Panchenko and
                   Maxim Panov and
                   Timothy Baldwin and
                   Mrinmaya Sachan and
                   Preslav Nakov and
                   Artem Shelmanov},
  bibsource     = {dblp computer science bibliography, https://dblp.org},
  doi           = {10.48550/ARXIV.2505.20045},
  eprint        = {2505.20045},
  eprinttype    = {arXiv},
  journal       = {CoRR},
  title         = {{U}ncertainty-Aware {A}ttention {H}eads: {E}fficient {U}nsupervised {U}ncertainty {Q}uantification for {L}{L}{M}s},
  url           = {https://doi.org/10.48550/arXiv.2505.20045},
  volume        = {abs/2505.20045},
  year          = {2025}
}

@article{Cooley1969TheFF,
  author        = {Cooley, James W. and Lewis, Peter A. W. and Welch, Peter D.},
  doi           = {10.1109/TE.1969.4320436},
  journal       = {IEEE Transactions on Education},
  number        = {1},
  pages         = {27-34},
  title         = {{T}he {F}ast {F}ourier {T}ransform and {I}ts {A}pplications},
  volume        = {12},
  year          = {1969}
}

@article{wavelet1989,
  author        = {Mallat, S.G.},
  doi           = {10.1109/34.192463},
  journal       = {IEEE Transactions on Pattern Analysis and Machine Intelligence},
  number        = {7},
  pages         = {674-693},
  title         = {{A} theory for multiresolution signal decomposition: the wavelet representation},
  volume        = {11},
  year          = {1989}
}

@article{sandryhaila2014discrete,
  author        = {Aliaksei Sandryhaila and
                   Jos{\'{e}} M. F. Moura},
  bibsource     = {dblp computer science bibliography, https://dblp.org},
  doi           = {10.1109/TSP.2014.2321121},
  journal       = {{IEEE} Trans. Signal Process.},
  number        = {12},
  pages         = {3042--3054},
  title         = {{D}iscrete {S}ignal {P}rocessing on {G}raphs: {F}requency {A}nalysis},
  url           = {https://doi.org/10.1109/TSP.2014.2321121},
  volume        = {62},
  year          = {2014}
}

@inproceedings{dola_Chuang_2024,
  author        = {Yung{-}Sung Chuang and
                   Yujia Xie and
                   Hongyin Luo and
                   Yoon Kim and
                   James R. Glass and
                   Pengcheng He},
  bibsource     = {dblp computer science bibliography, https://dblp.org},
  booktitle     = {The Twelfth International Conference on Learning Representations,
                   {ICLR} 2024, Vienna, Austria, May 7-11, 2024},
  publisher     = {OpenReview.net},
  title         = {{D}o{L}a: {D}ecoding by {C}ontrasting {L}ayers {I}mproves {F}actuality in {L}arge {L}anguage {M}odels},
  url           = {https://openreview.net/forum?id=Th6NyL07na},
  year          = {2024}
}

@inproceedings{gong-etal-2024-damro,
  author        = {Xuan Gong and
                   Tianshi Ming and
                   Xinpeng Wang and
                   Zhihua Wei},
  bibsource     = {dblp computer science bibliography, https://dblp.org},
  booktitle     = {Proceedings of the 2024 Conference on Empirical Methods in Natural
                   Language Processing, {EMNLP} 2024, Miami, FL, USA, November 12-16,
                   2024},
  doi           = {10.18653/V1/2024.EMNLP-MAIN.439},
  editor        = {Yaser Al{-}Onaizan and
                   Mohit Bansal and
                   Yun{-}Nung Chen},
  pages         = {7696--7712},
  publisher     = {Association for Computational Linguistics},
  title         = {{DAMRO:} {D}ive into the {A}ttention {M}echanism of {LVLM} to {R}educe {O}bject {H}allucination},
  url           = {https://doi.org/10.18653/v1/2024.emnlp-main.439},
  year          = {2024}
}

@article{RefChecker_2024_hu,
  author        = {Xiangkun Hu and
                   Dongyu Ru and
                   Lin Qiu and
                   Qipeng Guo and
                   Tianhang Zhang and
                   Yang Xu and
                   Yun Luo and
                   Pengfei Liu and
                   Yue Zhang and
                   Zheng Zhang},
  bibsource     = {dblp computer science bibliography, https://dblp.org},
  doi           = {10.48550/ARXIV.2405.14486},
  eprint        = {2405.14486},
  eprinttype    = {arXiv},
  journal       = {CoRR},
  title         = {{R}ef{C}hecker: {R}eference-based {F}ine-grained {H}allucination {C}hecker and {B}enchmark for {L}arge {L}anguage {M}odels},
  url           = {https://doi.org/10.48550/arXiv.2405.14486},
  volume        = {abs/2405.14486},
  year          = {2024}
}

@article{nawrot2025sparsefrontiersparseattention,
  author        = {Piotr Nawrot and
                   Robert Li and
                   Renjie Huang and
                   Sebastian Ruder and
                   Kelly Marchisio and
                   Edoardo M. Ponti},
  bibsource     = {dblp computer science bibliography, https://dblp.org},
  doi           = {10.48550/ARXIV.2504.17768},
  eprint        = {2504.17768},
  eprinttype    = {arXiv},
  journal       = {CoRR},
  title         = {{T}he {S}parse {F}rontier: {S}parse {A}ttention {T}rade-offs in {T}ransformer {L}{L}{M}s},
  url           = {https://doi.org/10.48550/arXiv.2504.17768},
  volume        = {abs/2504.17768},
  year          = {2025}
}

@article{kiruluta2025attentionatomsspectraldictionary,
  author        = {Andrew Kiruluta},
  bibsource     = {dblp computer science bibliography, https://dblp.org},
  doi           = {10.48550/ARXIV.2505.00033},
  eprint        = {2505.00033},
  eprinttype    = {arXiv},
  journal       = {CoRR},
  title         = {{F}rom {A}ttention to {A}toms: {S}pectral {D}ictionary {L}earning for {F}ast, {I}nterpretable {L}anguage {M}odels},
  url           = {https://doi.org/10.48550/arXiv.2505.00033},
  volume        = {abs/2505.00033},
  year          = {2025}
}

@inproceedings{gu2025attentionsinkemergeslanguage,
  author        = {Xiangming Gu and
                   Tianyu Pang and
                   Chao Du and
                   Qian Liu and
                   Fengzhuo Zhang and
                   Cunxiao Du and
                   Ye Wang and
                   Min Lin},
  bibsource     = {dblp computer science bibliography, https://dblp.org},
  booktitle     = {The Thirteenth International Conference on Learning Representations,
                   {ICLR} 2025, Singapore, April 24-28, 2025},
  publisher     = {OpenReview.net},
  title         = {{W}hen {A}ttention {S}ink {E}merges in {L}anguage {M}odels: {A}n {E}mpirical {V}iew},
  url           = {https://openreview.net/forum?id=78Nn4QJTEN},
  year          = {2025}
}

@book{oppenheim1997signals,
  author        = {Oppenheim, Alan V. and Schafer, Alan S.},
  edition       = {2},
  publisher     = {Prentice Hall},
  title         = {{S}ignals and {S}ystems},
  year          = {1997}
}

@inproceedings{chen2024inside,
  author        = {Chao Chen and
                   Kai Liu and
                   Ze Chen and
                   Yi Gu and
                   Yue Wu and
                   Mingyuan Tao and
                   Zhihang Fu and
                   Jieping Ye},
  bibsource     = {dblp computer science bibliography, https://dblp.org},
  booktitle     = {The Twelfth International Conference on Learning Representations,
                   {ICLR} 2024, Vienna, Austria, May 7-11, 2024},
  publisher     = {OpenReview.net},
  title         = {{INSIDE:} {L}{L}{M}s' {I}nternal {S}tates {R}etain the {P}ower of {H}allucination {D}etection},
  url           = {https://openreview.net/forum?id=Zj12nzlQbz},
  year          = {2024}
}

@article{chern2023factool,
  author        = {I{-}Chun Chern and
                   Steffi Chern and
                   Shiqi Chen and
                   Weizhe Yuan and
                   Kehua Feng and
                   Chunting Zhou and
                   Junxian He and
                   Graham Neubig and
                   Pengfei Liu},
  doi           = {10.48550/ARXIV.2307.13528},
  journal       = {CoRR},
  title         = {{F}ac{T}ool: {F}actuality {D}etection in {G}enerative {AI} - {A} {T}ool {A}ugmented {F}ramework for {M}ulti-Task and {M}ulti-Domain {S}cenarios},
  year          = {2023}
}

@inproceedings{chuang2024lookback,
  author        = {Yung{-}Sung Chuang and
                   Linlu Qiu and
                   Cheng{-}Yu Hsieh and
                   Ranjay Krishna and
                   Yoon Kim and
                   James R. Glass},
  booktitle     = {EMNLP},
  pages         = {1419--1436},
  title         = {{L}ookback {L}ens: {D}etecting and {M}itigating {C}ontextual {H}allucinations in {L}arge {L}anguage {M}odels {U}sing {O}nly {A}ttention {M}aps},
  year          = {2024}
}

@article{deemter-2024-pitfalls,
  author        = {Kees van Deemter},
  bibsource     = {dblp computer science bibliography, https://dblp.org},
  doi           = {10.1162/COLI\_A\_00509},
  journal       = {Comput. Linguistics},
  number        = {2},
  pages         = {807--816},
  title         = {{T}he {P}itfalls of {D}efining {H}allucination},
  url           = {https://doi.org/10.1162/coli\_a\_00509},
  volume        = {50},
  year          = {2024}
}

@inproceedings{feng2023factkb,
  address       = {Singapore},
  author        = {Feng, Shangbin  and
                   Balachandran, Vidhisha  and
                   Bai, Yuyang  and
                   Tsvetkov, Yulia},
  booktitle     = {Proceedings of the 2023 Conference on Empirical Methods in Natural Language Processing},
  doi           = {10.18653/v1/2023.emnlp-main.59},
  editor        = {Bouamor, Houda  and
                   Pino, Juan  and
                   Bali, Kalika},
  pages         = {933--952},
  publisher     = {Association for Computational Linguistics},
  title         = {{F}act{KB}: {G}eneralizable {F}actuality {E}valuation using {L}anguage {M}odels {E}nhanced with {F}actual {K}nowledge},
  url           = {https://aclanthology.org/2023.emnlp-main.59},
  year          = {2023}
}

@article{hu2024lrp4ragdetectinghallucinationsretrievalaugmented,
  author        = {Haichuan Hu and
                   Yuhan Sun and
                   Quanjun Zhang},
  doi           = {10.48550/ARXIV.2408.15533},
  journal       = {CoRR},
  title         = {{LRP4RAG:} {D}etecting {H}allucinations in {R}etrieval-Augmented {G}eneration via {L}ayer-wise {R}elevance {P}ropagation},
  year          = {2024}
}

@article{ji2023survey,
  author        = {Ziwei Ji and
                   Nayeon Lee and
                   Rita Frieske and
                   Tiezheng Yu and
                   Dan Su and
                   Yan Xu and
                   Etsuko Ishii and
                   Yejin Bang and
                   Andrea Madotto and
                   Pascale Fung},
  doi           = {10.1145/3571730},
  journal       = {{ACM} Comput. Surv.},
  pages         = {248:1--248:38},
  title         = {{S}urvey of {H}allucination in {N}atural {L}anguage {G}eneration},
  year          = {2023}
}

@inproceedings{manakul-etal-2023-selfcheckgpt,
  address       = {Singapore},
  author        = {Manakul, Potsawee  and
                   Liusie, Adian  and
                   Gales, Mark},
  booktitle     = {Proceedings of the 2023 Conference on Empirical Methods in Natural Language Processing},
  doi           = {10.18653/v1/2023.emnlp-main.557},
  editor        = {Bouamor, Houda  and
                   Pino, Juan  and
                   Bali, Kalika},
  pages         = {9004--9017},
  publisher     = {Association for Computational Linguistics},
  title         = {{S}elf{C}heck{GPT}: {Z}ero-Resource {B}lack-Box {H}allucination {D}etection for {G}enerative {L}arge {L}anguage {M}odels},
  url           = {https://aclanthology.org/2023.emnlp-main.557},
  year          = {2023}
}

@inproceedings{min-etal-2023-factscore,
  address       = {Singapore},
  author        = {Min, Sewon  and
                   Krishna, Kalpesh  and
                   Lyu, Xinxi  and
                   Lewis, Mike  and
                   Yih, Wen-tau  and
                   Koh, Pang  and
                   Iyyer, Mohit  and
                   Zettlemoyer, Luke  and
                   Hajishirzi, Hannaneh},
  booktitle     = {Proceedings of the 2023 Conference on Empirical Methods in Natural Language Processing},
  doi           = {10.18653/v1/2023.emnlp-main.741},
  editor        = {Bouamor, Houda  and
                   Pino, Juan  and
                   Bali, Kalika},
  pages         = {12076--12100},
  publisher     = {Association for Computational Linguistics},
  title         = {{FA}ct{S}core: {F}ine-grained {A}tomic {E}valuation of {F}actual {P}recision in {L}ong {F}orm {T}ext {G}eneration},
  url           = {https://aclanthology.org/2023.emnlp-main.741},
  year          = {2023}
}

@inproceedings{niu-etal-2024-ragtruth,
  author        = {Niu, Cheng  and Wu, Yuanhao  and Zhu, Juno  and Xu, Siliang  and Shum, KaShun and Zhong, Randy  and Song, Juntong  and Zhang, Tong},
  booktitle     = {ACL},
  doi           = {10.18653/v1/2024.acl-long.585},
  pages         = {10862--10878},
  title         = {{RAGT}ruth: {A} {H}allucination {C}orpus for {D}eveloping {T}rustworthy {R}etrieval-Augmented {L}anguage {M}odels},
  year          = {2024}
}

@inproceedings{sriramananllm,
  author        = {Gaurang Sriramanan and
                   Siddhant Bharti and
                   Vinu Sankar Sadasivan and
                   Shoumik Saha and
                   Priyatham Kattakinda and
                   Soheil Feizi},
  bibsource     = {dblp computer science bibliography, https://dblp.org},
  booktitle     = {Advances in Neural Information Processing Systems 38: Annual Conference
                   on Neural Information Processing Systems 2024, NeurIPS 2024, Vancouver,
                   BC, Canada, December 10 - 15, 2024},
  editor        = {Amir Globersons and
                   Lester Mackey and
                   Danielle Belgrave and
                   Angela Fan and
                   Ulrich Paquet and
                   Jakub M. Tomczak and
                   Cheng Zhang},
  title         = {{L}{L}{M}-{C}heck: {I}nvestigating {D}etection of {H}allucinations in {L}arge {L}anguage {M}odels},

  year          = {2024}
}

@inproceedings{sun2024redeep,
  author        = {Zhongxiang Sun and
                   Xiaoxue Zang and
                   Kai Zheng and
                   Jun Xu and
                   Xiao Zhang and
                   Weijie Yu and
                   Yang Song and
                   Han Li},
  booktitle     = {ICLR},
  title         = {{R}e{D}e{E}{P}: {D}etecting {H}allucination in {R}etrieval-Augmented {G}eneration via {M}echanistic {I}nterpretability},
  year          = {2025}
}

@article{donoho1998minimax,
  title={Minimax estimation via wavelet shrinkage},
  author={Donoho, David L and Johnstone, Iain M},
  journal={The annals of Statistics},
  volume={26},
  number={3},
  pages={879--921},
  year={1998},
  publisher={Institute of Mathematical Statistics}
}
\bibliographystyle{icml2026}

\newpage

\setcounter{figure}{0}
\renewcommand{\thefigure}{A\arabic{figure}}
\setcounter{table}{0}
\renewcommand{\thetable}{A\arabic{table}}

\appendix
\onecolumn

\section{A three-step proof that larger $K$ increases (a lower bound on) attention roughness}\label{app:three_step_proof}

\paragraph{Goal.}
We consider a \emph{single-layer} causal self-attention mechanism and measure, at a fixed prediction position $t\ge 2$, how \emph{rough} the attention weights over past positions are as a function of the number of mixture components $K$ in the input distribution.
Roughness is quantified by the $\ell_2$ adjacent-difference energy
\begin{equation}\label{eq:R_def_app}
R_t \;\triangleq\; \sum_{j=1}^{t-2}\big(\alpha_{t,j+1}-\alpha_{t,j}\big)^2,
\end{equation}
where $\alpha_{t,1:t-1}$ is the attention row (probability vector) used to predict token $t$.

We prove the claim via three sub-claims:
(i) adjacent tokens switch mixture components more often when $K$ is larger;
(ii) component switches induce larger adjacent differences in \emph{logits} $s_{t,j}$;
(iii) softmax transfers logit roughness into attention-weight roughness under mild non-degeneracy.

\subsection{Setup and assumptions}

\paragraph{Mixture of topic labels.}
Let $K\in \mathbb{N}$ denote the number of mixture components. Each position $j$ has an associated latent topic label
\[
c_j\in\{1,\dots,K\}.
\]
\begin{assumption}[i.i.d.\ uniform labels]\label{ass:labels_app}
$(c_j)_{j\ge 1}$ are i.i.d.\ and $\Pr(c_j=r)=1/K$ for all $r\in\{1,\dots,K\}$, which means we randomly assign a topic label for the tokens, and each topic has a unique Gaussian distribution. Then, the input token embedding under the view of Gaussian mixture model will follow a GMM:
For each $j$,
\[
x_j = \mu_{c_j} + \varepsilon_j,\qquad \varepsilon_j\stackrel{i.i.d.}{\sim}\mathcal N(0,\sigma^2 I_d),
\]
where $\mu_1,\dots,\mu_K\in\mathbb R^d$ are fixed means and $(\varepsilon_j)_j$ are independent of $(c_j)_j$.
\end{assumption}

Then, if we consider a single-layer causal attention at position $t$ (the position for the next generated token), and Fix $t\ge 2$. Let
\[
q_t = W_Q x_t,\qquad k_j = W_K x_j,\quad j\le t-1,
\]
and define logits (pre-softmax scores)
\[
s_{t,j}=\frac{q_t^\top k_j}{\sqrt{d_q}}=\frac{q_t^\top W_K x_j}{\sqrt{d_q}},\qquad j=1,\dots,t-1.
\]
To simplify the notation system, we further define the induced \emph{score direction} in the input space, where
\begin{equation}\label{eq:u_def_app}
u \;\triangleq\; \frac{W_K^\top q_t}{\sqrt{d_q}}\in\mathbb R^d,
\end{equation}
so that
\begin{equation}\label{eq:s_proj_app}
s_{t,j} = u^\top x_j.
\end{equation}
Attention weights from existing token towards next generation token are
\begin{equation}\label{eq:softmax_app}
\alpha_{t,j}=\frac{e^{s_{t,j}}}{\sum_{\ell=1}^{t-1}e^{s_{t,\ell}}},\qquad j=1,\dots,t-1.
\end{equation}

\begin{assumption}[Query conditioning]\label{ass:query_cond_app}
We condition on $x_t$ (equivalently on $q_t$), and treat $u$ in \eqref{eq:u_def_app} as deterministic. All expectations below are over $\{(c_j,\varepsilon_j)\}_{j\le t-1}$.
\end{assumption}

\paragraph{Separability along the score direction.}
\begin{assumption}[Separability]\label{ass:sep_app}
There exists $\Delta>0$ such that for all $r\neq r'$,
\[
|u^\top(\mu_{r'}-\mu_r)|\ge \Delta.
\]
\end{assumption}

In this assumption, we hope that the centroid of Gaussian should be still separable  after projection based on $W_k$ and $q_t$ in Equation \ref{eq:u_def_app}

\paragraph{Softmax non-degeneracy (prevents vanishing pair mass).}
Fix $j\in\{1,\dots,t-2\}$ and define
\[
\Delta s \triangleq s_{t,j+1}-s_{t,j},\qquad
m \triangleq \alpha_{t,j}+\alpha_{t,j+1}.
\]
\begin{assumption}[Softmax non-degeneracy]\label{ass:nd_app}
There exist constants $\eta\in(0,1)$, $B\le 2$, and $\kappa'>0$, independent of $K$, such that with
\[
E \triangleq \{m\ge \eta\}\cap\{|\Delta s|\le B\},
\]
we have
\[
\mathbb E\big[(\Delta s)^2\mathbf{1}_E\big]\ \ge\ \kappa'\,\mathbb E\big[(\Delta s)^2\big].
\]
\end{assumption}

Intuitively, Assumption \ref{ass:nd_app} says that local logit fluctuations happen in parts of the sequence that the softmax actually ``looks at'': with nontrivial frequency, the adjacent pair $(j,j+1)$ receives at least some fixed amount of probability mass, and on those occasions the logit gap is not excessively large. This prevents the situation where logits vary wildly only at positions whose attention weights are almost always zero, in which case attention differences would remain small regardless of logit roughness.

\subsection{Sub-claim 1: Adjacent component switching probability increases with $K$}

\begin{lemma}[Switch probability]\label{lem:switch_app}
Under Assumption~\ref{ass:labels_app}, for any $j\ge 1$, the probability of
\[
\Pr(c_{j+1}\neq c_j)=1-\frac{1}{K}.
\]
\end{lemma}

\begin{proof}
By independence and uniformity,
\[
\Pr(c_{j+1}=c_j)=\sum_{r=1}^K \Pr(c_j=r,\,c_{j+1}=r)
=\sum_{r=1}^K \Pr(c_j=r)\Pr(c_{j+1}=r)
=\sum_{r=1}^K \frac{1}{K}\cdot\frac{1}{K}
=\frac{1}{K}.
\]
Taking complements gives $\Pr(c_{j+1}\neq c_j)=1-\frac{1}{K}$.
\end{proof}

\subsection{Sub-claim 2: Adjacent logit difference energy increases with $K$}

\begin{lemma}[Logit adjacent-difference energy]\label{lem:logit_energy_app}
Under Assumptions~\ref{ass:labels_app}--\ref{ass:sep_app} and \ref{ass:query_cond_app}, for any $j\le t-2$,
\[
\mathbb E\big[(s_{t,j+1}-s_{t,j})^2\big]
\;\ge\;
2\sigma^2\|u\|_2^2 \;+\;\Big(1-\frac{1}{K}\Big)\Delta^2.
\]
\end{lemma}

\begin{proof}
Using \eqref{eq:s_proj_app} and $x_j=\mu_{c_j}+\varepsilon_j$,
\[
s_{t,j+1}-s_{t,j}
= u^\top(x_{j+1}-x_j)
= u^\top(\mu_{c_{j+1}}-\mu_{c_j}) + u^\top(\varepsilon_{j+1}-\varepsilon_j).
\]
Let
\[
A \triangleq u^\top(\mu_{c_{j+1}}-\mu_{c_j}),\qquad
B \triangleq u^\top(\varepsilon_{j+1}-\varepsilon_j),
\]
so that $s_{t,j+1}-s_{t,j}=A+B$.
Then
\[
\mathbb E[(A+B)^2]=\mathbb E[A^2] + 2\mathbb E[AB] + \mathbb E[B^2].
\]
The cross term vanishes: $A$ depends only on topic labels $(c_j,c_{j+1})$ while $B$ depends only on Gaussian noises $(\varepsilon_j,\varepsilon_{j+1})$, which are independent of labels, and $\mathbb E[B\mid c_j,c_{j+1}]=0$, hence $\mathbb E[AB]=0$.

For the noise term, $\varepsilon_{j+1}-\varepsilon_j\sim\mathcal N(0,2\sigma^2 I_d)$, so
\[
B \sim \mathcal N\big(0,\,2\sigma^2\|u\|_2^2\big)
\quad\Rightarrow\quad
\mathbb E[B^2]=2\sigma^2\|u\|_2^2.
\]
For the mean-jump term, if $c_{j+1}=c_j$ then $A=0$.
If $c_{j+1}\neq c_j$, Assumption~\ref{ass:sep_app} gives $A^2\ge \Delta^2$.
Therefore
\[
\mathbb E[A^2]
= \mathbb E\!\left[A^2\mathbf{1}\{c_{j+1}\neq c_j\}\right]
\ge \Delta^2\,\Pr(c_{j+1}\neq c_j)
= \Delta^2\Big(1-\frac{1}{K}\Big),
\]
where we used Lemma~\ref{lem:switch_app}.
Combining yields the stated lower bound.
\end{proof}

\subsection{Sub-claim 3: Softmax transfers logit roughness to attention roughness}

\begin{lemma}[Pairwise softmax difference identity]\label{lem:tanh_app}
Fix $t$ and $j\le t-2$. Let $m=\alpha_{t,j}+\alpha_{t,j+1}$ and $\Delta s=s_{t,j+1}-s_{t,j}$.
Then
\[
\alpha_{t,j+1}-\alpha_{t,j} \;=\; m\,\tanh\!\Big(\frac{\Delta s}{2}\Big).
\]
\end{lemma}

\begin{proof}
From \eqref{eq:softmax_app},
\[
\alpha_{t,j+1}-\alpha_{t,j}=\frac{e^{s_{t,j+1}}-e^{s_{t,j}}}{Z},\qquad
Z=\sum_{\ell=1}^{t-1}e^{s_{t,\ell}}.
\]
Also $m=(e^{s_{t,j+1}}+e^{s_{t,j}})/Z$, hence
\[
\alpha_{t,j+1}-\alpha_{t,j}
= m\cdot\frac{e^{s_{t,j+1}}-e^{s_{t,j}}}{e^{s_{t,j+1}}+e^{s_{t,j}}}
= m\,\tanh\!\Big(\frac{s_{t,j+1}-s_{t,j}}{2}\Big)
= m\,\tanh\!\Big(\frac{\Delta s}{2}\Big).
\]
\end{proof}

\begin{lemma}[Softmax transfer lower bound]\label{lem:transfer_app}
Under Assumption~\ref{ass:nd_app} (with $B\le 2$), for any $j\le t-2$,
\[
\mathbb E\big[(\alpha_{t,j+1}-\alpha_{t,j})^2\big]
\;\ge\;
c\,\mathbb E\big[(\Delta s)^2\big],
\qquad
c\triangleq \frac{\eta^2\kappa'}{16}.
\]
\end{lemma}

\begin{proof}
By Lemma~\ref{lem:tanh_app},
\[
(\alpha_{t,j+1}-\alpha_{t,j})^2 = m^2 \tanh^2\!\Big(\frac{\Delta s}{2}\Big).
\]
On $\{|\Delta s|\le B\}$ with $B\le 2$, we have $|\Delta s|/2\le 1$ and use the elementary bound
\[
|\tanh(x)|\ge \frac{|x|}{2}\qquad\text{for all }|x|\le 1,
\]
which implies $\tanh^2(\Delta s/2)\ge (\Delta s)^2/16$ on $\{|\Delta s|\le B\}$.
On $\{m\ge \eta\}$ we have $m^2\ge \eta^2$. Therefore on the event
\[
E=\{m\ge \eta\}\cap\{|\Delta s|\le B\},
\]
\[
(\alpha_{t,j+1}-\alpha_{t,j})^2 \ge \eta^2\cdot\frac{(\Delta s)^2}{16}.
\]
Taking expectations and applying Assumption~\ref{ass:nd_app} yields
\[
\mathbb E\big[(\alpha_{t,j+1}-\alpha_{t,j})^2\big]
\ge \frac{\eta^2}{16}\,\mathbb E\big[(\Delta s)^2\mathbf{1}_E\big]
\ge \frac{\eta^2\kappa'}{16}\,\mathbb E\big[(\Delta s)^2\big].
\]
\end{proof}

\subsection{Main theorem and concise proof via Sub-claims 1--3}

\begin{theorem}[Monotone $K$-dependent lower bound for attention roughness]\label{thm:main_app}
Fix $t\ge 2$ and consider the single-layer causal attention row $\alpha_{t,1:t-1}$ defined in \eqref{eq:softmax_app}.
Under Assumptions~\ref{ass:labels_app}, \ref{ass:query_cond_app}, \ref{ass:sep_app}, and \ref{ass:nd_app},
there exists a constant $C>0$ independent of $K$ such that
\[
\mathbb E[R_t] \;\ge\; C\,(t-2)\Big(1-\frac{1}{K}\Big)\Delta^2.
\]
In particular, the right-hand side is monotone increasing in $K$; hence $\mathbb E[R_t]$ admits a monotone increasing-in-$K$ lower bound.
\end{theorem}

\begin{proof}
By Lemma~\ref{lem:transfer_app}, for each $j\le t-2$,
\[
\mathbb E\big[(\alpha_{t,j+1}-\alpha_{t,j})^2\big]
\ge c\,\mathbb E\big[(s_{t,j+1}-s_{t,j})^2\big],
\qquad c=\frac{\eta^2\kappa'}{16}.
\]
Summing over $j=1,\dots,t-2$ gives
\[
\mathbb E[R_t]
=\sum_{j=1}^{t-2}\mathbb E\big[(\alpha_{t,j+1}-\alpha_{t,j})^2\big]
\ge c\sum_{j=1}^{t-2}\mathbb E\big[(s_{t,j+1}-s_{t,j})^2\big].
\]
Applying Lemma~\ref{lem:logit_energy_app} yields, for each $j$,
\[
\mathbb E\big[(s_{t,j+1}-s_{t,j})^2\big]
\ge \Big(1-\frac{1}{K}\Big)\Delta^2,
\]
where the factor $\big(1-\frac{1}{K}\big)$ comes from Lemma~\ref{lem:switch_app}.
Therefore
\[
\mathbb E[R_t]
\ge c\sum_{j=1}^{t-2}\Big(1-\frac{1}{K}\Big)\Delta^2
= c\,(t-2)\Big(1-\frac{1}{K}\Big)\Delta^2.
\]
Setting $C=c$ completes the proof.
\end{proof}

\subsection{Summary of the logical connection between Sub-claims 1--3}

\paragraph{How the pieces fit together.}
The proof decomposes the $K$-dependence into three steps:
\begin{enumerate}[leftmargin=*,label=\textbf{(\roman*)}]
\item \textbf{(Sub-claim 1)} Increasing $K$ increases the probability that adjacent inputs come from different mixture components:
$\Pr(c_{j+1}\neq c_j)=1-\frac{1}{K}$.
\item \textbf{(Sub-claim 2)} Under separability along the attention score direction $u$, a component switch forces a nontrivial expected squared jump in adjacent logits, yielding a lower bound
$\mathbb E[(s_{t,j+1}-s_{t,j})^2]\ge (1-\frac{1}{K})\Delta^2$ (up to additional noise energy).
\item \textbf{(Sub-claim 3)} A pairwise identity for softmax shows
$\alpha_{t,j+1}-\alpha_{t,j}=m\tanh(\Delta s/2)$, and under non-degeneracy ($m$ not vanishing too often) this yields
$\mathbb E[(\alpha_{t,j+1}-\alpha_{t,j})^2]\gtrsim \mathbb E[(\Delta s)^2]$.
\end{enumerate}
Chaining (i)--(iii) and summing over $j$ proves Theorem~\ref{thm:main_app}, which provides a monotone increasing-in-$K$ lower bound on $\mathbb E[R_t]$.

\section{Frequency-aware Operators' Energy Equivalence}
\label{apx:energy_equal}
In this appendix section, we provide additional technical details on the spectral operators used in our framework and clarify how high-frequency energy can be consistently quantified using the $\ell_2$ norm.
We show that the Discrete Fourier Transform (DFT), Discrete Wavelet Transform (DWT), and the discrete Laplacian operator all provide principled ways to extract high-frequency components from discrete attention signals, and that their corresponding $\ell_2$ norms measure comparable notions of signal energy despite operating in different domains.

\subsection{Discrete Fourier Transform}
\label{apx:sec:appendix_dft}

We begin with the DFT, which provides a global frequency-domain representation of discrete signals.
Given a token-level attention signal $a_{l,h} \in \mathbb{R}^T$ from layer $l$ and head $h$, its DFT is defined as
\begin{equation}
\hat{a}_{l,h}(\omega) = \sum_{t=0}^{T-1} a_{l,h}(t)\, e^{-i 2\pi \omega t / T}, \quad \omega = 0, \dots, T-1.
\end{equation}

A frequency-domain high-pass filter $M(\omega)$ can be applied to isolate high-frequency components, yielding
\begin{equation}
\hat{z}^{\mathrm{hf}}_{l,h}(\omega) = M(\omega)\, \hat{a}_{l,h}(\omega),
\end{equation}
where $M(\omega)=1$ for $\omega \in \Omega_{\text{high}}$ and $0$ otherwise.
The corresponding high-frequency signal in the token domain is obtained via the inverse DFT.

The use of the $\ell_2$ norm as a hallucination score is justified by the \emph{Parseval's theorem} in Section \ref{subsec:instability}, which guarantees energy equivalence between the token and frequency domains:
\begin{equation}
\| z^{\mathrm{hf}}_{l,h} \|_2^2
= \frac{1}{T} \sum_{\omega \in \Omega_{\text{high}}} 
\left| \hat{a}_{l,h}(\omega) \right|^2.
\end{equation}
Thus, the score $s_{l,h} = \| z^{\mathrm{hf}}_{l,h} \|_2$ directly measures the total high-frequency energy of the attention signal, capturing rapid oscillations and global irregularities across the token sequence.

\subsection{Discrete Wavelet Transform}
\label{apx:sec:appendix_dwt}

DWT provides a multi-resolution analysis of discrete signals, decomposing an attention signal into components at different spatial scales.
Using an orthonormal wavelet basis (e.g., Daubechies wavelets), the attention signal $a_{l,h}$ is decomposed into approximation coefficients $\mathbf{c_A}$ and detail coefficients $\mathbf{c_D}$ at multiple scales.

In our framework, the high-frequency component $z^{\mathrm{hf}}_{l,h}$ is reconstructed from $\mathbf{c_D}$ corresponding to fine scales, which capture localized and abrupt changes in attention.

By \emph{Parseval’s theorem}, the energy of the attention signal is preserved under an orthonormal wavelet transform, and equals the sum of squared wavelet coefficients.
Let $d_{j,k}$ denote the detail coefficient $\mathbf{c_D}$ at scale $j$ and position $k$. Then,
\begin{equation}
\| z^{\mathrm{hf}}_{l,h} \|_2^2
= \sum_{j \in \mathcal{J}_{\text{high}}} \sum_{k} |d_{j,k}|^2,
\end{equation}
where $\mathcal{J}_{\text{high}}$ denotes the set of high-frequency scales.

Accordingly, the score $s_{l,h} = \| z^{\mathrm{hf}}_{l,h} \|_2$ quantifies the localized detail energy of the attention distribution, emphasizing sharp transitions and spatially concentrated irregularities that are characteristic of unstable attention patterns.

\subsection{Discrete Laplacian Operator}
\label{apx:sec:appendix_laplacian}

The discrete Laplacian operator provides a spatial-domain alternative for extracting high-frequency variation.
For a one-dimensional attention signal $a_{l,h}$, the discrete Laplacian $\mathbf{L}$ computes second-order differences, measuring how much each token's attention deviates from the average of its immediate neighbors.

The $\ell_2$ norm of the Laplacian response,
\begin{equation}
\| \mathbf{L} a_{l,h} \|_2^2,
\end{equation}
corresponds to the discrete \emph{Dirichlet energy} of the signal, which quantifies its roughness or lack of smoothness \cite{sandryhaila2014discrete}.
A higher energy value indicates that the attention distribution exhibits rapid local oscillations or fragmentation, rather than smooth decay or coherent concentration.

The role of the Laplacian as a high-frequency operator can be further understood through its frequency response.
In the spectral domain, the discrete Laplacian corresponds to a filter with transfer function
\begin{equation}
H(\omega) = 2 - 2\cos(\omega).
\end{equation}
As $\omega \to 0$, $H(\omega) \to 0$, suppressing low-frequency (smooth) components.
As $\omega \to \pi$, $H(\omega)$ attains its maximum value, strongly amplifying high-frequency components associated with abrupt transitions.

Consequently, the score $s_{l,h} = \| \mathbf{L} a_{l,h} \|_2$ serves as a computationally efficient proxy for high-frequency energy, capturing localized attention instability without requiring an explicit frequency-domain transform.

\section{Experiment Details}
\label{apx:exp_detail}

\subsection{Implementation of Baselines}
\label{sec:appendix-baselines}

\subsubsection{Verification and Probabilistic Baselines}

For verification-based baselines, both \textbf{SelfCheckGPT} and \textbf{RefChecker} use \texttt{gpt-4o-mini} as the backbone model for claim extraction and factual verification. 
For probabilistic consistency baselines, \textbf{EigenScore} constructs a similarity matrix $\mathbf{S}$ based on BERTScore between generated responses, and derives a consistency score from the largest eigenvalue $\lambda_{\max}(\mathbf{S})$. 
\textbf{ReDeEP} computes two complementary signals during generation: an external context score derived from attention allocation over input tokens, and a parametric knowledge score derived from feed-forward network activations, following the original implementation.

Both EigenScore and ReDeEP are strong unsupervised methods specifically designed for hallucination detection that leverage internal model signals, including attention. Despite not being directly trained for the task, they have shown powerful performance and are therefore included as competitive and fair baselines in our comparison.

\subsubsection{Attention-based Feature Extraction}

For all intrinsic, attention-based methods (including entropy-based baselines and our frequency-aware features), we extract attention weights from all transformer layers and heads. 
Let $\mathbf{A} = \{ \mathbf{A}_{l,h} \}$ denote the collection of attention distributions, where $\mathbf{A}_{l,h} \in \mathbb{R}^{T}$ is the token-level attention vector from layer $l$ and head $h$.

Following \cite{chuang2024lookback}, we partition attention into context-directed attention $\mathbf{A}^{(c)}_{l,h}$ and generated-token attention $\mathbf{A}^{(g)}_{l,h}$.
For a given head, a feature extraction function $f(\cdot)$ is applied to the corresponding attention distribution.
Aggregating across all layers and heads yields separate feature vectors for context and generated attention:
\begin{equation}
\mathbf{v}^{(c)} = [ f(\mathbf{A}^{(c)}_{1,1}), \dots, f(\mathbf{A}^{(c)}_{L,H}) ]^\top,
\quad
\mathbf{v}^{(g)} = [ f(\mathbf{A}^{(g)}_{1,1}), \dots, f(\mathbf{A}^{(g)}_{L,H}) ]^\top.
\end{equation}
The final representation is formed by concatenation, $[\mathbf{v}^{(c)} ; \mathbf{v}^{(g)}]$, which is then used for hallucination prediction, consistent with the formulation in the main text.

For entropy-based baselines, the feature function $f(\cdot)$ computes the entropy of the attention distribution for each head:
\begin{equation}
H(\mathbf{A}_{l,h}) = - \sum_{i} a_i \log a_i,
\end{equation}
where $a_i$ denotes the normalized attention weight assigned to token $i$.

\subsection{Frequency-aware Operators}
\label{sec:appendix-frequency-operators}

\paragraph{Discrete Fourier Transform.}
For DFT, we separate low-frequency and high-frequency components using a frequency cutoff.
We systematically vary the cutoff threshold and evaluate hallucination detection performance across models, datasets, and sliding-window settings, as shown in \autoref{fig:ablation_cutoff}.
Across most settings, performance improves as the cutoff increases from low values, remains stable over a broad intermediate range, and drops sharply when the cutoff approaches the Nyquist limit.


Here, a normalized frequency of $0.5$ corresponds to the Nyquist limit under the DFT formulation.
As shown in \autoref{fig:ablation_cutoff}, performance remains stable across a broad range of cutoff frequencies, but begins to degrade when the cutoff exceeds approximately 0.45, with a pronounced drop on the Nyquist boundary. This trend suggests that high-frequency attention variations relevant to hallucination detection are effectively captured below this transition point.
Empirically, we find that a cutoff frequency of 0.45 yields consistently strong and stable performance across models and datasets. Additionally, performance is relatively insensitive to lower cutoff settings, with improvements increasing gradually rather than abruptly—indicating robustness to the exact choice of cutoff within a reasonable range.

\paragraph{Discrete Wavelet Transform.}
For DWT, we perform a multi-resolution decomposition of each attention signal $a_{l,h}$ using the Daubechies-4 (db4) wavelet.
The db4 wavelet offers a favorable trade-off between locality and smoothness due to its vanishing moments, making it suitable for capturing abrupt transitions in attention distributions.
For boundary handling in finite-length attention sequences, we use zero padding for token-level detection and symmetric padding for span-level detection, based on overall performance observed across datasets and models.

We further compare depth-1 (level1) and depth-2 (level2) wavelet decompositions, as shown in \autoref{tab:wavelet_sliding1_full} and \autoref{tab:wavelet_sliding8_full}.
Empirically, a depth-1 decomposition consistently outperforms deeper decompositions across datasets and models.
We attribute this behavior to the fact that higher-level decompositions increasingly mix lower-frequency components into the detail coefficients, reducing their sensitivity to the finest-scale variations that are most informative for hallucination detection.
Accordingly, we adopt a level1 wavelet decomposition in all experiments.

\paragraph{Discrete Laplacian Operator.}
The discrete Laplacian requires no cutoff or scale selection.
Instead, it directly computes second-order differences in the token domain, acting as a local high-pass filter that amplifies rapid attention fluctuations.
This simplicity makes the Laplacian operator computationally efficient and parameter-free, while still capturing localized high-frequency variation in attention distributions.
As discussed in Appendix \ref{apx:sec:appendix_laplacian}, the $\ell_2$ norm of the Laplacian response corresponds to the discrete Dirichlet energy, providing a principled measure of attention roughness without additional hyperparameters.


\subsection{Model Details}
\label{apx:model_detail}

\paragraph{Inference.}
All experiments are conducted on NVIDIA A100 (80GB) GPUs.
The LLMs are kept frozen throughout.
We perform inference using teacher-forcing decoding with model response tokens to extract attention distributions for spectral analysis.

\paragraph{Classifiers.}
Our method uses a lightweight single-layer Logistic Regression classifier. This choice ensures that the detector does not introduce additional non-linear modeling capacity, and that performance differences primarily reflect the discriminative power of the proposed spectral features, rather than classifier expressiveness.
The linear detector was implemented using the \texttt{scikit-learn} library. All hyperparameters, including the $\ell_2$ penalty, were kept at their default values, with the exception of the maximum number of iterations, which was set to 1,000 to ensure model convergence.

\paragraph{Evaluation.}
For threshold-dependent metrics such as F1, the decision threshold is selected on a held-out validation split and fixed for test evaluation. In cases where an official validation set is not predefined, we reserve 10\% of the original training data for this purpose.
We use AUROC as our primary metric because AUROC is threshold-independent and are unaffected by this choice.

\subsection{Dataset Details}
We use two publicly available contextual hallucination detection datasets in this work: RAGTruth and HalluRAG.

\begin{table}[ht]
\centering
\caption{Statistics of datasets used in our experiments. Length and ratio are calculated on the token level.}
\label{tab:dataset_stats}
\begin{tabular}{llcccccc}
\toprule
\textbf{Dataset} & \textbf{Task} & \textbf{Train} & \textbf{Val} & \textbf{Test} & \textbf{Prompt Len.} & \textbf{Response Len.} & \textbf{Halluc. Ratio} \\
\midrule
RagTruth & QA   & 839 & --  & 150 & 439.5 & 208.1 & 0.1045 \\
         & D2T  & 883 & --  & 150 & 892.8 & 225.4 & 0.0768 \\
         & Summ & 793 & --  & 150 & 840.1 & 153.4 & 0.0527 \\
\arrayrulecolor{black!30}\cmidrule(lr){1-8}\arrayrulecolor{black}
HalluRAG & QA   & 756 & 162 & 162 & 649.3 & 44.4  & 0.1530 \\
\bottomrule
\end{tabular}
\end{table}

\paragraph{RAGTruth}RAGTruth is a large-scale corpus designed for word-level hallucination detection within retrieval-augmented generation (RAG) frameworks~\cite{niu-etal-2024-ragtruth}. It consists of 2,965 manually annotated responses with precise hallucination span labels across three primary tasks: question answering (QA), data-to-text (D2T), and summarization (Summ). The dataset provides standardized train and test splits for each task. Specifically, the QA, D2T, and Summ tasks comprise 839, 883, and 793 training samples, respectively, with each task sharing a consistent test set size of 150 samples.

\paragraph{HalluRAG.} HalluRAG focuses on sentence-level hallucination detection in RAG-based QA scenarios~\cite{ridder2025halluragdatasetdetectingcloseddomain}. It provides generated responses paired with sentence-level annotations, comprising 756 training, 162 validation, and 162 test samples.

More details can be seen in \autoref{tab:dataset_stats}.
All token-level statistics are computed using the tokenizer corresponding to each model, and the reported values are averaged across models.
In particular, the hallucination ratio is defined as the number of hallucinated tokens in the response divided by the total number of response tokens. Since this ratio may vary across models, we report the average value over different models.

\section{Additional Analysis and Results}
\label{apx:full_result}

This section reports the full experimental results and ablation studies omitted from the main text due to space constraints.
These results complement the analyses in the main paper and provide additional details on span-level detection, operator configurations, and the robustness of frequency-aware features across different settings.

\subsection{Cross-domain Transfer Analysis}

To examine whether the detector captures intrinsic attention-based signals rather than overfits task-specific artifacts, we conduct cross-domain transfer evaluation, where detectors are trained on one task domain and evaluated on another, as shown in \autoref{tab:sliding_cross_domain}. This experiment also serves to evaluate the generalization ability of the classifier and to test its susceptibility to overfitting.

\definecolor{midgreen}{RGB}{69, 173, 92}

\newcommand{\OffDom}[2]{%
  \cellcolor{midgreen!#1}{\strut #2}
}

\begin{table}[ht]
\centering
\small
\setlength{\tabcolsep}{4pt}
\renewcommand{\arraystretch}{1.15}

\caption{
Cross-domain evaluation. Models are trained on the column domain (Source) and evaluated on the row domain (Target). \textbf{Bold} values indicate in-domain performance (diagonal, Target=Source). Darker shading corresponds to higher cross-domain performance.}
\label{tab:sliding_cross_domain}

\begin{tabular}{ll|ccc|ccc}
\toprule
\textbf{Method} & \textbf{Target}
& \multicolumn{3}{c|}{\textbf{Token-Level (Source)}}
& \multicolumn{3}{c}{\textbf{Span-Level (Source)}} \\
\cmidrule(lr){3-5} \cmidrule(lr){6-8}
& & \textbf{QA} & \textbf{D2T} & \textbf{Summ}
  & \textbf{QA} & \textbf{D2T} & \textbf{Summ} \\
\midrule

\multirow{3}{*}{Lookback-lens}
& QA
& \textbf{0.8482} & \OffDom{45}{0.7839} & \OffDom{42}{0.7741}
& \textbf{0.8467} & \OffDom{51}{0.8140} & \OffDom{47}{0.7985} \\
& D2T
& \OffDom{4}{0.6340} & \textbf{0.8442} & \OffDom{17}{0.7211}
& \OffDom{4}{0.6334} & \textbf{0.8551} & \OffDom{18}{0.7275} \\
& Summ
& \OffDom{8}{0.6618} & \OffDom{4}{0.6400} & \textbf{0.7156}
& \OffDom{7}{0.6563} & \OffDom{5}{0.6496} & \textbf{0.6635} \\
\midrule

\multirow{3}{*}{Laplacian}
& QA
& \textbf{0.8449} & \OffDom{46}{0.7928} & \OffDom{46}{0.7881}
& \textbf{0.8365} & \OffDom{46}{0.7898} & \OffDom{45}{0.7857} \\
& D2T
& \OffDom{5}{0.6438} & \textbf{0.8519} & \OffDom{11}{0.6809}
& \OffDom{1}{0.6147} & \textbf{0.8646} & \OffDom{2}{0.6264} \\
& Summ
& \OffDom{5}{0.6428} & \OffDom{8}{0.6678} & \textbf{0.7040}
& \OffDom{0}{0.6132} & \OffDom{7}{0.6596} & \textbf{0.6619} \\
\midrule

\multirow{3}{*}{Wavelet-high}
& QA
& \textbf{0.8526} & \OffDom{73}{0.8056} & \OffDom{74}{0.8076}
& \textbf{0.8680} & \OffDom{86}{0.8316} & \OffDom{90}{0.8418} \\
& D2T
& \OffDom{6}{0.6525} & \textbf{0.8569} & \OffDom{16}{0.7179}
& \OffDom{4}{0.6405} & \textbf{0.8821} & \OffDom{19}{0.7307} \\
& Summ
& \OffDom{8}{0.6630} & \OffDom{10}{0.6773} & \textbf{0.7165}
& \OffDom{6}{0.6541} & \OffDom{16}{0.7032} & \textbf{0.7199} \\
\midrule

\multirow{3}{*}{Fourier-high}
& QA
& \textbf{0.8584} & \OffDom{85}{0.8282} & \OffDom{83}{0.8233}
& \textbf{0.8725} & \OffDom{94}{0.8561} & \OffDom{95}{0.8593} \\
& D2T
& \OffDom{10}{0.6668} & \textbf{0.8595} & \OffDom{17}{0.7228}
& \OffDom{5}{0.6434} & \textbf{0.8869} & \OffDom{23}{0.7584} \\
& Summ
& \OffDom{12}{0.6781} & \OffDom{16}{0.7040} & \textbf{0.7426}
& \OffDom{11}{0.6748} & \OffDom{18}{0.7296} & \textbf{0.7641} \\
\bottomrule
\end{tabular}

\end{table}

Overall, spectral-based detectors exhibit more stable cross-domain behavior compared to Lookback-Lens.
In particular, Fourier- and Wavelet-based variants maintain stronger performance when transferring across task boundaries, whereas Lookback-Lens shows larger performance degradation under domain shift.
This suggests that frequency-domain attention features capture more task-robust signals than heuristics based on attention mass or recency.

We also observe an asymmetric transfer pattern across tasks.
Models trained on QA generally transfer worse to Data-to-Text and Summarization than models trained on Data-to-Text or Summarization transferring to QA.
This asymmetry holds consistently across methods and detection granularities.
A plausible explanation is that QA exhibits more constrained and localized attention patterns, which may limit the generality of learned aggregation weights when applied to structurally different generation tasks.

\subsection{Full Results of Span-level Detection}
We report full results for span-level hallucination detection using a sliding-window setting with window size of 8, following the protocol described in the main text.
\autoref{tab:sliding8_full} presents full performance metrics across models, datasets, and frequency operators.

\begin{table*}[htbp]
\centering
\small
\setlength{\tabcolsep}{4pt}

\caption{Full performance comparison on RagTruth and HalluRAG with chunk size set to 8.
For each model, the best result is highlighted in \textbf{bold}, and the second-best result is \underline{underlined}.
}
\label{tab:sliding8_full}

\begin{tabular}{l cc cc cc cc cc}
\toprule
\multirow{2.5}{*}{\textbf{Model / Method}} &
\multicolumn{2}{c}{\textbf{RT-QA}} &
\multicolumn{2}{c}{\textbf{RT-D2T}} &
\multicolumn{2}{c}{\textbf{RT-Summ}} &
\multicolumn{2}{c}{\textbf{HalluRAG}} &
\multicolumn{2}{c}{\textbf{Overall Avg.}} \\
\cmidrule(lr){2-3} \cmidrule(lr){4-5} \cmidrule(lr){6-7}
\cmidrule(lr){8-9} \cmidrule(lr){10-11}
& F & AUROC & F & AUROC & F & AUROC & F & AUROC & Avg-F & Avg-A \\
\midrule

\textbf{LLaMA-7B} \\
\quad Lookback-lens
& 0.7218 & 0.8467 & 0.7249 & 0.8551 & 0.5852 & 0.6635 & 0.6623 & 0.7856 & 0.6736 & 0.7877 \\
\quad Attn-variance
& 0.5077 & 0.6979 & 0.4714 & 0.5886 & 0.4840 & 0.6348 & 0.4446 & 0.5209 & 0.4769 & 0.6106 \\
\quad Attn-entropy
& 0.7292 & 0.8502 & 0.7060 & 0.8457 & 0.5489 & 0.6637 & 0.6282 & 0.6909 & 0.6531 & 0.7626 \\
\quad Laplacian
& 0.7074 & 0.8365 & 0.7334 & 0.8646 & 0.5833 & 0.6619 & 0.6775 & 0.7754 & 0.6754 & 0.7846 \\
\quad Wavelet-high
& \underline{0.7331} & \underline{0.8680} & \textbf{0.7478} & \underline{0.8821} & \underline{0.6045} & \underline{0.7199} & \underline{0.6812} & \underline{0.7928} & \underline{0.6917} & \underline{0.8157} \\
\quad Fourier-high
& \textbf{0.7473} & \textbf{0.8725} & \underline{0.7348} & \textbf{0.8869} & \textbf{0.6188} & \textbf{0.7641} & \textbf{0.6866} & \textbf{0.8100} & \textbf{0.6969} & \textbf{0.8334} \\

\midrule
\textbf{LLaMA-13B} \\
\quad Lookback-lens
& 0.7004 & 0.8594 & \underline{0.7610} & \underline{0.8872} & 0.5728 & 0.7083 & 0.7096 & \underline{0.8505} & 0.6860 & 0.8264 \\
\quad Attn-variance
& 0.5301 & 0.6871 & 0.5329 & 0.7283 & 0.4917 & 0.7014 & 0.5548 & 0.6412 & 0.5274 & 0.6895 \\
\quad Attn-entropy
& 0.6876 & 0.8478 & 0.7226 & 0.8602 & 0.5426 & 0.6377 & 0.5872 & 0.6249 & 0.6350 & 0.7427 \\
\quad Laplacian
& 0.6966 & 0.8467 & 0.7562 & 0.8853 & \underline{0.5930} & 0.6798 & \underline{0.7217} & 0.8264 & \underline{0.6919} & 0.8096 \\
\quad Wavelet-high
& 0.7002 & 0.8685 & 0.7225 & 0.8793 & 0.5821 & \underline{0.7202} & \textbf{0.7417} & 0.8438 & 0.6866 & \underline{0.8279} \\
\quad Fourier-high
& \textbf{0.7365} & \textbf{0.8863} & \textbf{0.7706} & \textbf{0.8988} & \textbf{0.6119} & \textbf{0.7904} & \underline{0.7217} & \textbf{0.8515} & \textbf{0.7102} & \textbf{0.8568} \\

\midrule
\textbf{Mistral-7B} \\
\quad Lookback-lens
& \underline{0.7920} & \textbf{0.9206} & 0.7751 & 0.9002 & \underline{0.6992} & \underline{0.8082} & 0.7752 & 0.8403 & \underline{0.7604} & 0.8673 \\
\quad Attn-variance
& 0.7036 & 0.8216 & 0.6339 & 0.8112 & 0.4741 & 0.6507 & 0.5709 & 0.7215 & 0.5956 & 0.7512 \\
\quad Attn-entropy
& 0.7857 & 0.9007 & 0.7300 & 0.8656 & 0.6606 & 0.7727 & 0.6597 & 0.7140 & 0.7090 & 0.8132 \\
\quad Laplacian
& 0.7500 & 0.8822 & \underline{0.7719} & \underline{0.8928} & 0.6747 & 0.7734 & \textbf{0.8056} & \textbf{0.8920} & 0.7506 & 0.8601 \\
\quad Wavelet-high
& 0.7839 & 0.9083 & 0.7287 & 0.8866 & 0.6771 & 0.8073 & 0.7682 & 0.8802 & 0.7395 & \underline{0.8706} \\
\quad Fourier-high
& \textbf{0.8042} & \underline{0.9190} & \textbf{0.7833} & \textbf{0.9057} & \textbf{0.7114} & \textbf{0.8188} & \underline{0.7883} & \underline{0.8872} & \textbf{0.7718} & \textbf{0.8827} \\

\bottomrule
\end{tabular}

\end{table*}

\begin{figure}[htbp]
    \centering
    \begin{subfigure}{0.24\textwidth}
        \centering
        \includegraphics[width=\textwidth]{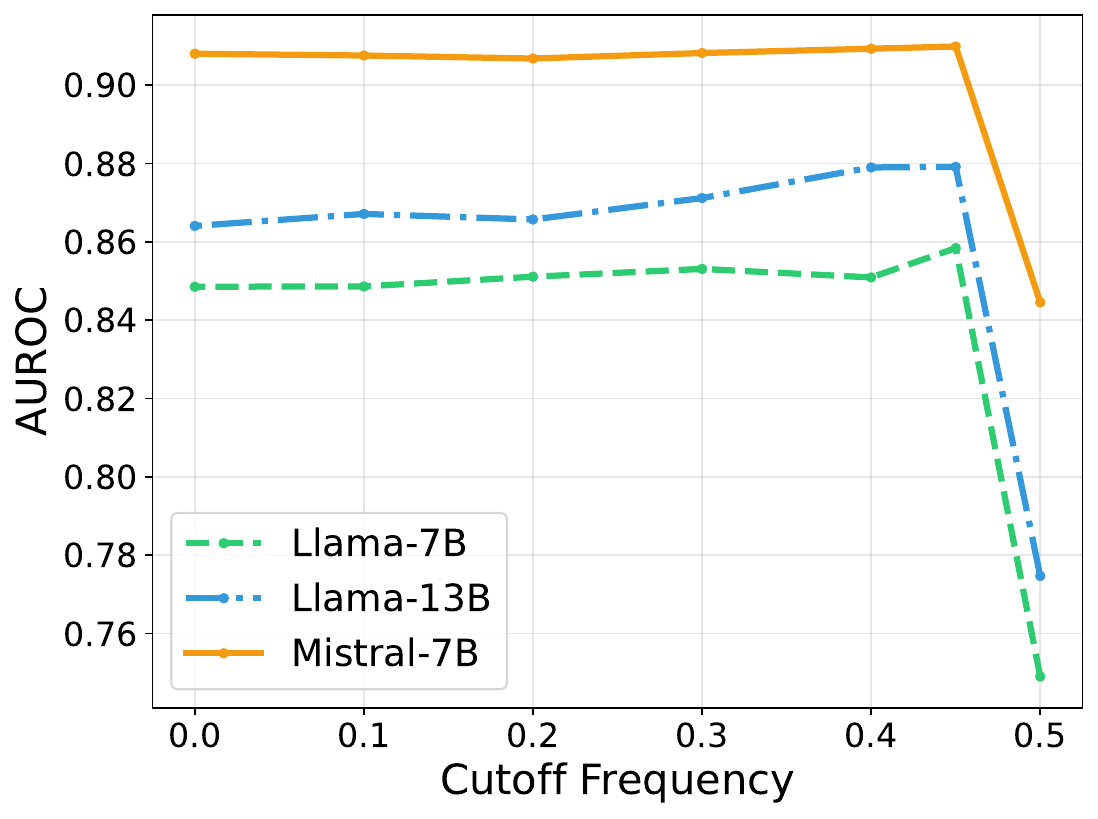}
        \caption{RAGTruth-QA (Token)}
        \label{fig:rag_qa_s1}
    \end{subfigure}
    \hfill
    \begin{subfigure}{0.24\textwidth}
        \centering
        \includegraphics[width=\textwidth]{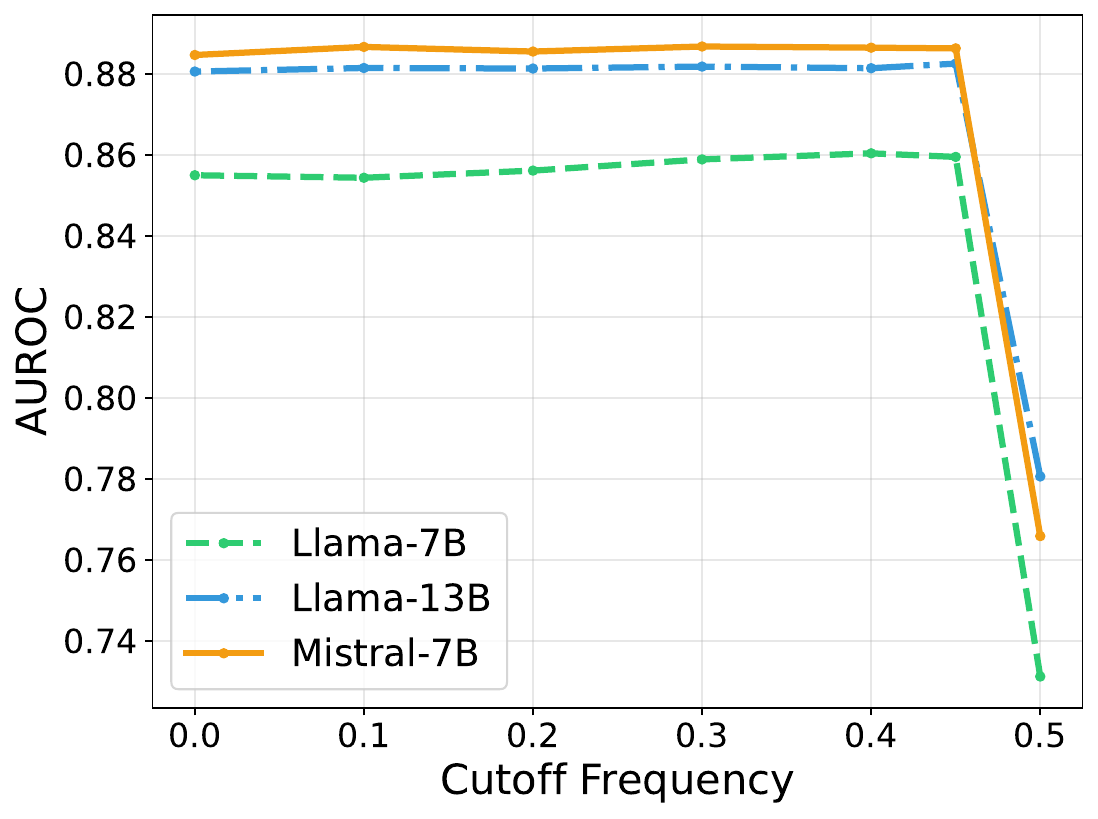}
        \caption{RAGTruth-D2T (Token)}
        \label{fig:rag_d2t_s1}
    \end{subfigure}
    \hfill
    \begin{subfigure}{0.24\textwidth}
        \centering
        \includegraphics[width=\textwidth]{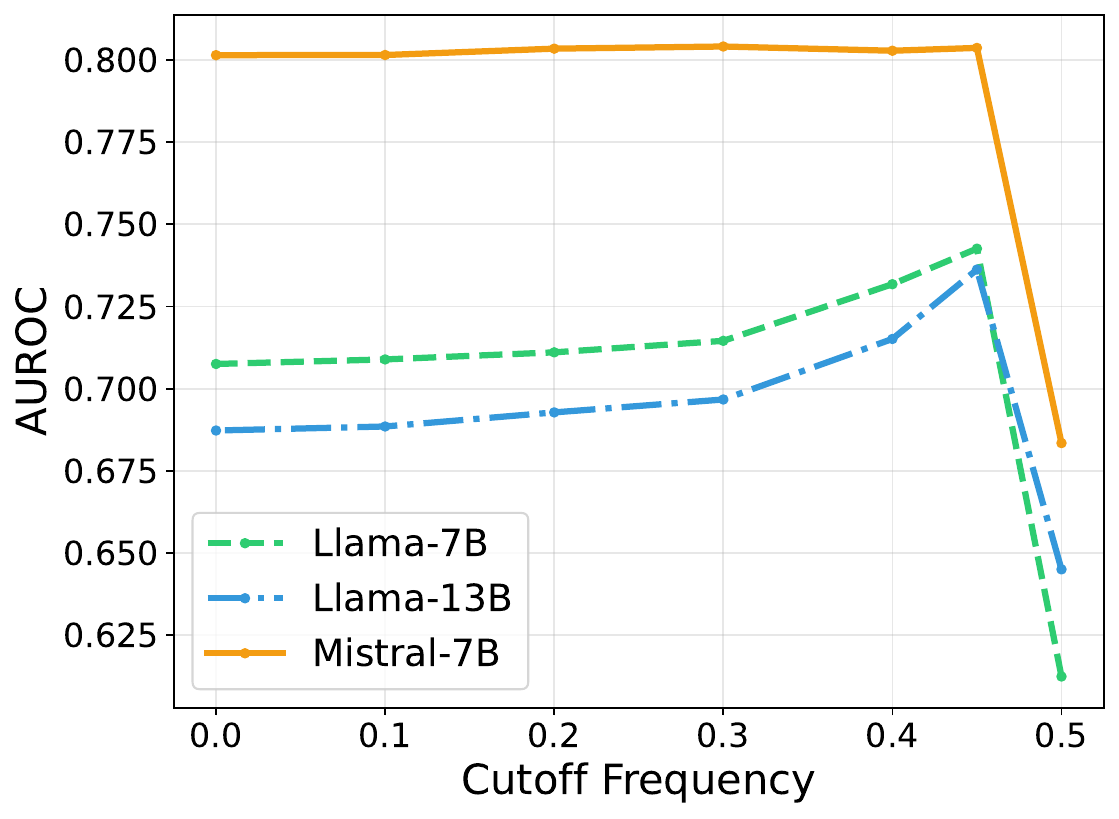}
        \caption{RAGTruth-Summ (Token)}
        \label{fig:rag_sum_s1}
    \end{subfigure}
    \hfill
    \begin{subfigure}{0.24\textwidth}
        \centering
        \includegraphics[width=\textwidth]{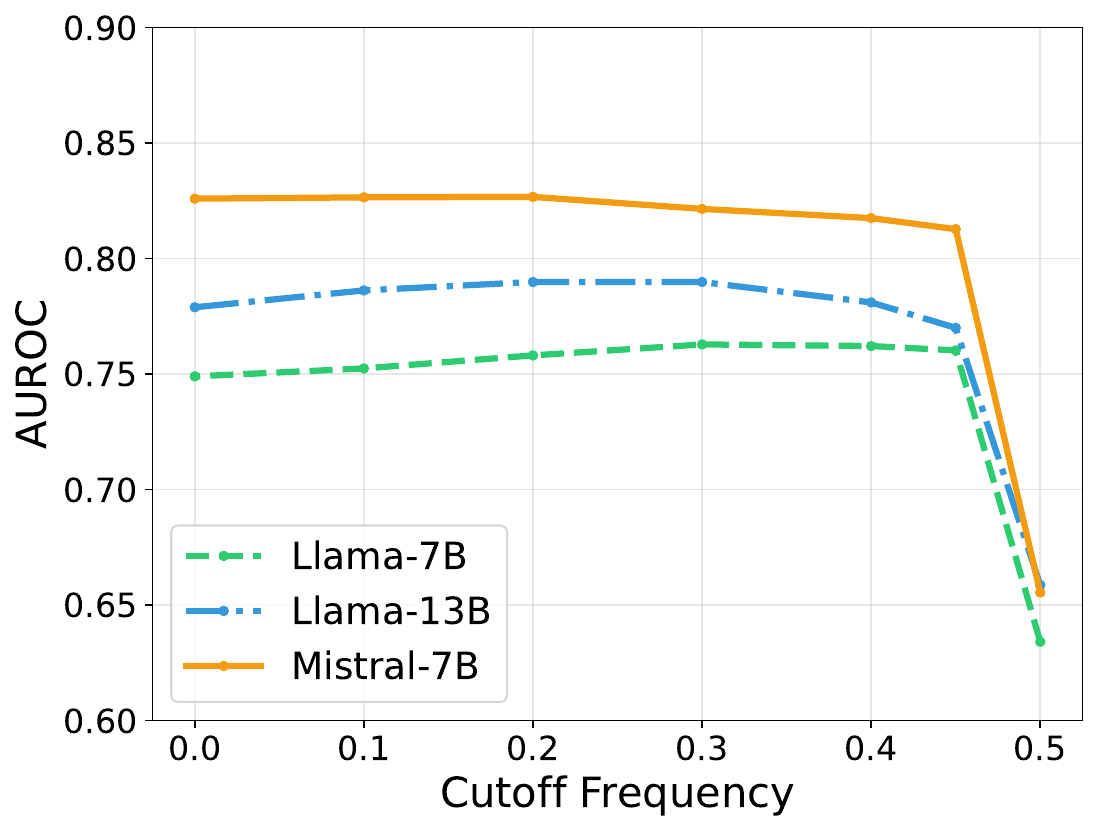}
        \caption{HalluRAG (Token)}
        \label{fig:hallu_s1}
    \end{subfigure}

    \vspace{1em} 

    \begin{subfigure}{0.24\textwidth}
        \centering
        \includegraphics[width=\textwidth]{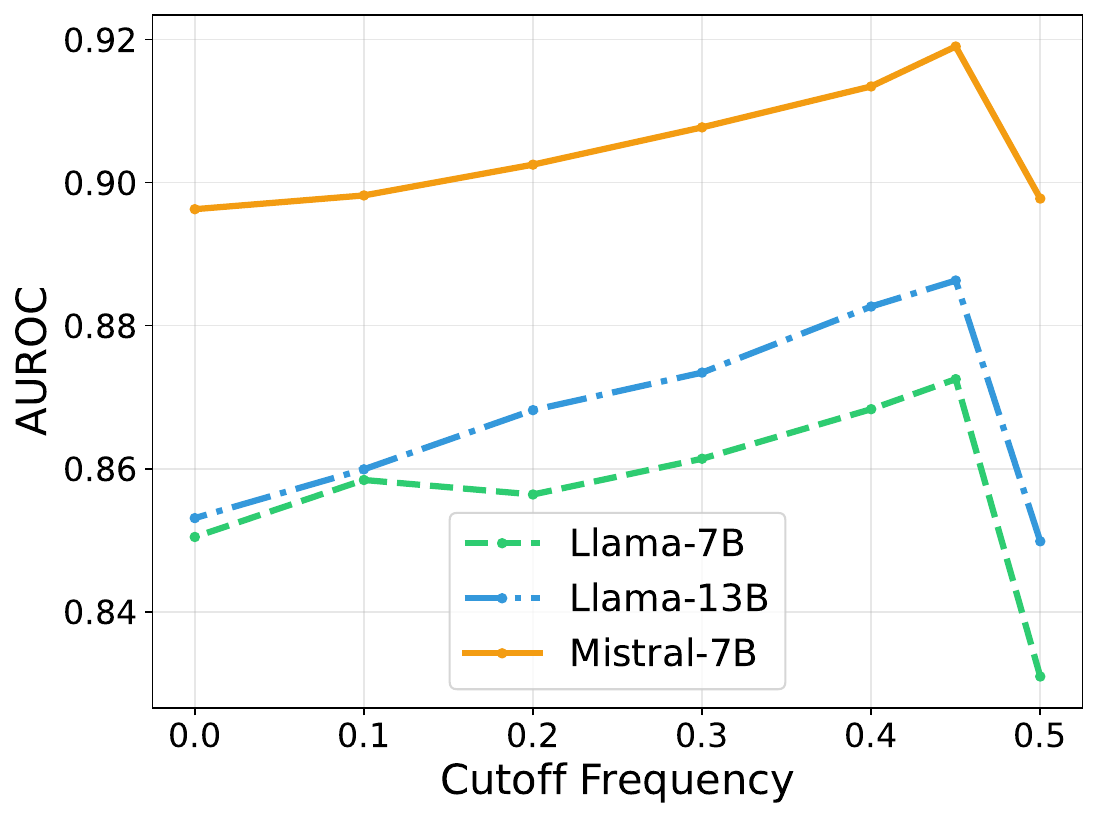}
        \caption{RAGTruth-QA (Span)}
        \label{fig:rag_qa_s8}
    \end{subfigure}
    \hfill
    \begin{subfigure}{0.24\textwidth}
        \centering
        \includegraphics[width=\textwidth]{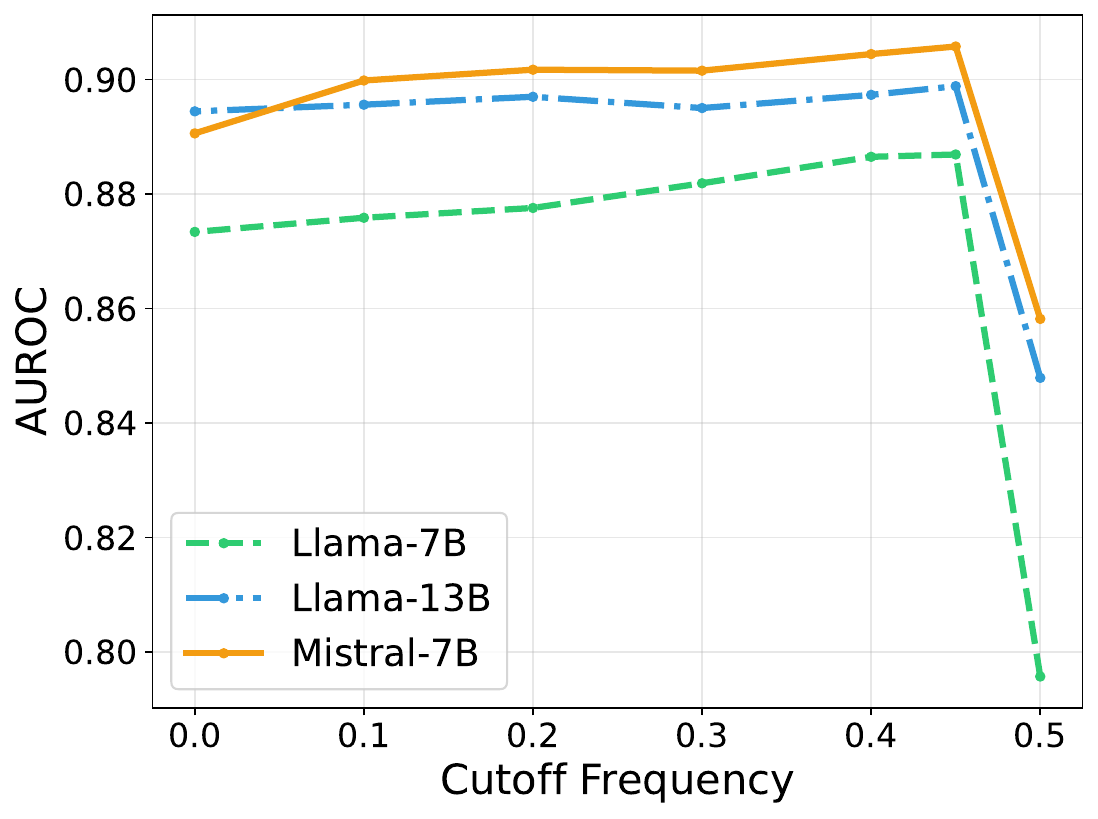}
        \caption{RAGTruth-D2T (Span)}
        \label{fig:rag_d2t_s8}
    \end{subfigure}
    \hfill
    \begin{subfigure}{0.24\textwidth}
        \centering
        \includegraphics[width=\textwidth]{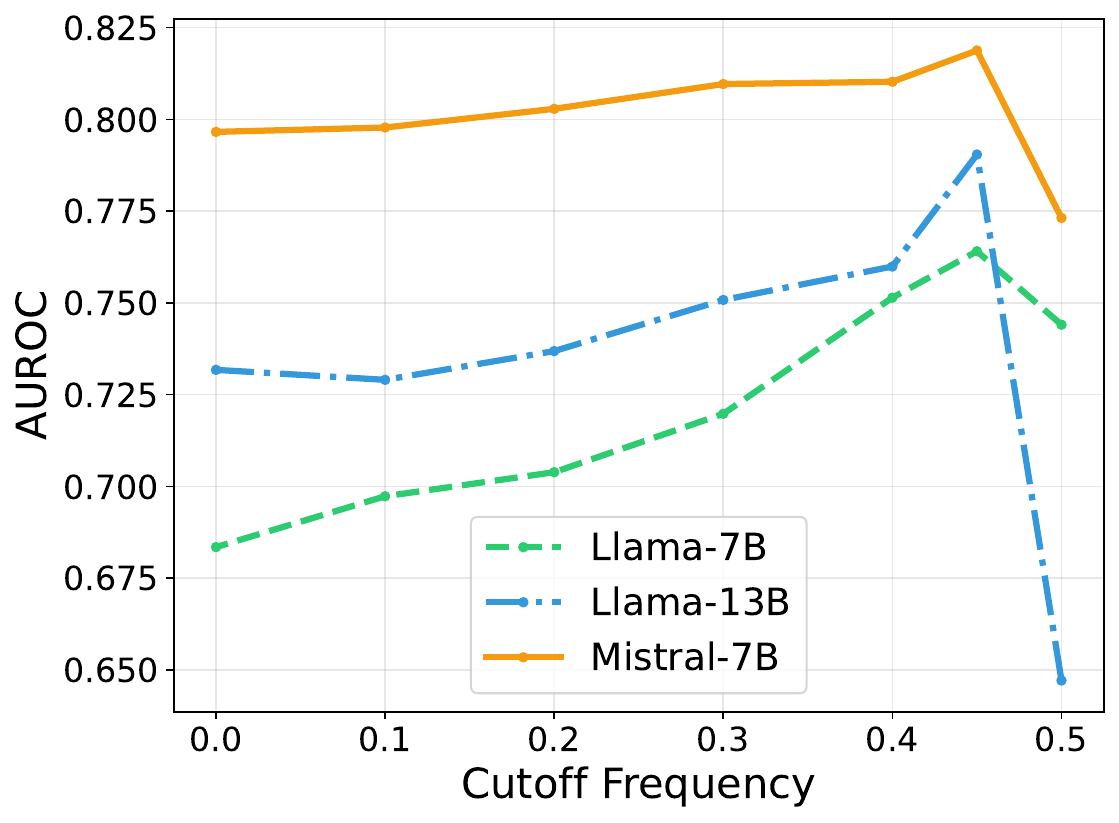}
        \caption{RAGTruth-Summ (Span)}
        \label{fig:rag_sum_s8}
    \end{subfigure}
    \hfill
    \begin{subfigure}{0.24\textwidth}
        \centering
        \includegraphics[width=\textwidth]{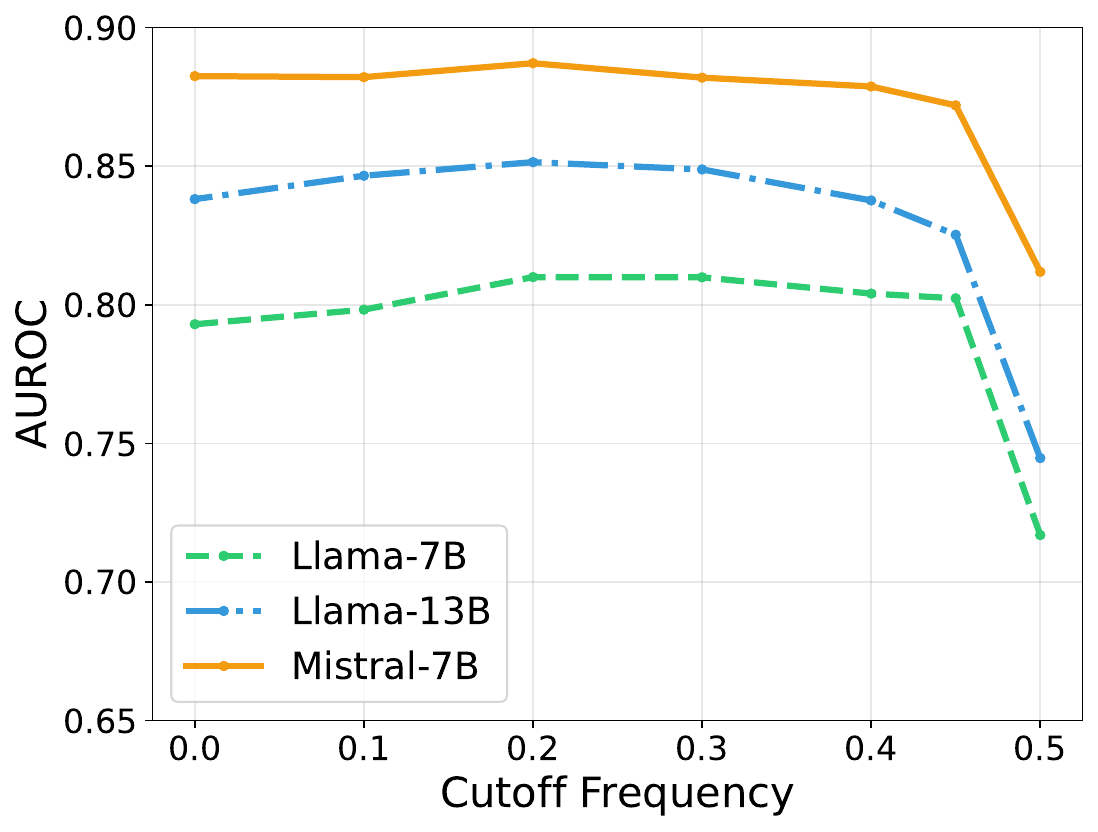}
        \caption{HalluRAG (Span)}
        \label{fig:hallu_s8}
    \end{subfigure}

    \caption{Ablation study on different frequency cutoffs at token and span levels on hallucination detection performance. The top row (a-d) shows results at the token level, while the bottom row (e-h) shows results at the span level. 
    }
    \label{fig:ablation_cutoff}
\end{figure}

\begin{table*}[t]
\centering
\small
\setlength{\tabcolsep}{4pt}
\newcommand{\lightrule}{\arrayrulecolor{black!30}\cmidrule(lr){2-12}\arrayrulecolor{black}}

\caption{Performance comparison on token level using Wavelet-high across levels and padding methods. Best results within each model are in \textbf{bold}; second-best are \underline{underlined}.}
\label{tab:wavelet_sliding1_full}

\begin{tabular}{l l cc cc cc cc cc}
\toprule
\textbf{Model/Level} & \textbf{Padding} &
\multicolumn{2}{c}{\textbf{RT-QA}} &
\multicolumn{2}{c}{\textbf{RT-D2T}} &
\multicolumn{2}{c}{\textbf{RT-Summ}} &
\multicolumn{2}{c}{\textbf{HalluRAG}} &
\multicolumn{2}{c}{\textbf{Overall Avg.}} \\
\cmidrule(lr){3-4} \cmidrule(lr){5-6} \cmidrule(lr){7-8} \cmidrule(lr){9-10} \cmidrule(lr){11-12}
 & & F & AUROC & F & AUROC & F & AUROC & F & AUROC & Avg-F & Avg-A \\
\midrule


\textbf{LLaMA-7B} \\
\quad \multirow{3}{*}{level1}  & zero 
& \underline{0.7194} & \underline{0.8526} & \textbf{0.6898} & \textbf{0.8569} & \textbf{0.5929} & \underline{0.7165} & \textbf{0.6384} & \textbf{0.7550} & \textbf{0.6601} & \textbf{0.7953} \\
 & period 
& 0.7146 & 0.8514 & 0.6887 & 0.8517 & 0.5799 & 0.7091 & 0.6360 & 0.7449 & 0.6548 & 0.7893 \\
 & symm 
& \textbf{0.7215} & \textbf{0.8557} & 0.6875 & \underline{0.8545} & \underline{0.5880} & \textbf{0.7182} & 0.6353 & 0.7521 & \underline{0.6581} & \underline{0.7951} \\
\lightrule
\quad \multirow{3}{*}{level2} & zero 
& 0.7157 & 0.8522 & \underline{0.6893} & 0.8531 & 0.5822 & 0.7088 & 0.6359 & 0.7464 & 0.6558 & 0.7901 \\
 & period 
& 0.7155 & 0.8514 & 0.6891 & 0.8528 & 0.5814 & 0.7069 & 0.6344 & 0.7438 & 0.6551 & 0.7887 \\
 & symm 
& 0.7170 & 0.8519 & 0.6879 & 0.8531 & 0.5851 & 0.7159 & \underline{0.6375} & \underline{0.7543} & 0.6569 & 0.7938 \\
\midrule
\textbf{LLaMA-13B} \\
\quad \multirow{3}{*}{level1}  & zero 
& \textbf{0.7029} & \textbf{0.8741} & \textbf{0.7383} & \textbf{0.8932} & 0.5651 & 0.7042 & \underline{0.6684} & \underline{0.7809} & \underline{0.6687} & \textbf{0.8131} \\
 & period 
& 0.6962 & 0.8679 & 0.7207 & 0.8811 & 0.5678 & 0.7044 & 0.6527 & 0.7675 & 0.6593 & 0.8052 \\
 & symm 
& \underline{0.7002} & 0.8685 & \underline{0.7225} & 0.8793 & \textbf{0.5821} & \textbf{0.7202} & \textbf{0.6718} & \textbf{0.7839} & \textbf{0.6691} & \underline{0.8130} \\
\lightrule
\quad  \multirow{3}{*}{level2} & zero 
& 0.6971 & \underline{0.8711} & 0.7213 & 0.8815 & \underline{0.5699} & 0.7043 & 0.6567 & 0.7701 & 0.6613 & 0.8068 \\
 & period 
& 0.6945 & 0.8675 & 0.7199 & 0.8816 & 0.5687 & 0.7045 & 0.6515 & 0.7669 & 0.6587 & 0.8051 \\
 & symm 
& 0.6993 & 0.8701 & 0.7216 & \underline{0.8834} & 0.5678 & \underline{0.7046} & \underline{0.6710} & 0.7789 & 0.6649 & 0.8093 \\
\midrule
\textbf{Mistral-7B} \\
\quad \multirow{3}{*}{level1} & zero 
& \textbf{0.7876} & \textbf{0.9117} & 0.7136 & 0.8829 & \underline{0.6849} & \textbf{0.8075} & \underline{0.7274} & \underline{0.8360} & 0.7284 & \underline{0.8595} \\
 & period 
& 0.7812 & 0.9077 & 0.7228 & 0.8830 & 0.6840 & 0.8034 & 0.7139 & 0.8229 & 0.7255 & 0.8542 \\
 & symm 
& 0.7839 & 0.9083 & \textbf{0.7287} & \underline{0.8866} & 0.6771 & \underline{0.8073} & 0.7130 & 0.8231 & 0.7257 & 0.8563 \\
\lightrule
\quad \multirow{3}{*}{level2} & zero 
& \underline{0.7873} & \underline{0.9113} & 0.7238 & 0.8826 & 0.6844 & 0.8038 & 0.7213 & 0.8265 & \underline{0.7292} & 0.8561 \\
 & period 
& 0.7770 & 0.9059 & 0.7238 & 0.8826 & 0.6829 & 0.8041 & 0.7130 & 0.8228 & 0.7242 & 0.8539 \\
 & symm 
& 0.7831 & 0.9104 & \underline{0.7265} & \textbf{0.8869} & \textbf{0.6884} & 0.8069 & \textbf{0.7308} & \textbf{0.8413} & \textbf{0.7322} & \textbf{0.8614} \\
\bottomrule
\end{tabular}

\end{table*}

\begin{table*}[t]
\centering
\small
\setlength{\tabcolsep}{4pt}
\newcommand{\lightrule}{\arrayrulecolor{black!30}\cmidrule(lr){2-12}\arrayrulecolor{black}}

\caption{Performance comparison when chunk size is 8 using Wavelet-high across levels and padding methods. Best results within each model are in \textbf{bold}; second-best are \underline{underlined}.}
\label{tab:wavelet_sliding8_full}

\begin{tabular}{l l cc cc cc cc cc}
\toprule
\textbf{Model/Level} & \textbf{Padding} &
\multicolumn{2}{c}{\textbf{RT-QA}} &
\multicolumn{2}{c}{\textbf{RT-D2T}} &
\multicolumn{2}{c}{\textbf{RT-Summ}} &
\multicolumn{2}{c}{\textbf{HalluRAG}} &
\multicolumn{2}{c}{\textbf{Overall Avg.}} \\
\cmidrule(lr){3-4} \cmidrule(lr){5-6} \cmidrule(lr){7-8} \cmidrule(lr){9-10} \cmidrule(lr){11-12}
 & & F & AUROC & F & AUROC & F & AUROC & F & AUROC & Avg-F & Avg-A \\
\midrule

\textbf{LLaMA-7B} \\
\quad \multirow{3}{*}{level1}
& zero & 0.7313 & 0.8611 & \textbf{0.7479} & 0.8785 & 0.6020 & 0.7169 & \textbf{0.6846} & \textbf{0.8047} & 0.6915 & 0.8153 \\
 & period & 0.7271 & 0.8572 & 0.7426 & 0.8748 & 0.5973 & 0.7008 & 0.6786 & 0.7963 & 0.6864 & 0.8073 \\
 & symm & \textbf{0.7331} & \textbf{0.8680} & 0.7478 & \textbf{0.8821} & \underline{0.6045} & \textbf{0.7199} & 0.6812 & 0.7928 & \textbf{0.6917} & \textbf{0.8157} \\
\lightrule
\quad \multirow{3}{*}{level2}
& zero & 0.7277 & 0.8594 & 0.7461 & 0.8760 & 0.5996 & 0.6997 & 0.6800 & 0.7953 & 0.6884 & 0.8076 \\
 & period & 0.7272 & 0.8579 & 0.7452 & 0.8752 & 0.5975 & 0.6972 & 0.6792 & 0.7937 & 0.6873 & 0.8060 \\
 & symm & 0.7277 & \underline{0.8602} & \underline{0.7478} & \underline{0.8768} & \textbf{0.6049} & \underline{0.7107} & \underline{0.6771} & \underline{0.7938} & \underline{0.6894} & \underline{0.8104} \\

\midrule

\textbf{LLaMA-13B} \\
\quad \multirow{3}{*}{level1}
& zero & 0.7050 & 0.8616 & \textbf{0.7833} & \textbf{0.9091} & 0.5974 & 0.7557 & 0.7306 & 0.8479 & 0.7041 & 0.8436 \\
 & period & \underline{0.7105} & 0.8610 & 0.7634 & 0.8951 & 0.5984 & 0.7471 & 0.7141 & 0.8338 & 0.6966 & 0.8343 \\
 & symm & \textbf{0.7200} & \textbf{0.8718} & 0.7616 & 0.8928 & \textbf{0.6043} & \textbf{0.7697} & \textbf{0.7482} & \textbf{0.8541} & \textbf{0.7085} & \textbf{0.8471} \\
\lightrule
\quad \multirow{3}{*}{level2}
& zero & 0.7079 & 0.8612 & 0.7638 & 0.8953 & 0.5940 & 0.7499 & 0.7168 & 0.8351 & 0.6956 & 0.8354 \\
& period & 0.7102 & \underline{0.8617} & 0.7637 & 0.8957 & 0.5957 & 0.7485 & 0.7163 & 0.8337 & 0.6965 & 0.8349 \\
 & symm & 0.7094 & 0.8625 & \underline{0.7664} & \underline{0.8962} & \underline{0.5967} & \underline{0.7563} & \underline{0.7274} & \underline{0.8474} & \underline{0.7000} & \underline{0.8406} \\

\midrule

\textbf{Mistral-7B} \\
\quad \multirow{3}{*}{level1}
& zero & \textbf{0.7984} & 0.9111 & 0.7686 & \textbf{0.9080} & 0.7010 & 0.8111 & 0.7904 & 0.8777 & 0.7646 & 0.8770 \\
 & period & 0.7858 & 0.9108 & \textbf{0.7875} & 0.8997 & \underline{0.7027} & 0.8084 & 0.7900 & 0.8717 & 0.7665 & 0.8726 \\
 & symm & \underline{0.7970} & 0.9052 & 0.7820 & 0.8995 & 0.7001 & \textbf{0.8179} & \underline{0.7921} & \underline{0.8841} & \underline{0.7678} & 0.8767 \\
\lightrule
\quad \multirow{3}{*}{level2}
& zero & 0.7879 & \textbf{0.9127} & \underline{0.7887} & 0.8995 & \textbf{0.7037} & 0.8096 & 0.7902 & 0.8729 & 0.7676 & 0.8737 \\
 & period & 0.7865 & 0.9116 & \underline{0.7887} & 0.8995 & 0.7023 & 0.8079 & 0.7896 & 0.8727 & 0.7668 & 0.8729 \\
 & symm & 0.7919 & \underline{0.9074} & 0.7837 & \underline{0.9000} & 0.7035 & \underline{0.8128} & \textbf{0.8063} & \textbf{0.8876} & \textbf{0.7713} & \textbf{0.8770} \\

\bottomrule
\end{tabular}

\end{table*}

\subsection{Full Results of Ablation Study}
This section reports the full results corresponding to the attention-based analyses discussed in the main paper, evaluated across all datasets and models studied.

\autoref{fig:highvlow_full} reports the full results comparing low-pass and high-pass Fourier attention features across all evaluated models, datasets, and both token-level and span-level settings.

\autoref{fig:ablation_layerwise} shows the complete layer-wise importance profiles of frequency-aware attention features across models, including both token-level and span-level detection.

\autoref{tab:top_k_headRmv_full} presents the full results of the head-level sparsity analysis, reporting detection performance when restricting attention features to the Top-$k$ most important heads across datasets and models.

\autoref{tab:ablation_context_new_full} provides the complete ablation results comparing context-only and generated-only attention features for different spectral operators across all evaluation settings.

\begin{figure}[ht]
  \vskip 0.2in
  \begin{center}
    \centerline{\includegraphics[width=0.9\columnwidth]{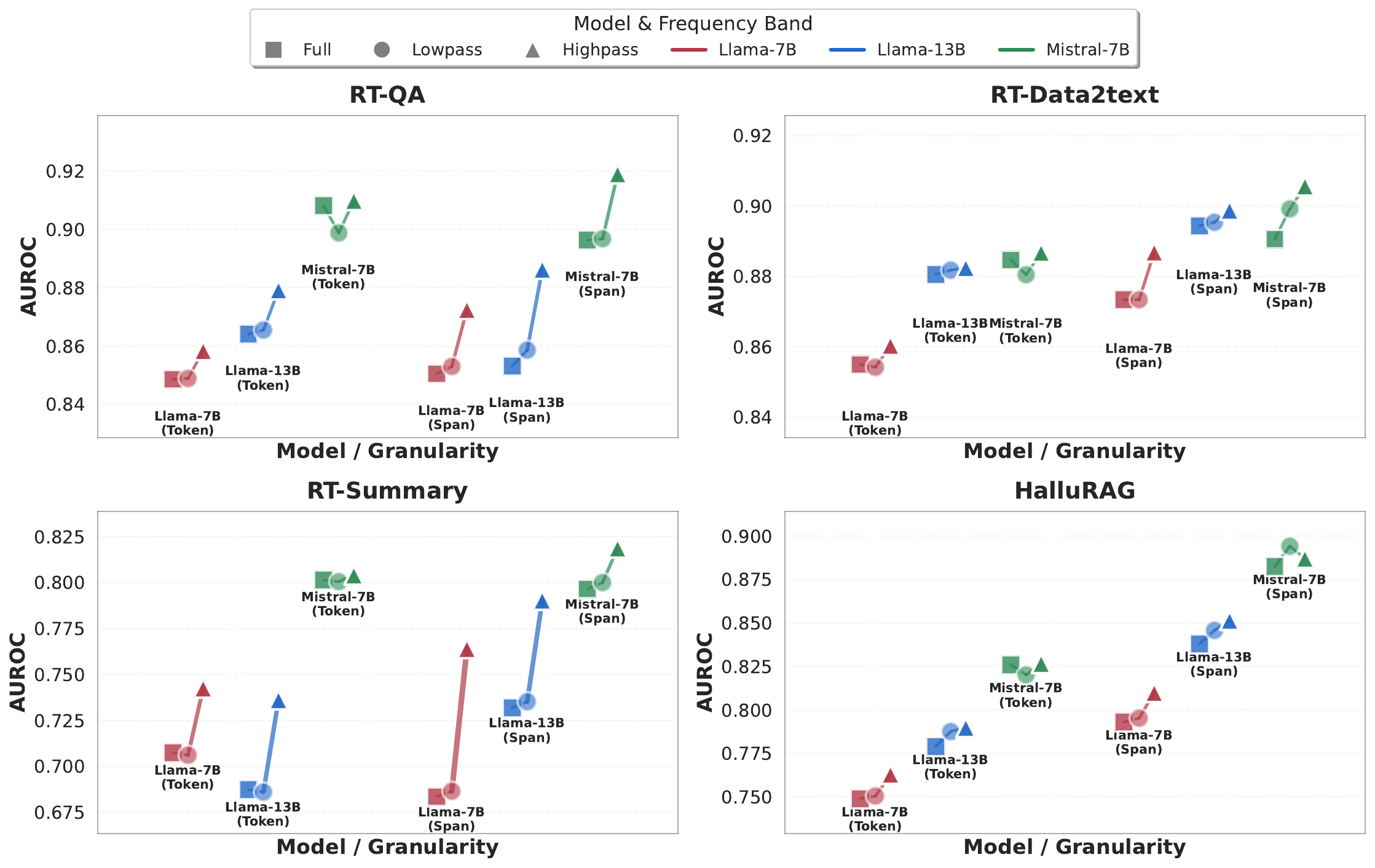}}
    \caption{
      Full results of comparing Full-, Low-, and High-Pass Fourier attention features. Average AUROC across models under token- and span-level evaluation settings.
    }
    \label{fig:highvlow_full}
  \end{center}
\end{figure}

\begin{figure}[htbp]
    \centering
    \begin{subfigure}{0.3\textwidth}
        \centering
        \includegraphics[width=\textwidth]{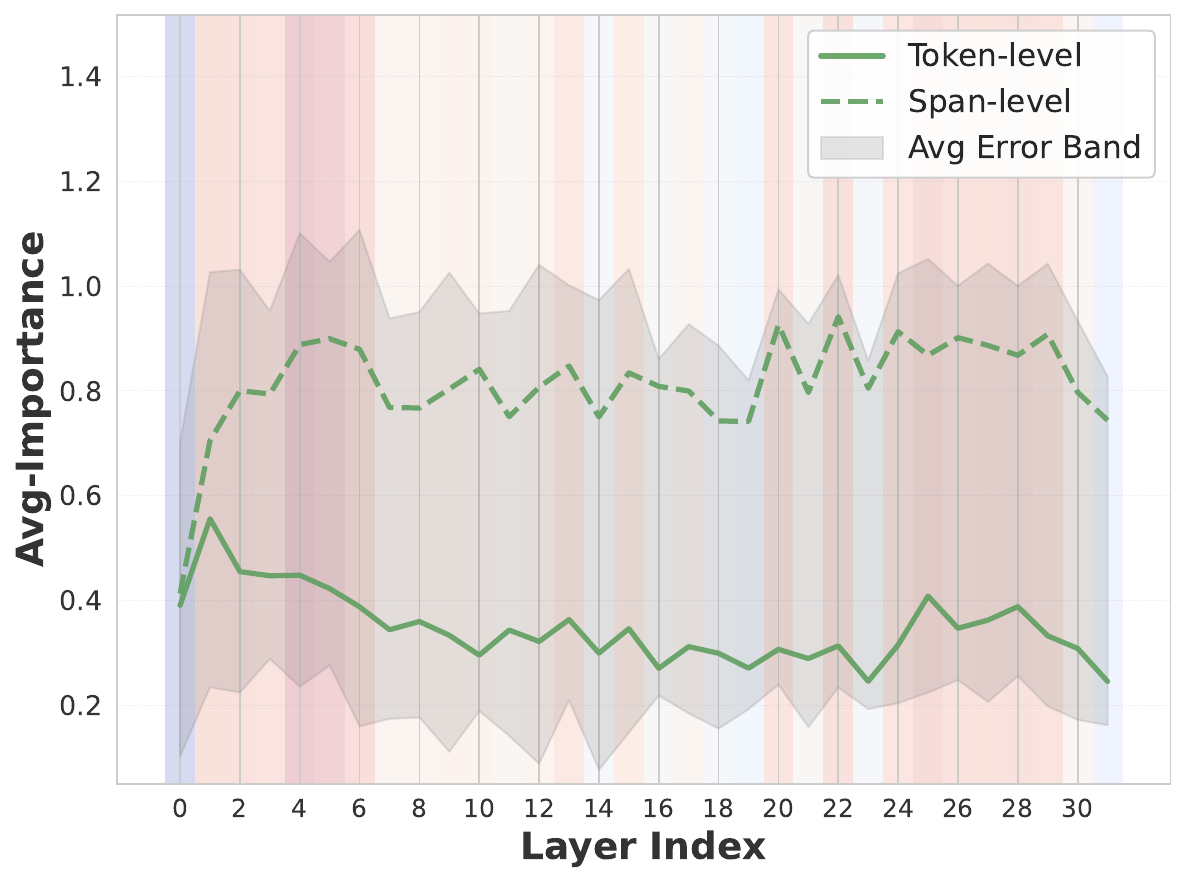}
        \caption{Laplacian of LLaMA-7B}
        \label{fig:layer_l7b_l}
    \end{subfigure}
    \hfill
    \begin{subfigure}{0.3\textwidth}
        \centering
        \includegraphics[width=\textwidth]{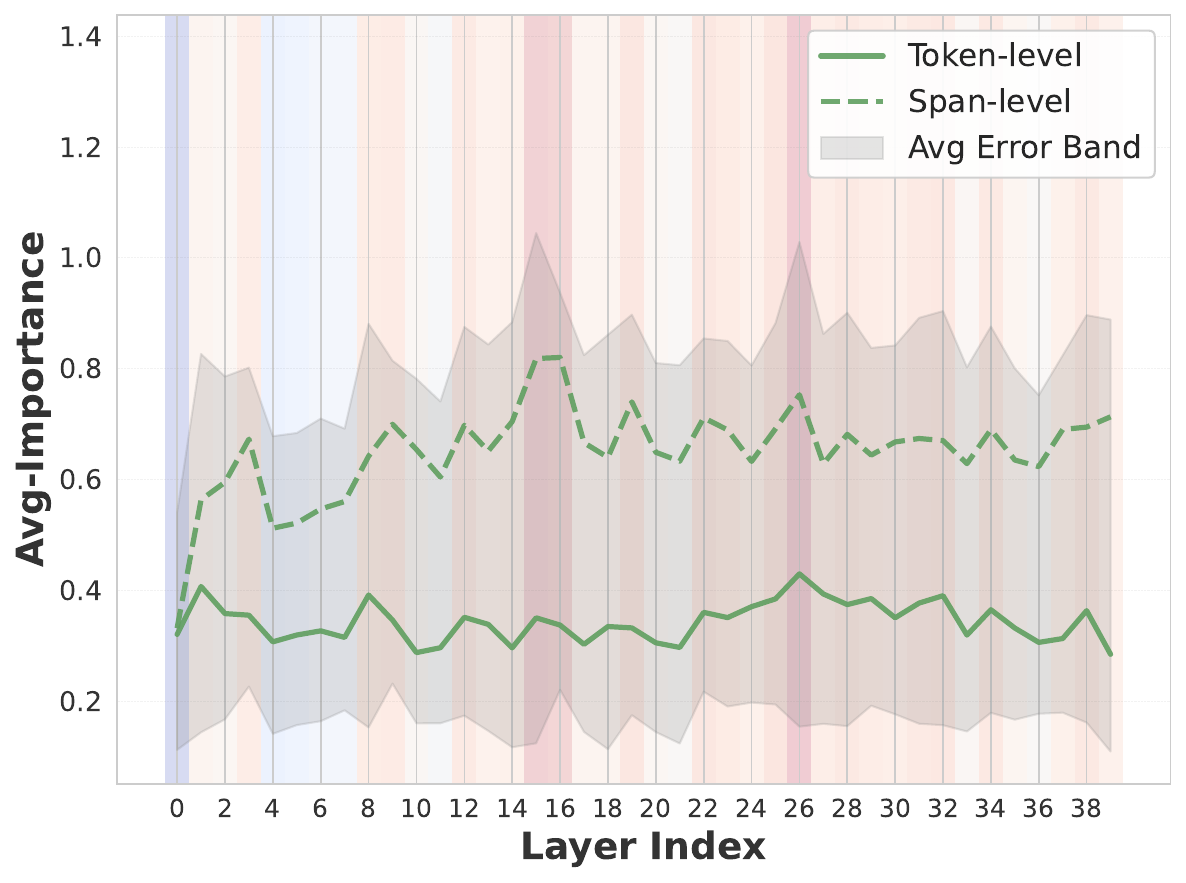}
        \caption{Laplacian of LLaMA-13B}
        \label{fig:layer_l13b_l}
    \end{subfigure}
    \hfill
    \begin{subfigure}{0.305\textwidth}
        \centering
        \includegraphics[width=\textwidth]{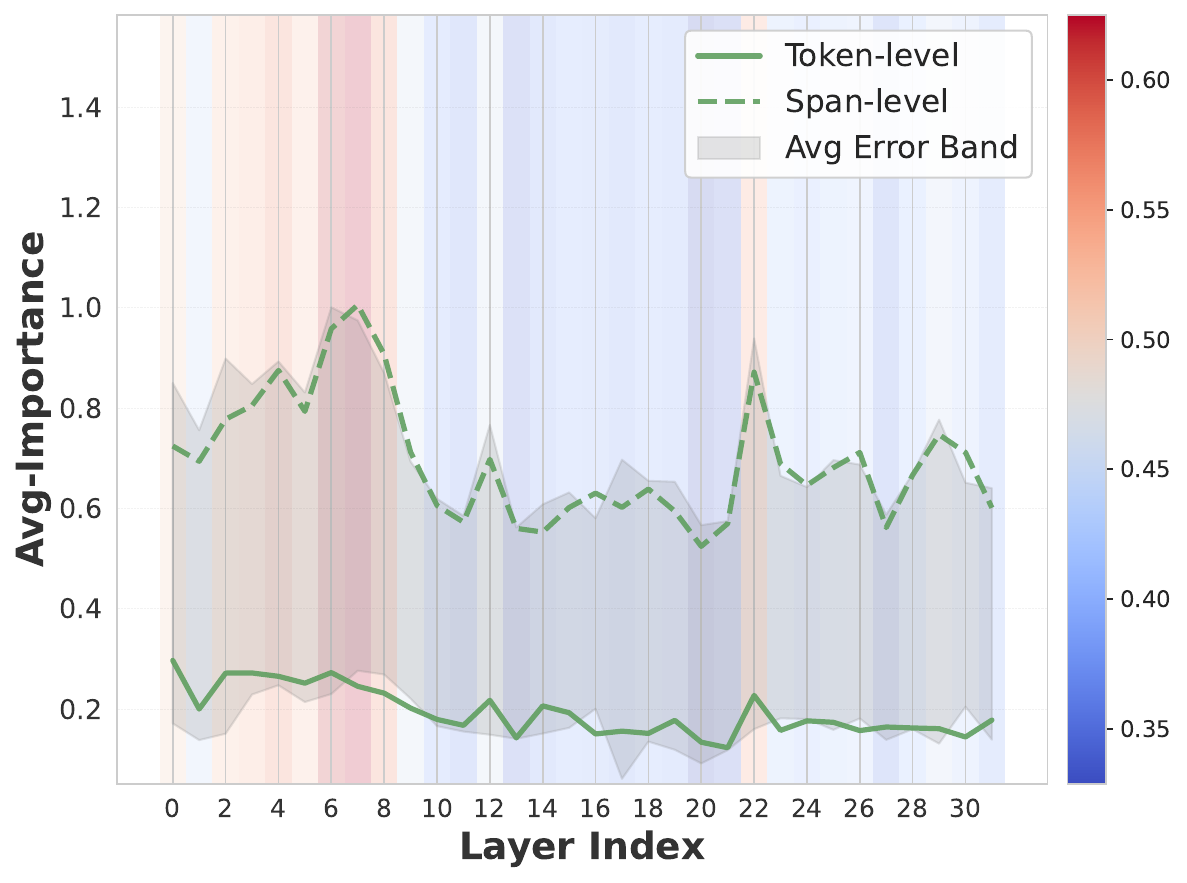}
        \caption{Laplacian of Mistral-7B}
        \label{fig:layer_m7b_l}
    \end{subfigure}

    \vspace{1em} 

    \begin{subfigure}{0.3\textwidth}
        \centering
        \includegraphics[width=\textwidth]{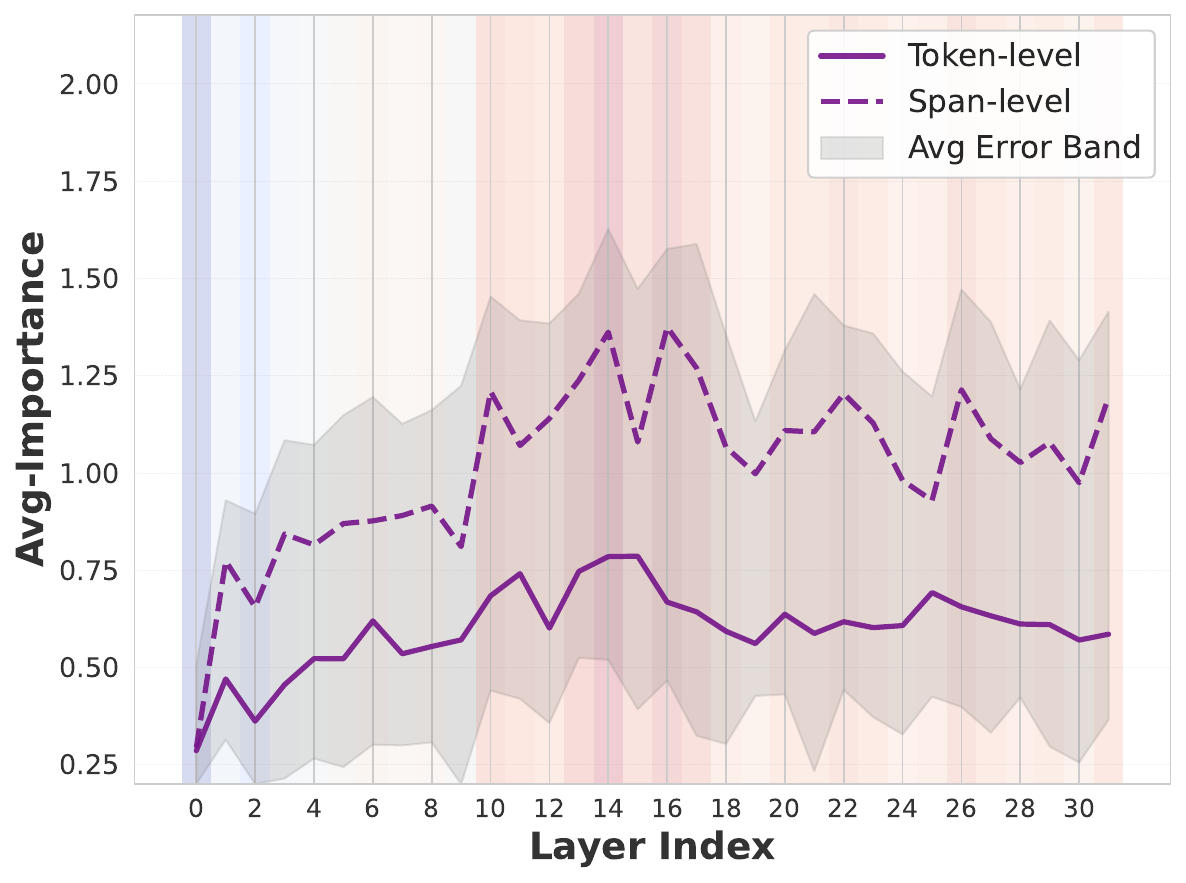}
        \caption{Wavelet-high of LLaMA-7B}
        \label{fig:layer_l7b_w}
    \end{subfigure}
    \hfill
    \begin{subfigure}{0.3\textwidth}
        \centering
        \includegraphics[width=\textwidth]{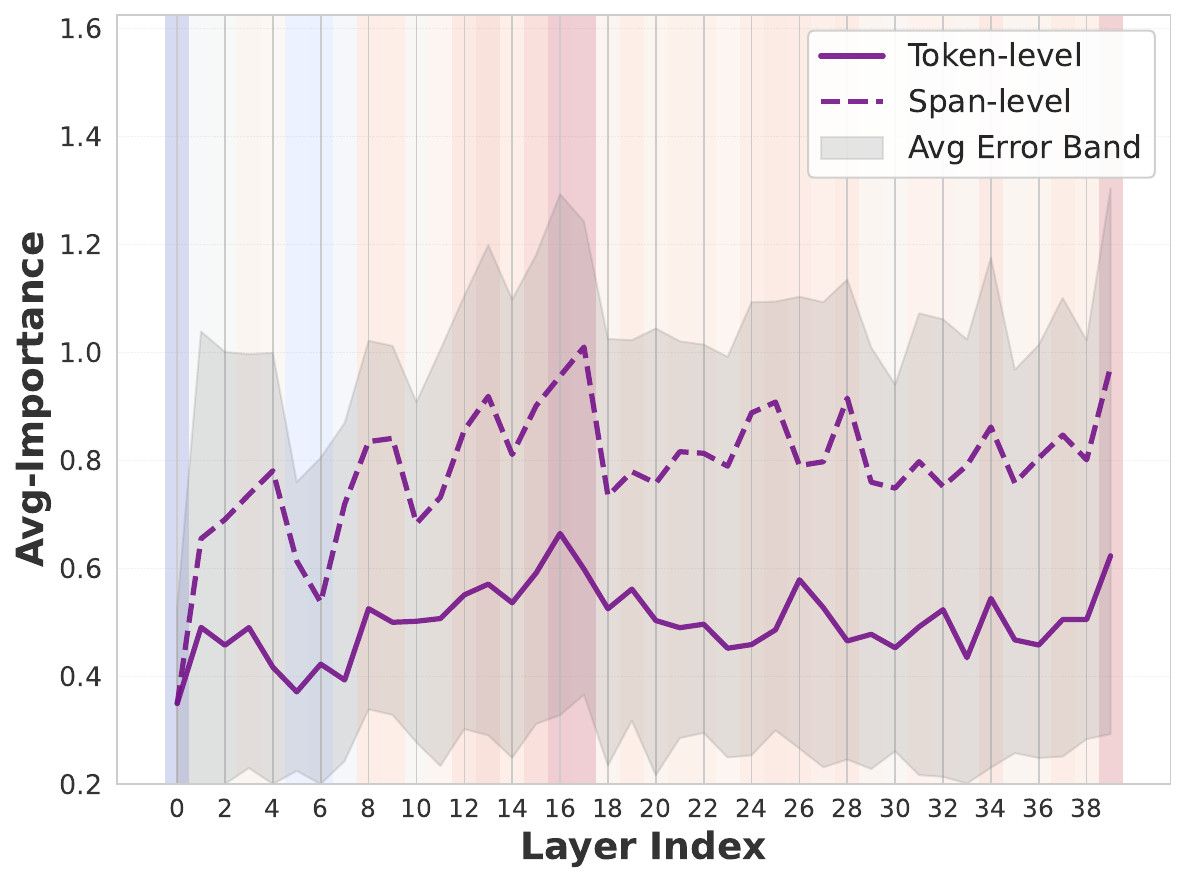}
        \caption{Wavelet-high of LLaMA-13B}
        \label{fig:layer_l13b_w}
    \end{subfigure}
    \hfill
    \begin{subfigure}{0.305\textwidth}
        \centering
        \includegraphics[width=\textwidth]{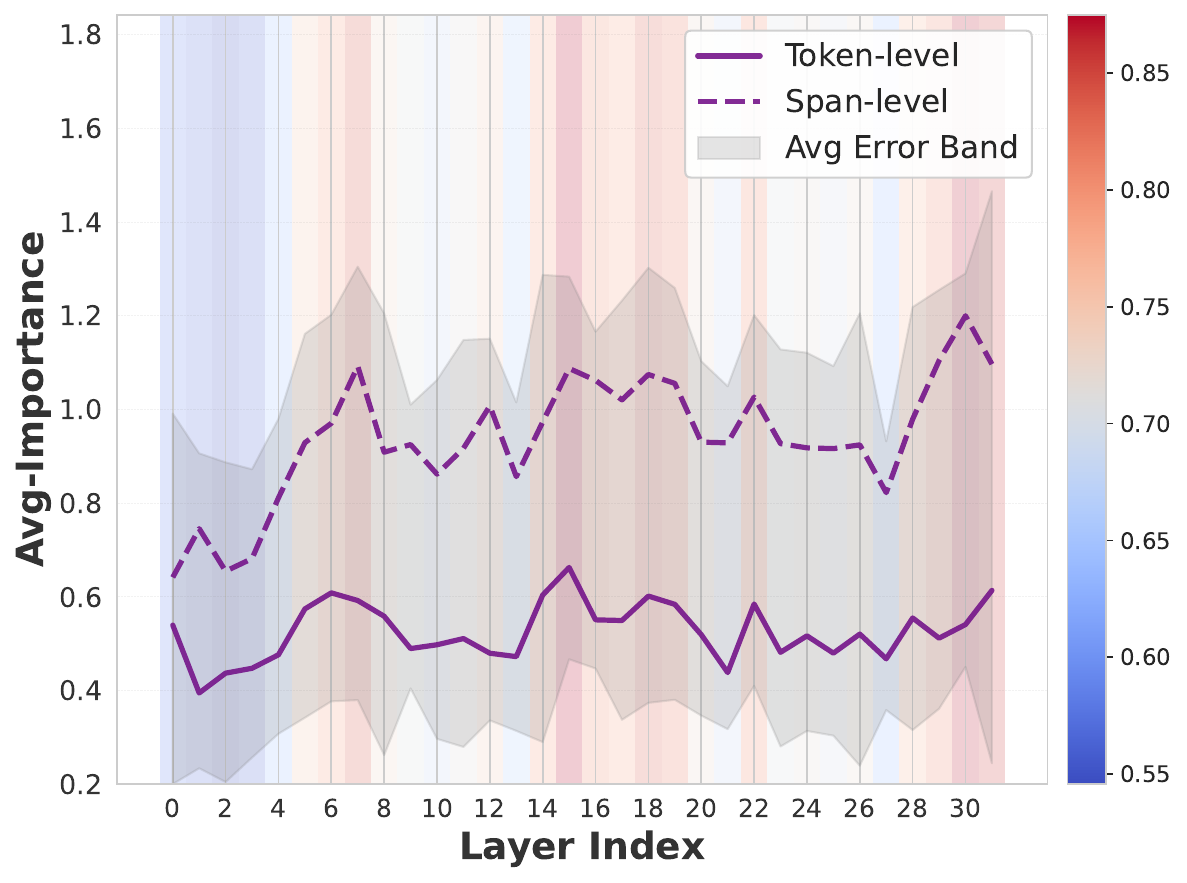}
        \caption{Wavelet-high of Mistral-7B}
        \label{fig:layer_m7b_w}
    \end{subfigure}

    \vspace{1em}

    \begin{subfigure}{0.3\textwidth}
        \centering
        \includegraphics[width=\textwidth]{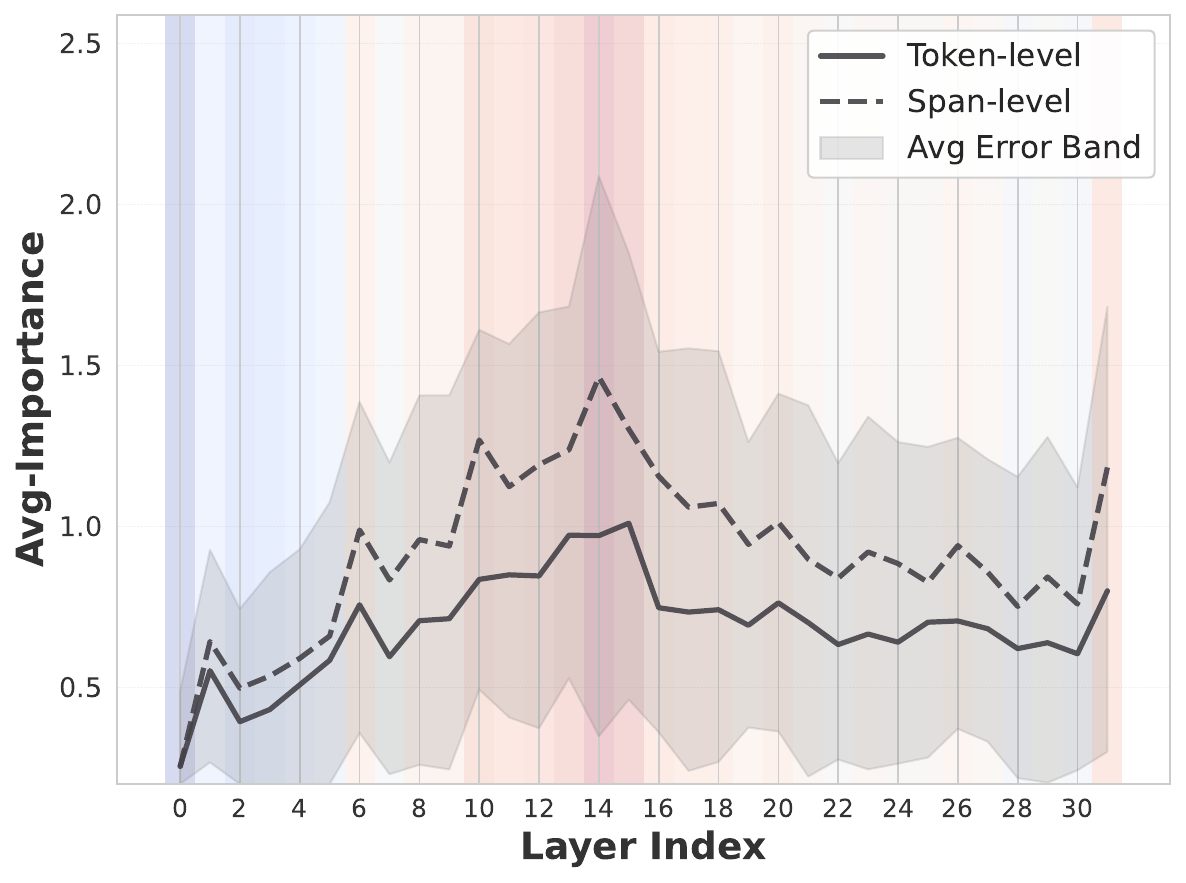}
        \caption{Fourier-high of LLaMA-7B}
        \label{fig:layer_l7b_f}
    \end{subfigure}
    \hfill
    \begin{subfigure}{0.3\textwidth}
        \centering
        \includegraphics[width=\textwidth]{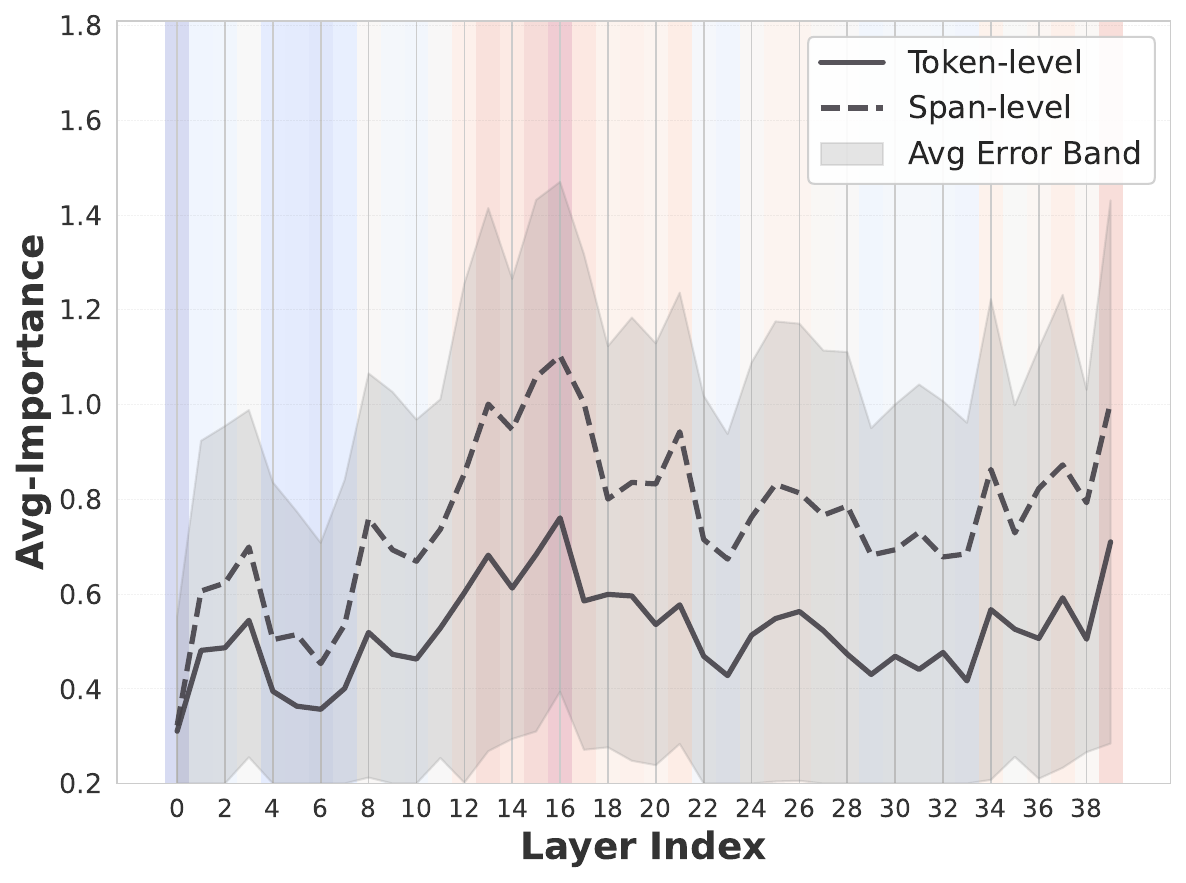}
        \caption{Fourier-high of LLaMA-13B}
        \label{fig:layer_l13b_f}
    \end{subfigure}
    \hfill
    \begin{subfigure}{0.305\textwidth}
        \centering
        \includegraphics[width=\textwidth]{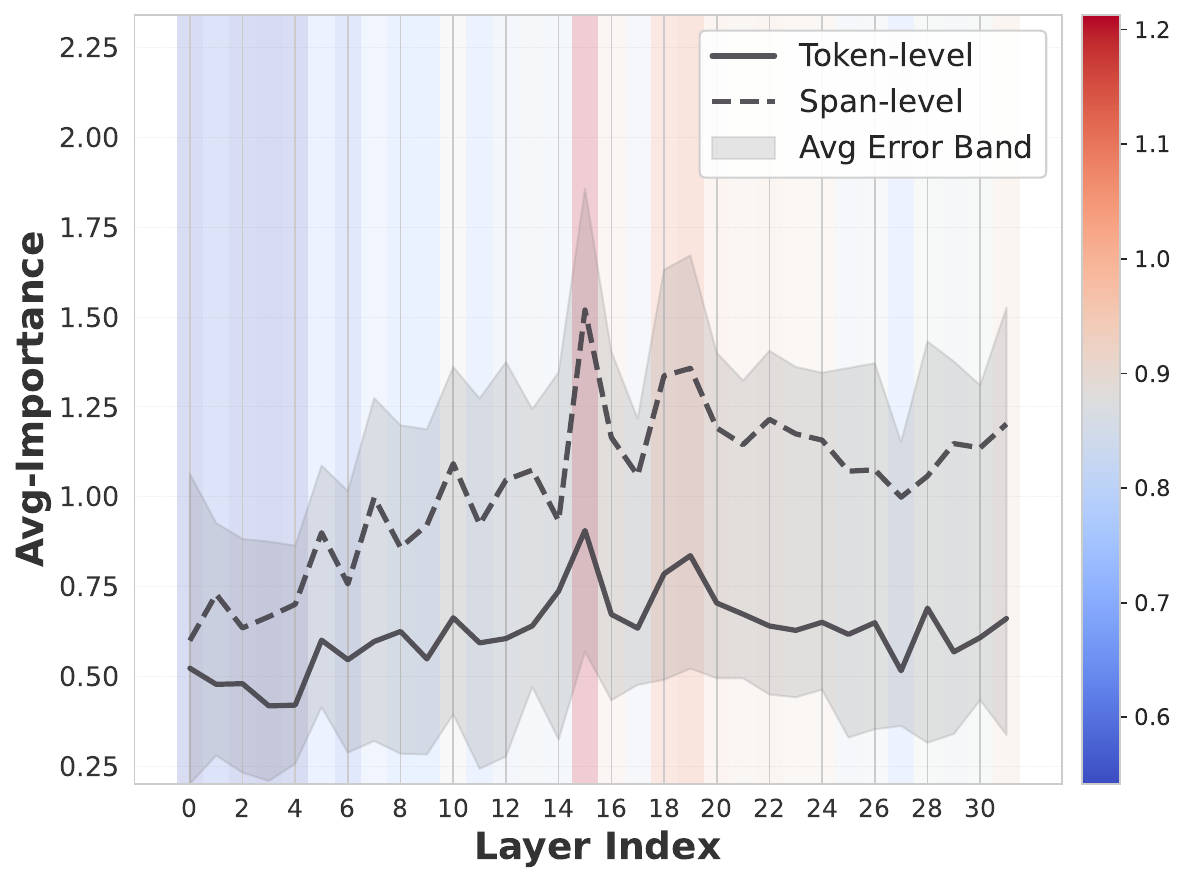}
        \caption{Fourier-high of Mistral-7B}
        \label{fig:layer_m7b_f}
    \end{subfigure}

    \vspace{1em}

    \caption{Full results for layer-wise importance. Solid and dashed lines correspond to token-level and span-level detection, respectively. The shaded region indicates the standard deviation of head-level importance within each layer.}
    \label{fig:ablation_layerwise}
\end{figure}

\begin{table}[htbp]
\centering
\small
\setlength{\tabcolsep}{5.0pt}
\setlength{\dashlinedash}{0.6pt}
\setlength{\dashlinegap}{1.4pt}
\caption{Original vs. Top-$k$ head only Performance.}
\label{tab:top_k_headRmv_full}

\begin{subtable}[t]{0.48\linewidth}
\centering
\caption{AUROC Results for LLaMA-13B.}
\label{tab:headrmv_l13b}
\begin{tabular}{l c:ccc}
\toprule
\multicolumn{5}{c}{\textbf{RagTruth-Avg}} \\
\cmidrule(lr){1-5}
\multirow{2}{*}{\textbf{Method}} 
& \multirow{2}{*}{\textbf{Original}} 
& \multicolumn{3}{c}{\textbf{Top-$k$ heads}} \\
\cmidrule(lr){3-5}
& & \textit{k = 100} & \textit{50} & \textit{10} \\
\midrule
\quad Laplacian & 0.8016 & 0.7858 & 0.7595 & 0.6424 \\
\quad Wavelet-high   & 0.8238 & 0.8079 & 0.7821 & 0.7191 \\
\quad Fourier   & 0.8326 & 0.8204 & 0.7899 & 0.7041 \\
\midrule
\multicolumn{5}{c}{\textbf{HalluRAG}} \\
\cmidrule(lr){1-5}
\multirow{2}{*}{\textbf{Method}} 
& \multirow{2}{*}{\textbf{Original}} 
& \multicolumn{3}{c}{\textbf{Top-$k$ heads}} \\
\cmidrule(lr){3-5}
& & \textit{k = 100} & \textit{50} & \textit{10} \\
\midrule
\quad Laplacian & 0.7624 & 0.6924 & 0.6764 & 0.6335 \\
\quad Wavelet-high   & 0.7809 & 0.7433 & 0.7133 & 0.6292 \\
\quad Fourier-high   & 0.7899 & 0.7749 & 0.7572 & 0.6455 \\
\bottomrule
\end{tabular}
\end{subtable}
\hfill
\begin{subtable}[t]{0.48\linewidth}
\centering
\caption{AUROC Results for Mistral-7b.}
\label{tab:headrmv_mistral}
\begin{tabular}{l c:ccc}
\toprule
\multicolumn{5}{c}{\textbf{RagTruth-Avg}} \\
\cmidrule(lr){1-5}
\multirow{2}{*}{\textbf{Method}} 
& \multirow{2}{*}{\textbf{Original}} 
& \multicolumn{3}{c}{\textbf{Top-$k$ heads}} \\
\cmidrule(lr){3-5}
& & \textit{k = 100} & \textit{50} & \textit{10} \\
\midrule
\quad Laplacian & 0.8625 & 0.8312 & 0.7839 & 0.6795 \\
\quad Wavelet-high   & 0.8673 & 0.8483 & 0.8317 & 0.7707 \\
\quad Fourier-high   & 0.8669 & 0.8440 & 0.8283 & 0.7624 \\
\midrule
\multicolumn{5}{c}{\textbf{HalluRAG}} \\
\cmidrule(lr){1-5}
\multirow{2}{*}{\textbf{Method}} 
& \multirow{2}{*}{\textbf{Original}} 
& \multicolumn{3}{c}{\textbf{Top-$k$ heads}} \\
\cmidrule(lr){3-5}
& & \textit{k = 100} & \textit{50} & \textit{10} \\
\midrule
\quad Laplacian & 0.8098 & 0.7611 & 0.7433 & 0.6807 \\
\quad Wavelet-high   & 0.8360 & 0.7715 & 0.7467 & 0.7091 \\
\quad Fourier-high   & 0.8267 & 0.7855 & 0.7587 & 0.6920 \\
\bottomrule
\end{tabular}
\end{subtable}

\end{table}

\newcommand{\lightrule}{\arrayrulecolor{black!30}\cmidrule(lr){2-7}\arrayrulecolor{black}}

\begin{table}[t]
\caption{Ablation study comparing context-only and generated-only attention features.
Results are reported across spectral operators and models.}
\label{tab:ablation_context_new_full}
\centering
\small
\setlength{\tabcolsep}{5pt}
\begin{tabular}{l l l cc cc}
\toprule
\textbf{Model} & \textbf{Setting} & \textbf{Operator}
& \multicolumn{2}{c}{\textbf{RagTruth-Avg}}
& \multicolumn{2}{c}{\textbf{HalluRAG}} \\
\cmidrule(lr){4-5} \cmidrule(lr){6-7}
& & & F & AUROC & F & AUROC \\
\midrule

\multirow{9}{*}{LLaMA-7B}
& \multirow{3}{*}{Original}
& Laplacian & 0.6588 & 0.8003 & 0.6370 & 0.7429 \\
& & Wavelet-high   & 0.6673 & 0.8087 & 0.6384 & 0.7550 \\
& & Fourier-high   & 0.6685 & 0.8205 & 0.6478 & 0.7629 \\
\lightrule
& \multirow{3}{*}{Context-only}
& Laplacian & 0.6504 & 0.8014 & 0.6340 & 0.7396 \\
& & Wavelet-high   & 0.6541 & 0.8074 & 0.6402 & 0.7479 \\
& & Fourier-high   & 0.6544 & 0.8138 & 0.6385 & 0.7538 \\
\lightrule
& \multirow{3}{*}{Generated-only}
& Laplacian & 0.6441 & 0.7855 & 0.6183 & 0.7240 \\
& & Wavelet-high   & 0.6430 & 0.7889 & 0.6264 & 0.7334 \\
& & Fourier-high   & 0.6416 & 0.7939 & 0.6262 & 0.7349 \\

\midrule

\multirow{9}{*}{LLaMA-13B}
& \multirow{3}{*}{Original}
& Laplacian & 0.6526 & 0.8016 & 0.6659 & 0.7624 \\
& & Wavelet-high   & 0.6688 & 0.8238 & 0.6684 & 0.7809 \\
& & Fourier-high   & 0.6759 & 0.8326 & 0.6732 & 0.7899 \\
\lightrule
& \multirow{3}{*}{Context-only}
& Laplacian & 0.6603 & 0.8144 & 0.6538 & 0.7627 \\
& & Wavelet-high   & 0.6634 & 0.8206 & 0.6524 & 0.7697 \\
& & Fourier-high   & 0.6688 & 0.8333 & 0.6581 & 0.7815 \\
\lightrule
& \multirow{3}{*}{Generated-only}
& Laplacian & 0.6468 & 0.7988 & 0.6430 & 0.7503 \\
& & Wavelet-high   & 0.6579 & 0.8080 & 0.6524 & 0.7697 \\
& & Fourier-high   & 0.6446 & 0.8053 & 0.6473 & 0.7642 \\

\midrule

\multirow{9}{*}{Mistral-7B}
& \multirow{3}{*}{Original}
& Laplacian & 0.7262 & 0.8625 & 0.7001 & 0.8098 \\
& & Wavelet-high   & 0.7287 & 0.8673 & 0.7274 & 0.8360 \\
& & Fourier-high   & 0.7300 & 0.8669 & 0.7221 & 0.8267 \\
\lightrule
& \multirow{3}{*}{Context-only}
& Laplacian & 0.7289 & 0.8636 & 0.7207 & 0.8242 \\
& & Wavelet-high   & 0.7280 & 0.8655 & 0.7176 & 0.8238 \\
& & Fourier-high   & 0.7299 & 0.8668 & 0.7128 & 0.8227 \\
\lightrule
& \multirow{3}{*}{Generated-only}
& Laplacian & 0.7151 & 0.8515 & 0.6840 & 0.7802 \\
& & Wavelet-high   & 0.7185 & 0.8548 & 0.6994 & 0.7983 \\
& & Fourier-high   & 0.7178 & 0.8531 & 0.6895 & 0.7913 \\

\bottomrule
\end{tabular}
\end{table}

\subsection{Examples of Raw Attention Signals}

\autoref{fig:attention_examples} presents qualitative examples of raw attention distributions from individual attention heads.
Each row corresponds to the same attention head at a fixed layer, while the left and right columns show attention over a hallucinated token and a non-hallucinated token, respectively.

We emphasize that these attention distributions reflect real model behavior and are substantially more irregular than the schematic examples shown in \autoref{fig:hypothese_compare}.
In practice, attention weights are often sparse, unevenly distributed, and exhibit non-trivial fluctuations across token positions, rather than forming unimodal patterns \cite{nawrot2025sparsefrontiersparseattention}.

Notably, for the examples shown in the top row, the two tokens share identical lookback ratios as measured by Lookback-Lens.
Similarly, for the bottom row, the context entropy of the two attention distributions is equal.
In these cases, these aggregate statistics alone are insufficient to distinguish hallucinated from non-hallucinated cases by visual inspection.
In contrast, after applying a high-pass filter to the same attention signals, the resulting high-frequency energy differs substantially between the two cases.
Although the raw attention curves may appear similar at a coarse level, frequency-domain filtering reveals differences in fine-grained variation patterns.

\begin{figure}[htbp]
    \centering

    \begin{subfigure}{0.485\linewidth}
        \centering
        \includegraphics[width=\linewidth, trim=50 90 50 75, clip]{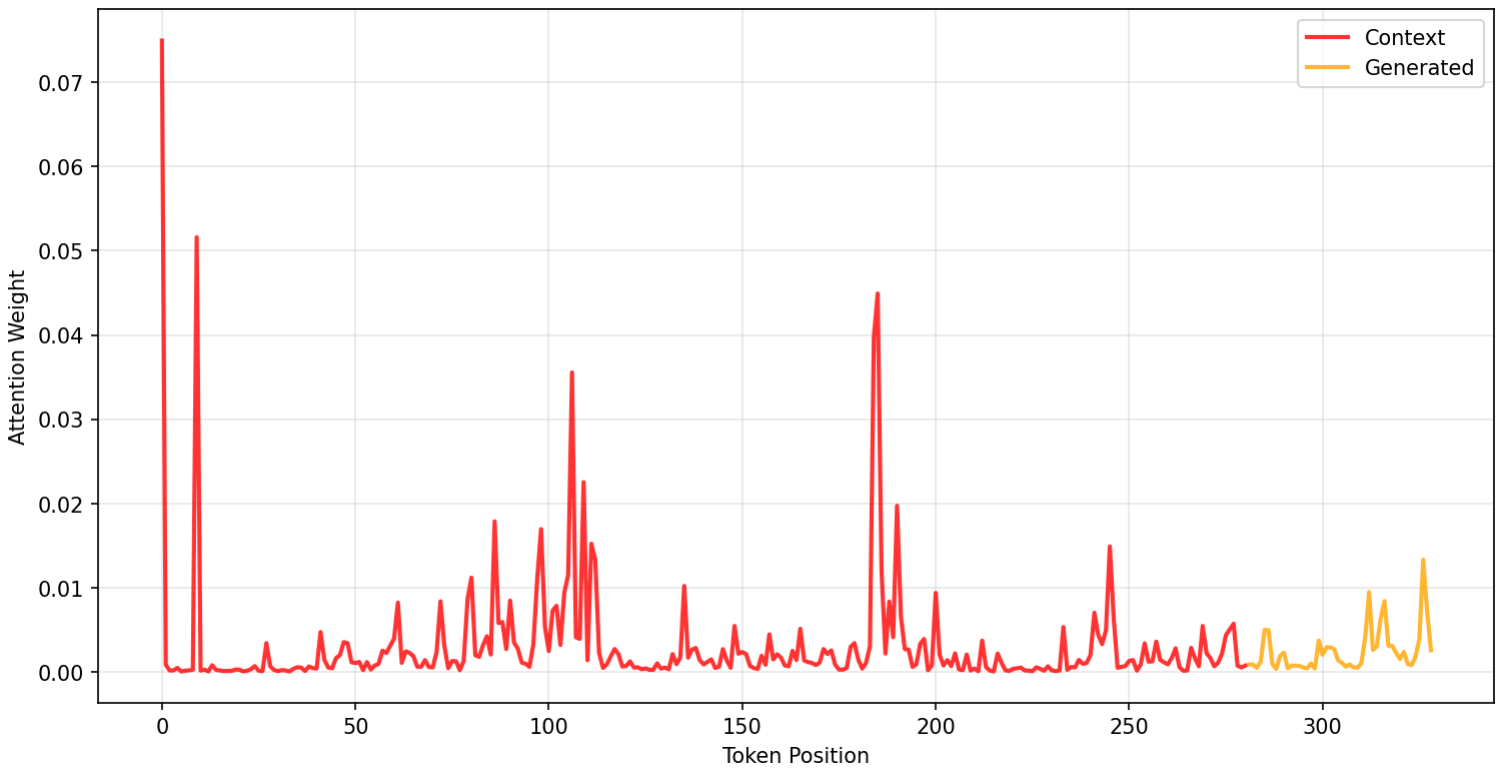}
        \caption{Layer 13, Head 09: Attention for a hallucinated token}
        \label{fig:attn_l13_h09_hallu_token}
    \end{subfigure}
    \hspace{0.01\linewidth}
    \begin{subfigure}{0.485\linewidth}
        \centering
        \includegraphics[width=\linewidth, trim=50 90 50 75, clip]{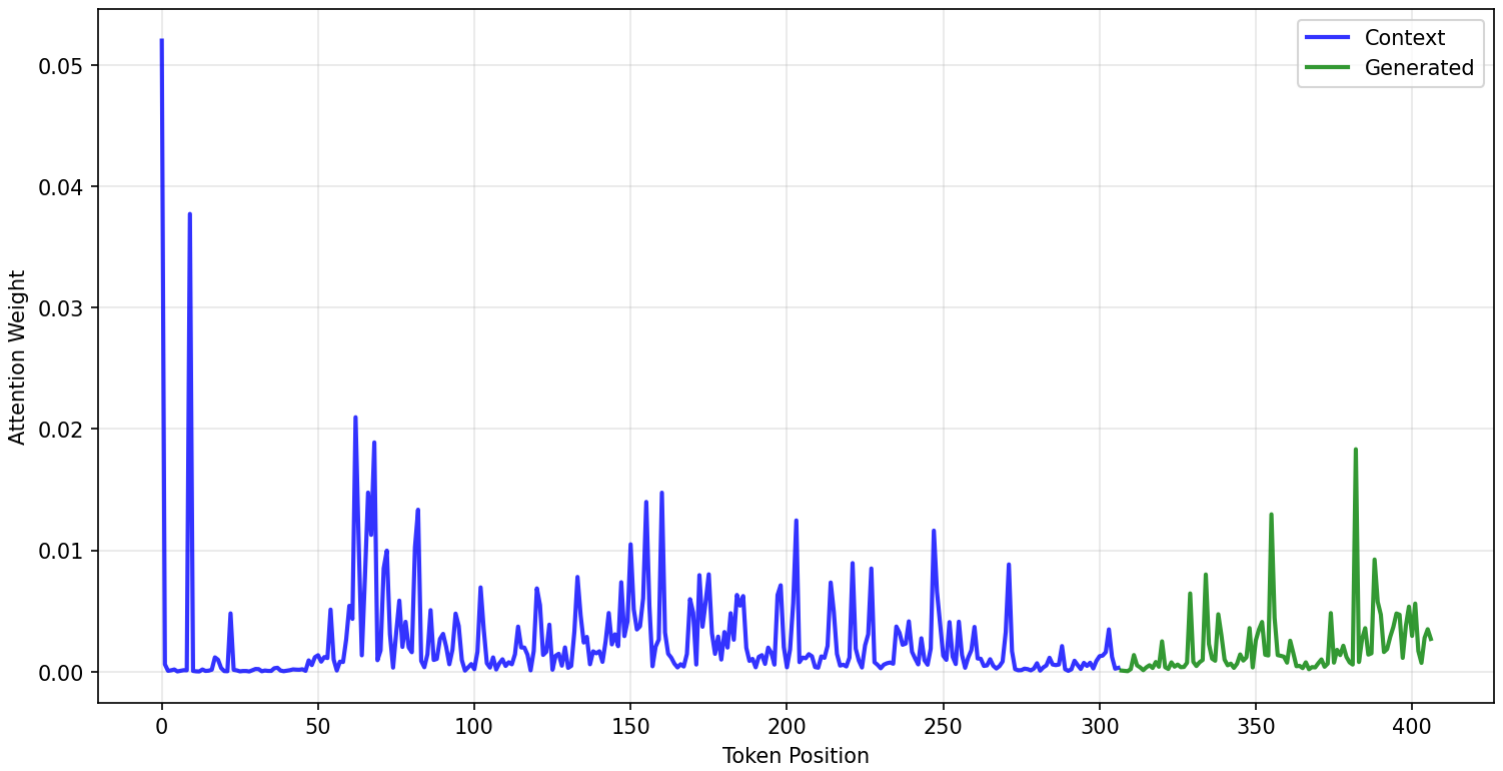}
        \caption{Layer 13, Head 09: Attention for a non-hallucinated token}
        \label{fig:attn_l13_h09_normal_span}
    \end{subfigure}

    \vspace{0.25em}

    \begin{subfigure}{0.485\linewidth}
        \centering
        \includegraphics[width=\linewidth, trim=50 90 50 75, clip]{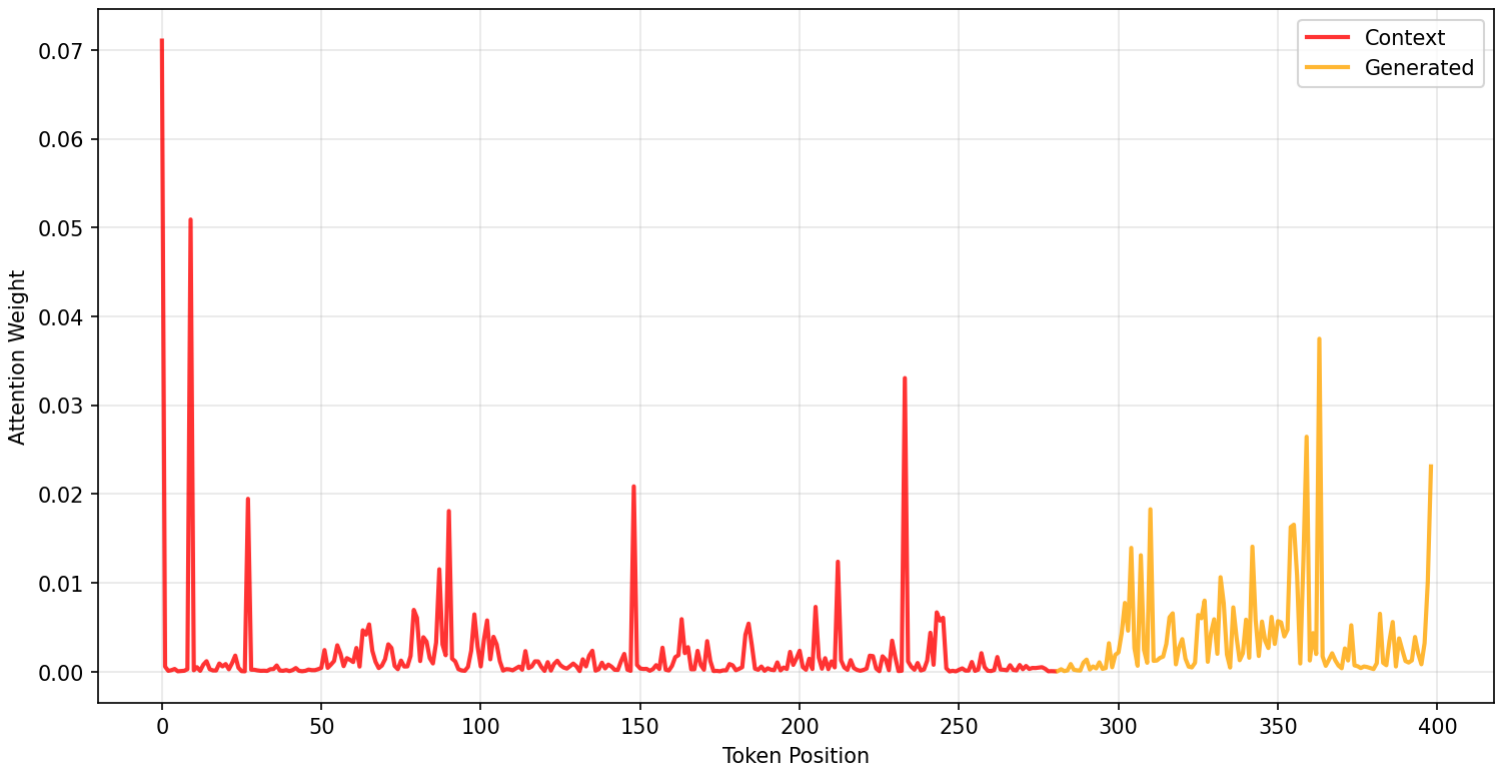}
        \caption{Layer 15, Head 29: Attention for a hallucinated token}
        \label{fig:attn_l15_h29_hallu_token}
    \end{subfigure}
    \hspace{0.01\linewidth}
    \begin{subfigure}{0.485\linewidth}
        \centering
        \includegraphics[width=\linewidth, trim=50 90 50 75, clip]{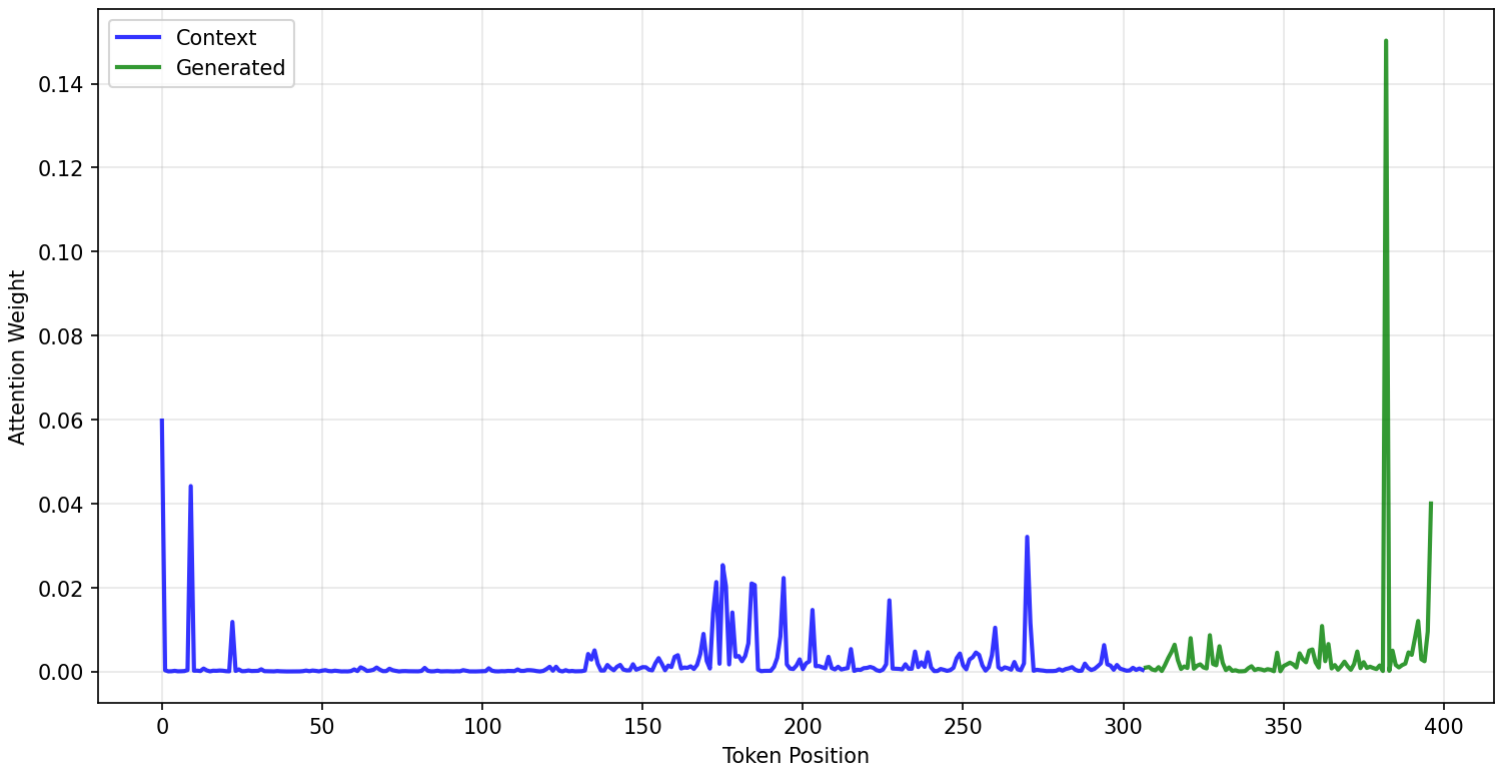}
        \caption{Layer 15, Head 29: Attention for a non-hallucinated token}
        \label{fig:attn_l15_h29_normal_span}
    \end{subfigure}

    \caption{
Raw attention signal visualizations.
Each subfigure compares raw attention distributions for a hallucinated token (left) and a non-hallucinated token (right).
Red and yellow curves denote attention over context and generated tokens, respectively, when generating a hallucinated token,
while blue and green curves denote attention over context and generated tokens for a non-hallucinated token.
    }
    \label{fig:attention_examples}
\end{figure}


\end{document}